\documentclass[10pt]{article}

\usepackage[usenames,dvipsnames,svgnames,table]{xcolor}
\usepackage[colorlinks=false,urlbordercolor={1 1 1},citebordercolor={1 1 1},linkbordercolor={1 1 1}]{hyperref}

\usepackage{calc}
\usepackage[round,comma]{natbib}

\usepackage[lined,linesnumbered,ruled,commentsnumbered]{algorithm2e}
\usepackage{enumitem}
\usepackage{sectsty,amsmath, amsfonts, amsthm, amssymb, mathtools,dsfont,mathrsfs}
\usepackage{tocvsec2}
\usepackage{etoolbox}
\usepackage{sectsty,amsmath, amsfonts, amsthm, amssymb, mathtools,dsfont}

\usepackage{fullpage}
\usepackage{subfig}
\usepackage{float}

\usepackage{forest, siunitx, booktabs}
\usepackage{bbm}	
\usepackage{etoolbox}
\usepackage{centernot}
\usepackage{graphics}
\usepackage{epsfig}
\usepackage[titletoc]{appendix}
\usepackage{fancyhdr}
\usepackage{courier}
\usepackage{soul}

\makeatletter
\def\namedlabel#1#2{\begingroup
    #2%
    \def\@currentlabel{#2}
    \phantomsection\label{#1}\endgroup
}
\newcommand{\specificthanks}[1]{\@fnsymbol{#1}}
\newcommand{\lrarrow}{\mathrel{\mathpalette\lrarrow@\relax}}
\newcommand{\lrarrow@}[2]{%
  \vcenter{\hbox{\ooalign{%
    $\m@th#1\mkern6mu\rightarrow$\cr
    \noalign{\vskip1pt}
    $\m@th#1\leftarrow\mkern6mu$\cr
  }}}%
}
\makeatother

\newcommand{\E}{\mathbb{E}}
\newcommand{\1}{\mathbbm{1}}
\renewcommand{\P}{\mathbb{P}}

\newcommand{\favdir}{\varsigma}

\newcommand{\sgn}{{\rm sgn}}

\newcommand{\Tcal}{\mathcal{T}}
\newcommand{\Pcal}{\mathcal{P}}

\newcommand{\RR}{\mathbb{R}}
\newcommand{\PP}{\mathbb{P}}
\newcommand{\eps}{\varepsilon}
\newcommand{\Chi}{\mathcal{X}}

\newcommand{\A}{\mathcal{A}}

\newcommand{\Bias}{\text{\rm Bias}}

\newcommand{\fhat}{\hat{f}}

\newcommand{\ftilde}{\tilde{f}}

\newcommand{\Yhat}{\widehat{Y}}
\newcommand{\signbias}{ \widetilde{bias} }
\newcommand{\bias}{ bias }
\newcommand{\modbias}{ \Bias }
\newcommand{\PDP}{ P\!D\!P }
\newcommand{\SHAP}{\varphi}
\newcommand{\CE}{ C\!E}
\newcommand{\ME}{ M\!E}
\newcommand{\vpdp}{v^{\ME}}
\newcommand{\vce}{v^{\CE}}

\def\clap#1{\hbox to 0pt{\hss#1\hss}}

\newtheorem{theorem}{Theorem}[section]
\newtheorem{definition}{Definition}[section]
\newtheorem{lemma}{Lemma}[section]
\newtheorem{remark}{Remark}[section]

\usepackage{xcolor}

\newbool{showpics}
\booltrue{showpics}

\usepackage[utf8]{inputenc}
\usepackage{amsmath}
\usepackage{calc}
\usepackage[multiple]{footmisc}

\numberwithin{equation}{section}

\title{Wasserstein-based fairness interpretability framework for machine learning models}
\author{ Alexey Miroshnikov\thanks{Emerging Capabilities Research Group, Discover Financial Services Inc., Riverwoods, IL} \textsuperscript{,$\!\!\!$}
\thanks{co-first author and corresponding author, alexeymiroshnikov@discover.com} \and
  Konstandinos Kotsiopoulos\textsuperscript{\specificthanks{1},}\thanks{co-first author, kostaskotsiopoulos@discover.com}  \and
  Ryan Franks\textsuperscript{\specificthanks{1}}\thanks{ryanfranks@discover.com} \and Arjun Ravi Kannan\textsuperscript{\specificthanks{1},}\thanks{arjunravikannan@discover.com}  }

\date{}

\begin{document}

\maketitle

\abstract{
The objective of this article is to introduce a fairness interpretability framework for measuring and explaining the bias in classification and regression models at the level of a distribution. In our work, we measure the model bias across sub-population distributions in the model output using the Wasserstein metric. To properly quantify the contributions of predictors, we take into account the favorability of both the model and predictors with respect to the non-protected class. The quantification is accomplished by the use of transport theory, which gives rise to the decomposition of the model bias and bias explanations to positive and negative contributions. To gain more insight into the role of favorability and allow for additivity of bias explanations, we adapt techniques from cooperative game theory.
}

\vspace{10pt}



\vspace{5pt}


\section{Introduction}\label{sec::introduction}

Contemporary machine learning (ML) techniques surpass traditional statistical methods in terms of their higher predictive power and their capability of processing a larger number of attributes. However, these novel ML algorithms generate models that have a complex structure which makes it difficult for their outputs to be interpreted with high precision. Another important issue is that a highly accurate predictive model might lack fairness by generating outputs that may result in discriminatory outcomes for protected subgroups. Thus, it is imperative to design predictive systems that are not only accurate but also achieve the desired fairness level.

When used in certain contexts, predictive models, and strategies that rely on such models, are subject to laws and regulations that ensure fairness. For instance, a hiring process in the United States (US) must comply with the Equal Employment Opportunity Act \citep{EEOA}. Similarly, financial institutions (FI) in the US that are in the business of extending credit to applicants are subject to the Equal Credit Opportunity Act \citep{ECOA}, the Fair Housing Act \citep{FHA}, and other fair lending laws. These laws often specify protected attributes that FIs must consider when maintaining fairness in lending decisions.

Examples of protected attributes include race, gender, age, ethnicity, national origin, marital status, and others. Under the ECOA, for example, it is unlawful for a creditor to discriminate against an applicant for a loan on the basis of race, gender or age. Even though direct usage of protected attributes in building a model is often prohibited by law (e.g. overt discrimination), some otherwise benign attributes can serve as ``proxies'' because they may share dependencies with a protected attribute. For this reason, it is crucial for data scientists to conduct a fairness review of their trained models in consultation with compliance professionals in order to evaluate the predictive modeling system for potential unfairness. In this paper, we develop a fairness interpretability framework to aid in this important task.

At an algorithmic level, bias can be viewed as an ability to differentiate between two subpopulations at the level of data or outcomes. Regardless of its definition, if bias is present in data when training an ML model, the ability to differentiate between subgroups might potentially lead to discriminatory outcomes. For this reason, the model bias can be viewed as a measure of unfairness and hence its measurement is central to the model fairness assessment.

There is a comprehensive body of research on ML fairness that discusses bias measurements and mitigation methodologies. \citet{Kamiran2009} introduced a classification scheme for learning unbiased models by modifying the biased data sets without direct knowledge of the protected attribute. \citet{Kamishima2012} proposed a regularization approach for discriminative probabilistic models. \citet{Zemel2013} designed an optimization problem that incorporates fairness constraints. \citet{Feldman2015} proposed a geometric repair scheme to remove disparate impact in classifiers by making data sets unbiased. \citet{Hardt2015} indtroduced post-processing techniques removing discrimination in classifiers based on equalized odds and equal opportunity fairness criteria. \citet{Woodworth2017} designed a framework for nearly-optimal learning predictors with equalized odds fairness constraint. \citet{Zhang2018} proposed to use adversarial learning to mitigate bias, and \citet{Jiang2020} suggested a bias correction technique via re-weighting the data.

The work of \citet{Dwork2012} studies Lipschitz randomized classifiers and  their statistical parity bias. It establishes a bound on that bias by a transport-like distance between the input subpopulation distributions. The bound aids in constructing an optimal Lipschitz classifier with control over the statistical parity bias by transporting one of the subpopulation input datasets into the other. The work of \citet{Gordaliza2020} establishes a similar bound for non-randomized classifiers by the total variance distance between input subpopulation distributions. Guided by the bound and  utilizing optimal transport theory, their method focuses on repairing input datasets in a way that allows for control of the total variance distance, and hence the statistical parity bias.

Though the bounds in the aforementioned works are of theoretical and practical importance, they provide little information on how each component of the input contributes to the bias in the output. The main reason for that is that the bias from the inputs propagates through the model structure in a non-trivial way. For this reason, in our work, we focus on designing a fairness interpretability framework that evaluates how each predictor contributes to the model bias, incorporating the predictor's favorability with respect to protected (or minority) class into the framework. The construction is carried out by employing optimal transport theory and game-theoretic techniques.

Another issue regarding the ML fairness literature is that it mainly focuses on classifiers. Specifically, given the data $(X,G,Y)$, where $X\in\RR^n$ are predictors, $G\in\{0,1\}$ is a protected attribute and $Y\in\{0,1\}$ is a binary output variable, with favorable outcome $Y=1$, the bias measurements are often based on fairness criteria such as statistical parity, which reads $\P(\hat{Y}=1|G=0)=\P(\hat{Y}=1|G=1)$, or alternative criteria such as equalized odds and equal opportunity \citep{Feldman2015, Hardt2015}.

Many models in the financial industry, however, are regressors $f=\E[Y|X]$. In turn, classification models are usually obtained by thresholding the regressor, $Y_t(X)=\1_{\{f(X)>t\}}$, but the thresholds are in general not chosen during the model development stage. Thus, data scientists select the classification score $f(X)=\widehat{\PP}(Y=1|X)$ based on the overall performance across all thresholds. The same is true for fairness assessment, which is conducted at the level of the whole classification score. The main reason for this is that the strategies and decision-making procedures in FIs may rely on the classification score or its distribution, not a single classifier with a fixed threshold. This motivates us to measure and explain the bias exclusively in the regressor model.

Our interpretability framework in principle can be applied to a wide range of predictive ML systems. For instance, it can provide insight into predictor attributions for models that appear in economics, social sciences, medicine, and other fields.

Another application of the framework is for bias mitigation under regulatory  constraints. In FIs, bias mitigation methodologies that require explicit consideration of protected class status in the training or prediction stages are not acceptable in view of ECOA. Consequently, bias mitigation methods such as those in \citet{Dwork2012,Feldman2015,Gordaliza2020} are not feasible. However, a probabilistic proxy model for a protected attribute $G$ such as the Bayesian Improved Surname and Geocoding (BISG) is allowed to be used for fairness assessment and subsequent post-processing\footnote{Compliance departments employ the proxy model for compliance purposes only.} \citep{Elliot2009,Hall2021}; for an alternative proxy model, see \citet{Chen2019}. This setup allows for the use of our framework in the following regulatory-compliant fashion:

\begin{itemize}[label=$\bullet$]
\item [(S1)] Given a model $f$ and the proxy protected attribute $\tilde{G}$, perform a fairness assessment by measuring the bias across the subpopulation distributions $f(X)|\tilde{G}=k$, $k\in\{0,1\}$.

\item [(S2)] If the model bias exceeds a certain threshold, determine the main drivers for the bias, that is, determine the list of predictors $X_{i_1}, X_{i_2},\dots,X_{i_r}$ contributing the most to that bias.

\item [(S3)]  Mitigate the bias by constructing a post-processed model $\ftilde(X; f)$ utilizing the information on the most biased predictors $\{X_{i_1},X_{i_2},\dots,X_{i_r}\}$ and without the direct use of  $\tilde{G}$  or any information on the joint distribution $(X,\tilde{G})$.
\end{itemize}

In this article, the interpretability framework we develop addresses steps (S1) and (S2). The post-processing methods (S3) are investigated in our companion paper \citet{Miroshnikov2021b}. In what follows, we provide a summary of the key ideas and main results.

\vspace{4pt}

\noindent{\bf Problem setup.} We consider the joint distribution $(X,G,Y)$, where $X\in\RR^n$ are predictors, $G \in \{0,1\}$ is the protected attribute, with the non-protected class $G=0$, and $Y$ is either a response variable with values in $\RR$ (not necessarily a continuous random variable) or binary one with values in $\{0,1\}$. We denote a trained model by $f(x)=\widehat{\E}[Y|X=x]$, assumed to be trained on $(X,Y)$ without access to $G$. We assume that there is a predetermined favorable model direction, denoted by $\uparrow$ and $\downarrow$; if the favorable direction is $\uparrow$ then the relationship $f(x)>f(y)$ favors the input $x$, and if it is $\downarrow$ the input $y$. In the case of binary $Y\in\{0,1\}$,  the favorable direction $\uparrow$ is equivalent to $Y=1$ being a favorable outcome, and $\downarrow$ to $Y=0$. To simplify the exposition, the main text focuses on the case of a binary protected attribute $G$. However, the framework and all of the results in the article have a natural extension to the multi-labeled case.

\vspace{4pt}

\noindent{\bf Key components of the framework.}

\begin{itemize}[label=$\bullet$]

\item Motivated by optimal transport theory, we focus on the bias measurement in the model output via the Wasserstein metric $W_1$
\[
    \modbias_{W_1}(f|G) = \inf_{\pi \in \mathscr{P}(\RR^2)} \Big\{ \int_{\RR^2} |x_1-x_2| \, d\pi(x_1,x_2), \,\, \text{with marginals $P_{f(X)|G=0}, P_{f(X)|G=1}$}  \Big\},
\]
which measures the minimal cost of transporting one distribution into another; see \citet{Santambrogio2015}. More importantly, we introduce the model bias decomposition into the sum of the positive and negative model biases, $\modbias_{W_1}^{\pm}(f|G)$, which measure the transport effort for moving points of the unprotected subpopulation distribution $f(X)|G=0$ in the non-favorable and favorable directions, respectively. This allows us to obtain a more informed perspective on the predictor's impact; see Sections \ref{subsec::posnegflow},\ref{subsec::W1modbias}.

\item We establish the connection of the model bias with that of a classifier. We show that the positive and negative model bias can be viewed as the integrated statistical parity bias over the family of classifiers induced by the regressor. This integral relationship is then used to construct an extended family of transport metrics for regressor bias. Via integration, these metrics incorporate generic group parity fairness criteria for classifiers induced by the given regressor. Furthermore, we prove a more general version of \citet[Theorem 3.3]{Dwork2012} that establishes the connection between the Wasserstein-based bias and the randomized classifier-based bias; see Sections \ref{subsec::W1modbias}, \ref{sec::genparityconsist}.

\item
We introduce bias predictor attributions called {\it bias explanations} in order to understand how predictors contribute to the model bias. The bias explanation $\beta_i$ of predictor $X_i$ is computed as the cost of transporting the distribution of $E_i|G=0$ to that of $E_i|G=1$, where $E_i(X;f)$ quantifies the contribution of $X_i$ to the model value. The transport theory gives rise to the decomposition $\beta_i=\beta_i^++\beta_i^-$ into the sum of positive and negative model bias explanations. Roughly speaking, $\beta_i^+$ quantifies the combined predictor contribution to the increase of the positive model bias and decrease in the negative model bias, and vice versa for $\beta_i^-$; see Section \ref{sec::biasexplanations}.

\item
The bias explanations are in general not additive, even if the predictor explanations are. To construct additive bias explanations and to better capture the interactions at the distribution level, we employ a cooperative game theory approach  motivated by the ideas of \citet{Strumbelji2010}. We design a cooperative bias game $v^{bias}$ which evaluates the bias in the model attributed to coalitions $X_S$, $S \subset \{1,\dots,n\}$, and define bias explanations via the Shapley value $\varphi[v^{bias}]$, which yields additivity. Similar approach is applied to construct additive positive and negative bias explanations; see Section \ref{app::biasgame}.

\item We choose to design the bias explanations based upon model explainers $E_i$ that are either conditional or marginal expectations, or game-theoretic explainers in the form of the Shapley value $\varphi[v]$ where $v$ is either a conditional game $\vce$ or a marginal game $\vpdp$. For each $v \in \{\vce,\vpdp\}$ we perform the stability analysis of non-additive and additive bias explanations. By adapting the grouping techniques from \citet{Kotsiopoulos2020}, we reduce the complexity of game-theoretic bias explanations and unite marginal and conditional approaches; see Sections \ref{subsec::biasexplstab}-\ref{sec::groupbiasexpl}.

\end{itemize}

\noindent{\bf Structure of the paper.} In Section \ref{sec::preliminaries}, we introduce the requisite notation and fairness criteria for classifiers, and discuss ML fairness literature related to our work. In Section \ref{sec::modbias}, we introduce the Wasserstein-based regressor bias and investigate its properties. In addition, we discuss a wide class of transport metrics that could be used for fairness assessment. In Section \ref{sec::biasexplmain}, we provide a theoretical characterization of the bias explanations and investigate their properties. In Section \ref{sec::application} we discuss some regulatory aspects of bias mitigation, and present an application of the framework to a UCI dataset. In Appendix \ref{sec::kpminimization}, we discuss the Kantorovich transport problem. In Appendix \ref{app::auxlemm}, we state and prove auxiliary lemmas.

\section{Preliminaries} \label{sec::preliminaries}
\subsection{Notation and hypotheses} \label{subsec::notation}

We consider the joint distribution $(X,G,Y)$, where $X=(X_1,X_2,\dots,X_n) \in \RR^n$ are the predictors, $G\in \{0,1,\dots,K-1\}$ is the protected attribute and $Y$ is either a response variable with values in $\RR$ (not necessarily a continuous random variable) or a binary one with values in $\{0,1\}$. We encode the non-protected class as $G=0$ and assume that all random variables are defined on the common probability space $(\Omega,\mathcal{F},\PP)$, where $\Omega$ is a sample space, $\PP$ a probability measure, and $\mathcal{F}$ a $\sigma$-algebra of sets.

The true model and a trained one, which is assumed to be trained without access to $G$, are denoted by 
\[
f(X)=\E[Y|X] \quad \text{and} \quad \fhat(X)=\widehat{\E}[Y|X],
\]
respectively. In the case of binary $Y$ they read $f(X)=\PP(Y=1|X)$ and $\fhat(X)=\widehat{\PP}(Y=1|X)$. We denote a classifier based on the trained model by 
\[
\Yhat_t=\Yhat_t(X;\fhat)=\1_{\{\hat{f}(X)>t\}}, \quad t \in \RR.
\]
The subpopulation cumulative distribution function (CDF) of  $\fhat(X)|G=k$ is denoted by
\begin{equation*} 
F_k(t)=F_{\fhat(X)|G=k}(t)=\PP(\fhat(X)\leq t|G=k)
\end{equation*}
and the corresponding generalized inverse (or quantile function) $F_k^{[-1]}$ is defined by:
\begin{equation*} 
F_k^{[-1]}(p)=F_{\fhat(X)|G=k}^{[-1]}(p)=\inf_{x \in \RR }\big\{ p \leq F_k(x) \big\}.
\end{equation*} 

We assume that there is a predetermined {\it favorable model direction}, denoted by either $\uparrow$ or $\downarrow$. If the favorable direction is $\uparrow$ then the relationship $f(x)>f(z)$ favors the input $x$, and if it is $\downarrow$ the input $z$. The sign of the favorable direction of $f$ is denoted by $\varsigma_{f}$ and satisfies 
\[
\favdir_{f} = \left\{ 
\begin{aligned}
&1,& &\text{if the favorable direction of $f$ is $\uparrow$}&\\
-&1,& &\text{if the favorable direction of $f$ is $\downarrow$}\,.&
\end{aligned}
\right.
\] 
In the case of binary $Y$,  the favorable direction $\uparrow$ is equivalent to $Y=1$ being a favorable outcome, and $\downarrow$ to $Y=0$; see Section \ref{subsec::classbias}. 

In what follows we first develop the framework in the context of the binary protected attribute $G\in\{0,1\}$ and then extend it to the case of the multi-labeled protected attribute; see Section \ref{sec::genparityconsist}.


\subsection{Fairness criteria for classifiers}

When undesired biases concerning demographic groups (or protected attributes) are in the training data, well-trained models will reflect those biases. There have been numerous articles devoted to ML systems that lead to fair decisions. In these works, various measurements for fairness have been suggested. In what follows, we describe several well-known definitions which help measure fairness of classifiers. 

\begin{definition}\label{def::clfparity}
Suppose that $Y$ is binary with values in $\{0,1\}$ and $Y=1$ is the favorable outcome. Let $\Yhat$ be a classifier.
\begin{itemize}[label=$\bullet$]
\item  $\Yhat$ satisfies statistical parity \citep{Feldman2015} if
\[
\PP(\Yhat=1|G=0) = \PP(\Yhat=1|G=1).
\]
\item  $\Yhat$ satisfies equalized odds \citep{Hardt2015} if
\[
\PP(\Yhat=1|Y=y,G=0) = \PP(\Yhat=1|Y=y,G=1), \quad y\in\{0,1\}.
\]
\item  $\Yhat$ satisfies equal opportunity \citep{Hardt2015} if
\[
\PP(\Yhat=1|Y=1,G=0) = \PP(\Yhat=1|Y=1,G=1).
\]
\item  The balanced error rate (BER) of $\Yhat$ \citep{Feldman2015} is given by
\[
 BER(\Yhat, G) = \tfrac{1}{2}(\PP(\Yhat=1|G=0) + \PP(\Yhat=0|G=1)).
\]
\end{itemize}
\end{definition}
The statistical parity requires that the proportions of people  in the favorable class $\Yhat=1$ within each group $G=k,k\in\{0,1\}$ are the same. The  equalized odds constraint requires the classifier to have the same misclassification error rates for each class of the protected attribute $G$ and the label $Y$.  Equal opportunity constraint requires the misclassification rates to be the same for each class $G=k$ only for the individuals labeled as $Y=1$. The BER is the average class-conditioned error rate of $\Yhat$.


\subsection{Group classifier fairness example}\label{subsec::classbiasexample}

There are numerous reasons why a trained classifier may lead to unfair outcomes. To illustrate, we provide an instructive example that shows how predictors and labels, as well as their relationship with the protected attribute, affect classifier fairness.

Consider a data set $(X,Y,G)$ where the predictor $X$ depends on $G \in \{0,1\}$, $Y\in\{0,1\}$ is binary, with favorable outcome $Y=0$, and the classification score $f$ depends explicitly on $X$ only:
\begin{equation}\label{fair_score_model}\tag{M1}
\begin{aligned}
&X \sim N(\mu-a\cdot G,\sqrt{\mu}), \quad \mu=5, \, a=1\\
&Y \sim Bernoulli(f(X)),  \quad f(X)=\P(Y=1|X)={logistic(\mu-X)}.
\end{aligned}
\end{equation}

The data set is constructed in such a way that the proportions of $Y=0|G=k$ in the two groups are different: $\PP(Y=0|G=0)=0.5$, $\PP(Y=0|G=1) = 0.36$. The predictor $X$ serves as a good proxy for $G$, which can be seen in Figure \ref{fig::xsubpop}. The plot depicts the density of $X$ and the conditional densities of $X$ given $G=0$ and $G=1$, respectively. The shifted conditional densities clearly show the dependence of $X$ on $G$. Though the true score $f(X)$ does not depend explicitly on $G$, a classifier trained on $X$ will learn that the higher the value of $X$ the more likely it is that $Y=0$. Using the logistic regression model $\fhat$ we observe that for any threshold $t \in (0,1)$  the classifier $\Yhat_{t}$ satisfies neither the statistical parity, nor the equal opportunity, nor the equalized odds criterion. Furthermore, since both classes of $G$ are equally likely, $BER(\Yhat_t,G)<0.5$ implies that one can potentially infer $G$ from $X$; see Figure \ref{fig::fairperf}. The vertical axis in the plot represents the difference between the probabilities for each of the first three fairness metrics described in Definition \ref{def::clfparity} as well as the value of the balanced error rate. Notice how only in the trivial cases where $t\in \{0,1\}$ are all metrics satisfied and the balanced error rate is equal to $0.5$, since $\Yhat_0 = 1, \Yhat_1 = 0$ for all $X$.
\begin{figure}
\centering
\subfloat[\footnotesize Subpopulation distributions of $X$.]{\includegraphics[width=0.36\textwidth]{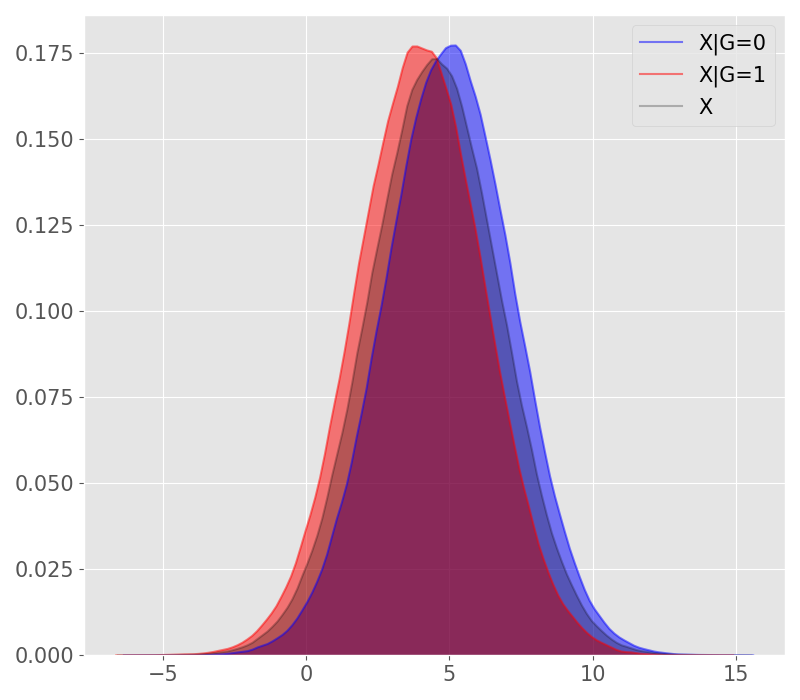}\label{fig::xsubpop}} 
~~
\subfloat[\footnotesize Fairness measurements.]{\includegraphics[width=0.36\textwidth]{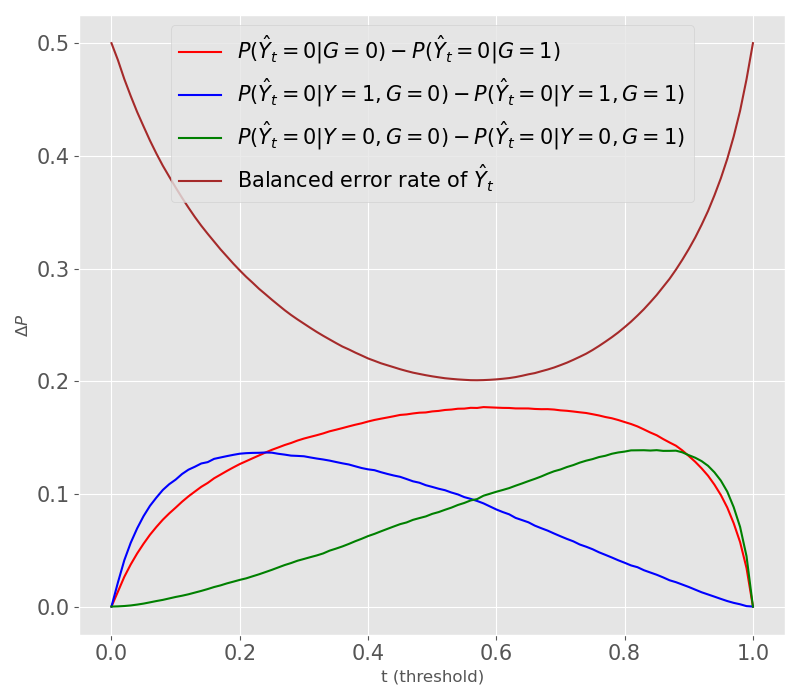}\label{fig::fairperf}}
\caption{ \footnotesize Predictor distributions and fairness for the model \eqref{fair_score_model}, $\favdir_f=-1$.} \label{fig::eqriskstat}
\end{figure}

%
%

\subsection{Classifier bias based on statistical parity}\label{subsec::classbias}

In this section we provide a definition for classifier bias based on the statistical parity fairness criterion and establish some basic properties of the classifier bias.  In what follows, we suppress the symbol $\,\hat{}\,$ , using it only when it is necessary to differentiate between the true model and the trained one. The same rule applies to classifiers.

\begin{definition}\label{def::statbias}
Let $f$ be a model, $X\in\RR^n$ predictors, $G \in \{0,1\}$ protected attribute, $G=0$ non-protected class, $\favdir_f$ the sign of the favorable direction, and $F_k$ the CDF of $f(X)|G=k$. 
\begin{itemize}[label=$\bullet$]
  \item The signed classifier (or statistical parity) bias for a threshold $t \in \RR$ is defined by
  \[
  \begin{aligned}
  \signbias^{C}_t(f|X,G) & = \big( \PP(Y_t=\1_{\{\favdir_f=1\}}|G=0)-\PP(Y_t=\1_{\{\favdir_f=1\}}|G=1) \big)\\    
  &=\big( F_1(t)-F_0(t) \big) \cdot \favdir_f.
  \end{aligned}
  \]
  \item The classifier bias at $t \in \RR$ is defined by
  \[
  \begin{aligned}
  \bias^C_t(f|X,G)=|\signbias^C_t(f|X,G)|.
  \end{aligned}
  \]
\end{itemize}
\end{definition}

We say that $Y_t$ favors the non-protected class $G=0$ if the signed bias is positive. Respectively, $Y_t$ favors the protected class $G=1$ if the signed bias is negative. 

\begin{remark}
Suppose that $Y \in \{0,1\}$ is binary and that the favorable direction is $\uparrow$, which implies that $\1_{\{\favdir_f=1\}}=1$. Then $Y_t$ favors the non-protected class $G=0$ if and only if there is a larger proportion of individuals from class $G=0$ for which $Y_t=1$  compared to the class $G=1$. This, from a statistical parity perspective, describes the outcome $Y=1$ as favorable. Similar remarks apply to the case when the favorable direction is $\downarrow$. Thus, the favorable direction is $\uparrow$ ($\downarrow$) is equivalent to the favorable outcome $Y=1$ ($Y=0$). 
\end{remark}


\subsection{Quantile bias and geometric parity}

Given a model $f$ and a threshold $t \in \RR$, the classifier bias based on statistical parity measures the difference in population sizes corresponding to groups $G=\{0,1\}$ for which $Y_t=0$. This measurement however does not take into account the geometry of the model distribution, that is, the score values themselves. 

For example, when measuring the bias in incomes among `females' and `males' one can view the difference of expected incomes in the two groups as `bias'. Alternatively, one can measure an income bias by evaluating the absolute difference  of the `female' median income and `male' median income, which is often done in various social studies. This motivates us to take into account the geometry of the score distribution when defining bias. For this reason, we propose the notion of the quantile bias which operates on the domain of the score rather than the sample space. 

\begin{definition}\label{def::quantbias}
Let $f,X,G,\favdir_f$ and $F_k$ be as in Definition \ref{def::statbias}. Let $p \in (0,1)$.
\begin{itemize}[label=$\bullet$]
  
  \item The signed $p$-th quantile is defined by
  \[
\signbias^{Q}_p(f|X,G) = \big( F^{[-1]}_0(p)-F^{[-1]}_1(p) \big) \cdot \favdir_f
  \]
  \item The $p$-th quantile bias is defined by
  \[
  \begin{aligned}
  \bias^Q_p(f|X,G)=|\signbias^Q_p(f|X,G)|.
  \end{aligned} 
  \]
\end{itemize}
\end{definition}

As a counterpart to statistical parity, we also introduce quantile (geometric) parity.
\begin{definition}[\bf geometric parity]\label{def::geomparity} Let $f$ be a model and $G \in \{0,1\}$ the protected attribute.
\begin{itemize}[label=$\bullet$]
\item We say that the model $f$ satisfies $p$-th quantile (or geometric) parity if
\[
\bias^Q_{p}(f|X,G)=0.
\]
\item Let $t \in \RR$. The classifier $Y_t$ satisfies quantile (or geometric) parity if \[
\bias^Q_{p_0}(f|X,G)=0, \quad p_0=F_0(t).
\]
\end{itemize}
\end{definition}

%
%

\begin{figure}
\centering
\subfloat[\footnotesize Classifier and quantile biases.]{\includegraphics[width=0.36\textwidth]{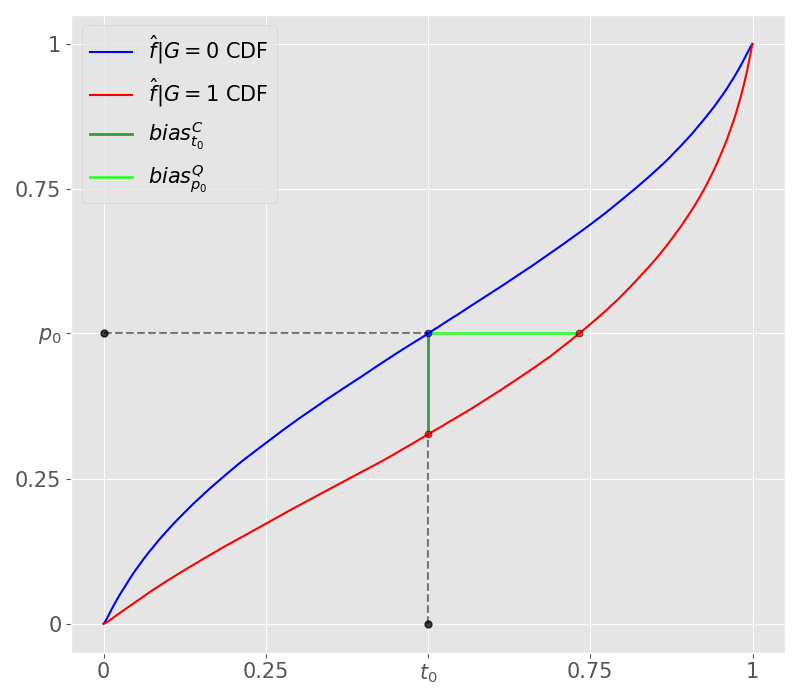}\label{fig::quantbias}}
~~
\subfloat[\footnotesize Model bias based on ${W_1}$ and $KS$.]{\includegraphics[width=0.36\textwidth]{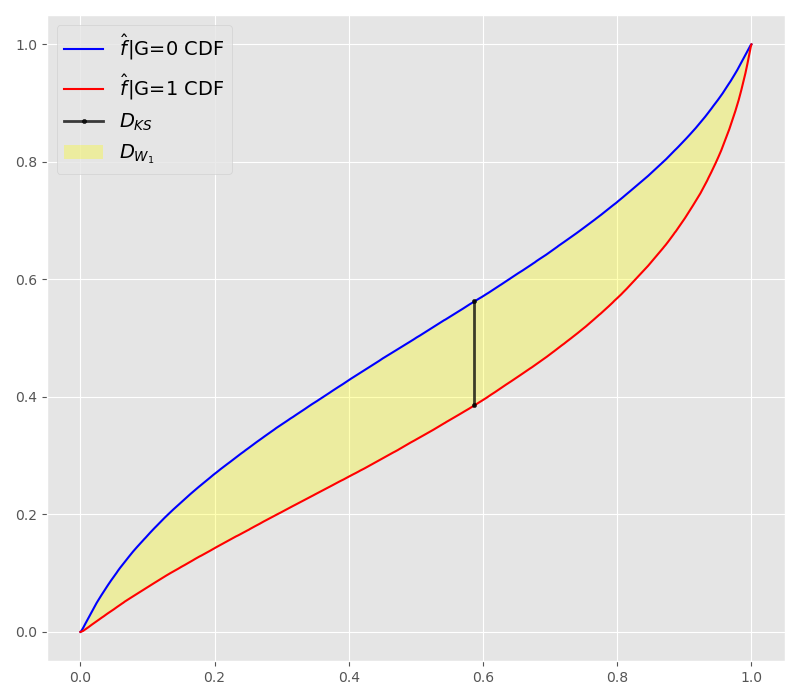}\label{fig::scoresCDFsimplemod}}
\caption{ \footnotesize Classifier and quantile bias, and model bias for the model \eqref{fair_score_model}. }
\end{figure}


Given a score $f$, the quantile bias measures the difference between subpopulation quantile values. For a given threshold $t$, the $p_0$-quantile signed bias, with $p_0=F_0(t)$, measures by how much the corresponding score values of the protected class $G=1$ differ from that of $G=0$ or equivalently by how much the threshold for the protected group should be shifted to achieve the quantile parity (and in some cases statistical parity) between the two populations.


\begin{lemma} Let $f$ be a model, $G \in \{0,1\}$ the protected attribute, and $G=0$ the non-protected class. Suppose that $t_0 \in \RR$ is a point at which the CDFs $F_0$ and $F_1$ are continuous and strictly increasing. Then $Y_{t_0}$ satisfies statistical parity if and only if it satisfies geometric parity.
\end{lemma}
\begin{proof}
The result follows from Definition \ref{def::statbias}, Definition \ref{def::quantbias}, and the fact that $F_0$ and $F_1$ are locally invertible at $t_0$.
\end{proof}

To better understand the classifier and quantile biases and their connection, see Figure \ref{fig::quantbias}. The conditional CDFs of the model scores are plotted given the protected attribute $G$. The blue line (corresponding to the scores given $G=0$) is above the red line (scores given $G=1$) for all values of $t$. Thus, for a given threshold $t_0$ we have that $F_0(t_0)-F_1(t_0)>0$, which means that if the favorable direction is $\uparrow$ ($\downarrow$) then the classifier favors the class $G=1$ ($G=0$). In view of the quantile bias, the green horizontal line segment represents the amount we would have to shift the threshold for one of the classes in order to achieve geometric parity. Since the CDFs are shown to be continuous and strictly increasing, the above lemma implies that doing so would achieve statistical parity as well.

\subsection{Optimal transport use in ML classifier fairness}\label{sec::opttransport}

\subsubsection{ Classifier bias mitigation via repaired datasets}

Two notable works that utilize optimal transport theory to reduce statistical parity bias are \citet{Feldman2015} and \citet{Gordaliza2020}.

The approach in \citet{Feldman2015} seeks to create an unbiased dataset by transforming predictors and then training a classifier on it. The authors propose a geometric repair scheme, which partially moves the two subpopulation distributions $\mu_{i,0}$  and $\mu_{i,1}$ of predictor $X_i$ along the Wasserstein geodesic towards their (unidimensional) Wasserstein barycenter $\mu_{i,B}$, a distribution minimizing the variance of the collection $\{\mu_{i,0}, \mu_{i,1}\}$; see Appendix \ref{sec::kpminimization}. The transformed dataset is then used to train a model that reduces disparate impact.

\citet{Gordaliza2020} proposes a method for transforming the multivariate distribution of predictors called random repair. Given two subpopulation distributions of predictors $\mu_k = P_{X|G=k}$, with $k\in \{0,1\}$, and a repair parameter $\lambda\in [0,1]$, the algorithm randomly chooses between the Wasserstein barycenter $\mu_B$ of $\{\mu_0,\mu_1\}$ and the original subpopulation distribution $\mu_k$, with $\lambda$ determining the probability of selecting $\mu_B$.

The authors establish the upper bound on the disparate impact (DI) and balanced error rate (BER) of classifiers with respect to $(X,G)$ using the total variance distance between the subpopulation distributions of predictors,
\[
\min_h BER(h,X,G) = \tfrac{1}{2}(1-d_{TV}(\mu_0,\mu_1)),
\]
and show that the $TV$-distance between repaired subpopulation distributions $\tilde{\mu}_{0,\lambda}$,$\tilde{\mu}_{1,\lambda}$ is bounded by $1-\lambda$. This, in turn, allows to control the bound on the DI and BER, and hence the closely related statistical parity bias on the repaired dataset is bounded by
\[
  \max_h \bias^C(h|G,\tilde{X}_{\lambda}) = d_{TV} (\tilde{\mu}_{0,\lambda},\tilde{\mu}_{1,\lambda})\le 1-\lambda.
\] 
The random repair algorithm allows for a tight control of TV-distance between repaired subpopulations unlike the geometric repair approach. They also establish bounds on the loss in performance due to modifying the data by the Wasserstein distance between the two subpopulation distributions of predictors. The performance loss is expressed as the difference in classification risk between the repaired and original data on $(X,G)$.

\begin{remark}\rm
Given the regulatory constraints, the approaches of \cite{Feldman2015} and \cite{Gordaliza2020} would not be permitted in financial institutions that extend credit because a) the protected attribute cannot be used in training or prediction, and b) introducing randomness into the input dataset is prohibited; for details see \citep{Hall2021}. To take into account the regulatory constraints and practical applications, in our companion paper \citep{Miroshnikov2021b} we propose a post-processing approach that relies on the fairness interpretability framework presented in the current article.

\end{remark}

\subsubsection{Individual fairness}
The work of \citet{Dwork2012} studies individual fairness of randomized classifiers. To understand the main results of the article, we first provide relevant definitions.

\begin{definition}\label{def::Drc} Let $(\Chi,d)$ be a metric space and $D$ a distance on $\mathscr{P}(\{0,1\})$.
\begin{itemize}[label=$\bullet$]
\item[$(i)$]  A map $M: \Chi \to \mathscr{P}(\{0,1\})$ is called a randomized classifier.

  \item[$(ii)$] $Lip_1(\Chi,\mathscr{P}(\{0,1\});d,D)=\{M: \Chi \to \mathscr{P}(\{0,1\}), \, D(M(x),M(y))\leq d(x,y) \}$. 

  \item[$(iii)$] Given $\nu \in \mathscr{P}(\Chi)$  the averaged $M_{\nu}$ is defined by $M_{\nu}(a)=\E_{x \sim \nu} [M(x)(a)]$, $a \subset \{0,1\}$.
\item [$(iv)$] The distance $D_{rc}$ between $\mu,\nu \in\mathscr{P}(\Chi)$ is defined by
\begin{equation*} 
D_{rc}(\mu,\nu; D,d) := \sup \Big\{ M_{\mu}(\{0\})-M_{\nu}(\{0\}),\,\, M \in Lip_1(\Chi,\mathscr{P}(\{0,1\}); d,D) \Big\}\in [0,1]. 
\end{equation*}
\end{itemize}
\end{definition}

Individual fairness is defined by imposing a Lipschitz property on the map $x \to M(x)\in \mathscr{P}(\{0,1\})$, $x\in \mathcal{X}$. As in \citet{Gordaliza2020}, the work of \citet{Dwork2012} relates the bias in the output to the bias in the input. In particular, the paper establishes the upper bound $D_{TV} (M_{P_0},M_{P_1})\le D_{rc}(P_0,P_1)$ for the statistical parity bias of Lipschitz randomized classifiers. Roughly speaking, the above bound means that when two subpopulations $P_0,P_1$ are ``similar'' in the sense of the $D_{rc}$ metric, then the Lipschitz condition ensures that the statistical parity bias is small.

The $D_{rc}$ metric has transport-like properties and is related to the Wasserstein metric; see \citet[Theorem 3.3]{Dwork2012} and Theorem \ref{thm::rbcbiaswasserstconn} in Section \ref{sec::IPMs}.

\section{Model bias metric}\label{sec::modbias}

In our work we shift the focus from measuring the bias in classifiers to the bias in regressor outputs. This is motivated by the fact that many strategies and decisions in the real-world make use of the regressor values or the classification scores of the trained ML models. Furthermore, in the case of classification scores, the bias assessment in FIs is carried out before any classifier threshold is determined.

In this section, we discuss how to measure the regressor bias using optimal transport. We also establish the connection between the regressor bias and the bias in the collection of classifiers induced by thresholding the regressor, and make use of this integral relationship to design generic regressor fairness metrics that incorporate group-based parity criteria, such as equalized odds \citep{Hardt2015}, into the transport formulation.

\begin{definition}[\bf $D$-model bias]\label{def::bias_joint} Let $X\in \RR^n$ be predictors, $f$ be a model, and $G\in\{0,1\}$ the protected attribute. Let $D(\cdot,\cdot)$ be a metric on the space of probability measures $\mathscr{P}_q(\RR)$, with $q \geq 0$. Provided $\E[|f(X)|^q]$ is finite, the $D$-based model bias is defined as the distance between the subpopulation distributions of the model:
\begin{equation}\label{def::modelbias}
\modbias_D(f|X,G) := D(P_{f(X)|G=0},P_{f(X)|G=1}),
\end{equation}
where $P_{f(X)|G=k}$ is the pushforward probability measure of $f(X)|G=k$. We say that the model $(X,f)$ is {\it fair} up to the $D$-based bias $\epsilon \geq 0$  if $\modbias_{D}(f|X,G)\leq \epsilon$. 
\end{definition}

Figure \ref{fig::scoresCDFsimplemod} illustrates the model bias for two choices of $D$: the 1-Wasserstein metric $W_1$ and the Kolmogorov-Smirnov distance $KS$. Notice the stark difference between the two model biases. This raises the general question on which metric should one use to evaluate the bias. We discuss this issue in the following subsection.

In what follows we suppress the explicit dependence of the model bias on $X$.

\subsection{Wasserstein distance}
To determine an appropriate metric $D$ to be used in \eqref{def::modelbias} is not a trivial task. The choice depends on the context in which the model bias is measured. We argue that it is desirable for the metric to have the following properties:
\begin{itemize} 
\item [(\namedlabel{metricprop1}{P1})] It should  be continuous with respect to the change in the geometry of the distribution. 
\item [(\namedlabel{metricprop2}{P2})] It should be non-invariant with respect to monotone transformations of the  distributions.
\end{itemize}


The property \eqref{metricprop1} makes sure that the metric keeps track of changes in the geometry. For instance, suppose an ``income'' of the group $\{G=0\}$ is $x_0$ and that of $\{G=1\}$ is $x_1$. A metric that measures income inequality should be able to sense the distance between $x_0$ and $x_0+\eps$. That is, having two delta measures $\delta_{x_0}$ and $\delta_{x_0+\eps}$ the metric must ensure that as $\eps \to 0$ the distance $D(\delta_{x_0},\delta_{x_0+\eps})$ approaches zero.  The property \eqref{metricprop1} also makes sure that slight changes in the subpopulation distributions lead to a slight change in bias measurements, which is important for stability with respect to changes in the dataset $X$. 

The property \eqref{metricprop2} makes sure that the metric is non-invariant with respect to monotone transformations. That is, given two random variables $X_0$ and $X_1$ and a continuous, strictly increasing  transformation $T:\RR \to \RR$, one would expect the change in distance between $T(X_0)$ and $T(X_1)$ whenever $T$ is not a shift. For example, if $T(x)=\alpha  x$, we would expect the distance between $T(X_0)=\alpha X_0$ and $T(X_1)=\alpha X_1$ depend  continuously on $\alpha$.

In what follows, we consider the Wasserstein distance $W_q$ as a potential candidate for fairness interpretability; for use cases in the ML fairness community see \citet{Dwork2012, Feldman2015, Gordaliza2020}.

To introduce the metric and investigate its properties we switch our focus to probability measures; recall that any random variable $Z$ gives rise to the pushforward probability measure $P_Z(A)=\PP(Z \in A)$ on $\RR$, and the reverse is true, for any $\mu \in \mathscr{P}(\RR)$ with the CDF $F_{\mu}(a)=\mu((-\infty,a])$ there is a random variable $Z$ such that $P_Z=\mu$. Similar remarks apply for random vectors; see \citet{Shiryaev}. Given $T:\RR^k \to \RR^m$ and $\mu\in \mathscr{P}(\RR^k)$, we denote by $T_{\#}\mu$ a measure such that $T_{\#}\mu(B)=\mu\big(T^{-1}(B))$.

The Wasserstein distance $W_q$ is connected to the concept of optimal mass transport. Given two probability measures $\mu_1,\mu_2 \in \mathscr{P}_q(\RR)$ with finite $q$-th moment and the cost function $c(x_1,x_2)=|x_1-x_2|^q$, the Wasserstein distance $W_q$ is defined by 
\[
\begin{aligned}
  W_q(\mu_1,\mu_2) := \mathscr{T}^{1/q}_{|x_1-x_2|^q}(\mu_1,\mu_2)
\end{aligned}
\]
where 
\[
  \mathscr{T}_{|x_1-x_2|^q}(\mu_1,\mu_2) = \inf_{\gamma \in \mathscr{P}(\RR^2)} \bigg\{ \int_{\RR^2} |x_1-x_2|^q \, d\gamma(x_1,x_2), \,\, \text{with marginals $\mu_1,\mu_2$}  \bigg\}
\]
is the minimal cost of transporting the distribution $\mu_1$ into $\mu_2$, and vice versa in view of the symmetry of the cost function. A joint probability measure $\gamma \in \mathscr{P}(\RR^2)$ with marginals $\mu_1$ and $\mu_2$ is called a {\it transport plan}. It specifies how each point $x_1$ from ${\rm supp(\mu_1)}$ gets distributed in the course of the transportation; specifically, the transport of $x_1$ is described by the conditional probability measure $\gamma_{x_2|x_1}$. 

It can be shown that the Wasserstein metric for probability measures on $\RR$ can be expressed in terms of the quantile functions 
\begin{equation}\label{distWq}
W_q(\mu_1,\mu_2) = \bigg( \int_0^1 |F_{\mu_1}^{[-1]}(p)-F_{\mu_2}^{[-1]}(p)|^q\, dp \bigg)^{1/q}, 
\end{equation}
which makes the computation straightforward; see Theorem \ref{thm::transportprop}.

To get an understanding of the behavior of $W_q$, consider two delta measures located at $x_0$ and $x_0+\eps$, respectively. By definition of the metric it follows that
\[
W_q(\delta_{x_0},\delta_{x_0+\eps})=\eps.
\]
Thus, $W_q$ is continuous with respect to a shift of a point mass. Furthermore, for any two random variables $X_0$ and $X_1$ and $\alpha > 0$
\[
\begin{aligned}
W_q(P_{\alpha X_0},P_{\alpha X_1})=\alpha W_q(P_{X_0},P_{X_1})
\end{aligned}
\]
which implies that a multiplicative map $T(x)=\alpha x$ affects the Wasserstein distance. 

To formally show that properties (P1) and (P2) are satisfied by the Wasserstein metric, we provide the following theorem.

\begin{theorem}\label{thm::geomcontWass}
The distance $W_q$ satisfies:
\begin{itemize}
  \item [(a)] $W_q$ on $\mathscr{P}_q(\RR)$ is continuous with respect to the geometry of the distribution.

 \item [(b)] Let $T:\RR \to \RR$ be a continuous, strictly increasing map. $W_q$ is non-invariant under $T$, provided, $T(x) \neq x+C$ and $T_{\#}\mu \in \mathscr{P}_q(\RR)$, $\mu \in \mathscr{P}_q(\RR)$ . 
\end{itemize}
\end{theorem}
\begin{proof}
See Appendix \ref{app::auxlemm}.
\end{proof}

Theorem \ref{thm::geomcontWass} states that the Wasserstein metric relies on the geometry of the distribution. In particular, the distance is affected in a continuous way by the change in the geometry of the distribution. This, in turn, provides the desired sensitivity of the Wasserstein metric with respect to slight changes in the dataset distribution, including shifts, which is relevant for ML models with ragged CDFs, which makes the Wasserstein metric an appropriate candidate for the model bias measurement. In addition, as we will see, the Wasserstein distance enables us to assess the favorability at the level of the model, which is useful for applications in financial institutions.

\subsection{Negative and positive flows under order preserving optimal transport plan}\label{subsec::posnegflow}


We now provide several properties of the Wasserstein metric, which we employ in the following sections. 

Given two probability measures $\mu_1,\mu_2 \in \mathscr{P}_q(\RR)$, it can be shown that the joint probability measure  $\pi^* \in \mathscr{P}(\RR^2)$ with the CDF 
\begin{equation}\label{monotonplanmod}
F_{\pi^*}(a,b)=\min(F_{\mu_1}(a),F_{\mu_2}(b))
\end{equation}
is an {\it optimal transport plan} for transporting  $\mu_1$ into $\mu_2$ with the cost function $c(x_1,x_2)=|x_1-x_2|^q$, and thus,
\begin{equation}\label{opttranspplan}
W_q^q(\mu_1,\mu_2) = \mathscr{T}_{|x_1-x_2|^q}(\mu_1,\mu_2) = \int_{\RR^2} |x_1-x_2|^q d\pi^*(x_1,x_2).
\end{equation}

\begin{figure}
\centering
\subfloat[\footnotesize Optimal transport map $T^*$.]{\includegraphics[width= 0.36\textwidth]{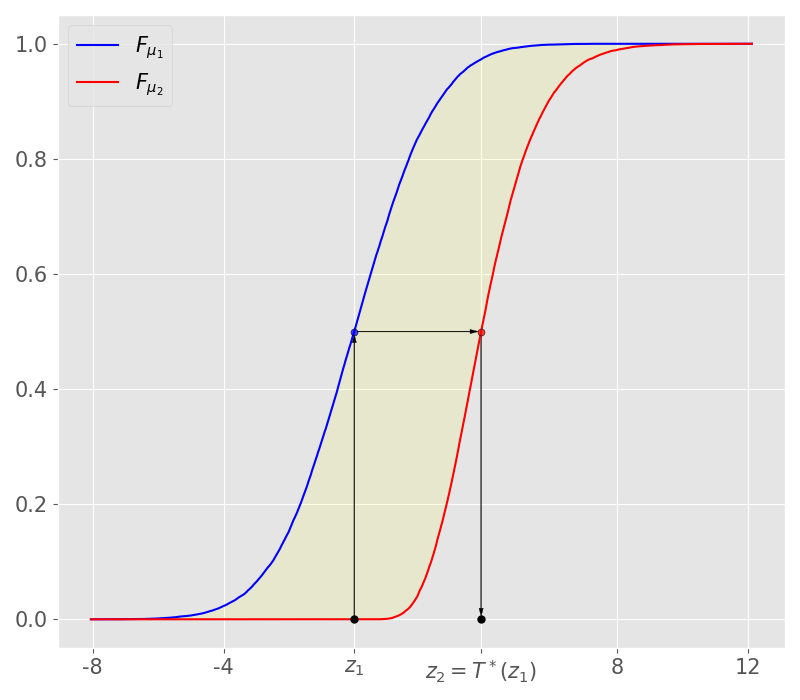}\label{fig::transportatomless}}
~~
\subfloat[\footnotesize  Transport flow of $\pi^*$.]{\includegraphics[width= 0.36\textwidth]{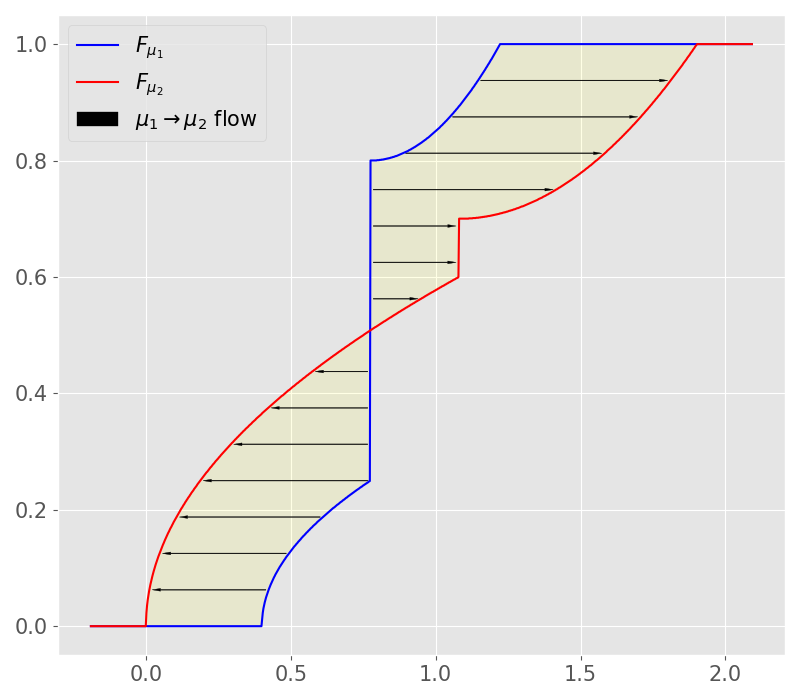}\label{fig::transportgeneric}}
\caption{\footnotesize Transporting $\mu_1$ to $\mu_2$ under the monotone transport plan $\pi^*$.}
\end{figure}


\noindent Most importantly, $\pi^*$ is the only monotone (order preserving) transport plan such that 
\[
(x_1,x_2),(x_1',x_2')\in{\rm supp}(\pi^*), \quad x_1<x_1' \quad \Rightarrow \quad  x_2 \leq x_2'.
\]
In a special case, when $\mu_1$ is atomless, $\pi^*$ is determined by the monotone map 
\begin{equation}\label{tranportmap}
T^*=F^{[-1]}_{\mu_2} \circ F_{\mu_1},
\end{equation}
called an optimal transport map. Specifically, each point $x_1$ of the distribution $\mu_1$ is transported to the point $x_2=T^*(x_1)$; see Figure \ref{fig::transportatomless} for an illustration. Thus, $\mu_2=T^*_{\#}\mu_1$, and the conditional probability measure $\pi^*_{x_2|x_1}=\delta_{T^*(x_1)}$ for $x_1 \in {\rm supp}(\mu_1)$. In this case, \eqref{opttranspplan} reads
\begin{equation}\label{opttranspplanatomless}
W_q^q(\mu_1,\mu_2) = \mathscr{T}_{|x_1-x_2|^q}(\mu_1,\mu_2) = \int_{\RR} |x_1-T^*(x_1)|^q d\mu_1(x_1).
\end{equation}
The results \eqref{monotonplanmod}-\eqref{opttranspplanatomless} follow from Theorem \ref{thm::transportprop} for the cost function $c(x_1,x_2)=|x_1-x_2|^q$.

In a general case, under the transport plan $\pi^*$, points $x_1 \in{\rm supp(\mu_1)}$ for which $\mu_1(\{x_1\})=0$ are transported as a whole, while the ``atoms'', points $x_1$ for which $\mu_1(\{x_1\})>0$, are allowed to be split or spread along $\RR$; see Figure \ref{fig::transportgeneric} that illustrates the transport flow under $\pi^*$ in the general case. The plot also provides a depiction of the order preservation; notice how the arrows do not intersect.

To compute the portion of the transport cost used for moving points of $\mu_1$ to the right or left, it is necessary to restrict the attention to the regions $x_1 < x_2 $ and $x_1>x_2$, respectively.

\begin{lemma}\label{lmm::W1decomp}
Let $\mu_1,\mu_2 \in \mathscr{P}_q(\RR)$, $q \in [1,\infty)$. Under the monotone plan $\pi^*$ the transport efforts to the left and right for the cost function $c(x_1,x_2)=|x_1-x_2|^q$ are given by:
\begin{equation}\label{leftrighteffort}
\begin{aligned}
\mathscr{T}_{|x_1-x_2|^q}^{\lrarrow}(\mu_1,\mu_2)&= \int_{\{ \pm(x_2-x_1)>0\}} |x_1-x_2|^q d\pi^*(x_1,x_2)\\
&=\int_{{\big\{\pm({F^{[-1]}_{\mu_2}(p)-F^{[-1]}_{\mu_1}(p)})>0 \big\}} } |F^{[-1]}_{\mu_1}(p)-F^{[-1]}_{\mu_2}(p)|^q dp.
\end{aligned}
\end{equation}
Hence, the Wasserstein distance $W_q$ can be expressed as
\begin{equation}\label{wassdisttwoflows}
W_q(\mu_1,\mu_2) = \big( \mathscr{T}_{|x_1-x_2|^q}^{\leftarrow}(\mu_1,\mu_2) + \mathscr{T}_{|x_1-x_2|^q}^{\rightarrow}(\mu_1,\mu_2) \big )^{1/q}.
\end{equation}

Furthermore, if $\mu_1$ is atomless, \eqref{leftrighteffort} reads
\begin{equation}\label{leftrighteffortspecial}
\begin{aligned}
\mathscr{T}_{|x_1-x_2|^q}^{\lrarrow}(\mu_1,\mu_2) &=\int_{{\big\{\pm({T^*(x_1)-x_1})>0 \big\}} } |x_1-T^*(x_1)|^q \, d\mu_1(x_1), \quad T^*=F^{[-1]}_{\mu_2}\circ F_{\mu_1}
\end{aligned}
\end{equation}

\end{lemma}

\begin{proof} By  \eqref{monotonplanmod} the monotone plan can be expressed as
\begin{equation*}
\pi^* = (F^{-1}_{\mu_1},F^{-1}_{\mu_2})_{\#} \lambda|_{[0,1]} \in \mathscr{P}(\RR^2)
\end{equation*}
where $\lambda|_{[0,1]}$ denotes the Lebesgue measure restricted to $[0,1]$. Then, by Proposition \ref{prop::changeofvar}, for any Borel set $B \subset \RR^2$  we have
\[
\int_{B} |x_1-x_2|^q d\pi^*(x_1,x_2) = \int_{\{ p\in(0,1): \, (F_{\mu_1}^{[-1]}(p),F_{\mu_2}^{[-1]})(p)) \in B \}} |F^{[-1]}_{\mu_1}(p)-F^{[-1]}_{\mu_2}(p)|^q dp .
\]
Then \eqref{leftrighteffort} follows from the above identity with $B=\{(x_1,x_2): \pm(x_1-x_2)>0\}$. Next, by \eqref{opttranspplan} and \eqref{leftrighteffort}, we obtain \eqref{wassdisttwoflows}. 

Finally, if $\mu_1$ is atomless, by Theorem \ref{thm::transportprop} the monotone plan $\pi^*=(I,T^*)_{\#}\mu_1$, where $T^*$ is the optimal transport map given by \eqref{tranportmap}. Then using Proposition \ref{prop::changeofvar} we obtain \eqref{leftrighteffortspecial}.
\end{proof}

\subsection{$W_1$-based model bias and its components}\label{subsec::W1modbias}

For $q=1$ the Wasserstein distance $W_1$ is known as the {\it Earth Mover distance}. Since the distance is symmetric, $\Bias_{W_1}(f|X,G)$ is the cost of transporting the distribution of $f(X)|G=0$ into that of $f(X)|G=1$ or vice versa.

It can be shown that the $W_1$-based model bias formulation is consistent with both statistical parity fairness criterion as well as quantile parity criterion, which is shown by the following theorem.

\begin{lemma}\label{thm::modbiasconnections}
Let $f$ be a model and $G\in\{0,1\}$ the protected attribute. Then 
\begin{equation*}
  \begin{aligned}
\Bias_{W_1}(f|G) & = \int_0^1 \bias^Q_p(f|G) \, dp = \int_{\RR} \bias^C_t(f|G) \,dt.\\
  \end{aligned}
\end{equation*}
\end{lemma}
\begin{proof} By assumption $\E|f(X)|<\infty$ and hence $\E[|f(X)|G=k|]<\infty$ for $k\in\{0,1\}$. Then,  we have (\citet{Shorack1986})
\[
\begin{aligned}
W_1\big(f(X)|G=0,f(X)|G=1\big) &=\int_0^1 |F_{f(X)|G=0}^{[-1]}(p)-F_{f(X)|G=1}^{[-1]}(p)|\,dp\\
&= \int_{\RR} |F_{f(X)|G=0}(t)-F_{f(X)|G=1}(t)|\,dt \, < \, \infty.
\end{aligned}
\]
Hence, the result follows from Definition \ref{def::statbias}, Definition \ref{def::quantbias}, and the above equality.
\end{proof}
\begin{remark}
  The above lemma establishes the representation of the model bias as an integration over the statistical parity bias of classifiers obtained by considering all thresholds. Here, the consistency of the model bias with statistical parity is understood in the sense of the equality in the above lemma. In comparison, \citet{Dwork2012} establishes a connection of statistical parity of Lipschitz randomized classifiers and subpopulations in a dataset upon which the models are built.

  While the results in \citet{Dwork2012} do not imply the above lemma, it is appealing to provide a connection between the two. For example, consider the triplet $(X,G,Y)$ with $Y\in \{0,1\}$ and a smooth regressor $f(X) = P(Y=1|X)$. Consider a randomized classifier $z\to \mu_z$ where $z=(x,g,y)$, and $\mu_z(1) = f(x)$. Let $P_g = P_{Z|G=g}$. Then, the upper bound on statistical parity bias of $\mu_z$ provided by \citet{Dwork2012} reads 
  \[
  D_{TV}(\mu_{P_0},\mu_{P_1}) = |\E[f(X)|G=0]-\E[f(X)|G=1]| \leq W_1(P_0,P_1),
  \]
  which illustrates the difference between Lemma 3.1 of \citet{Dwork2012} and our lemma.
\end{remark}

\paragraph{Positive and negative model bias.}  According to Lemma \ref{lmm::W1decomp}, the cost of transporting a distribution is the sum of the transport effort to the left and the transport effort to the right. This motivates us to define the positive bias as the transport effort for moving the particles of $f(X)|G=0$ in the non-favorable direction and the negative bias as the transport effort in the favorable one; equivalently the latter is the transport effort for moving  the particles of $f(X)|G=1$ into the favorable direction and the former is the transport effort into the non-favorable one. 

Motivated by Lemma \ref{lmm::W1decomp} we  define positive and negative model biases as follows:
\begin{definition}\label{def::posnegmodbias}
Let $f,G,\favdir_f$ and $F_k$ be as in Definition \ref{def::statbias}.
\begin{itemize}[label=$\bullet$]
  \item The positive and negative $W_1$ based model biases are defined by
\begin{equation*}
\begin{aligned}
\Bias_{W_1}^{\pm}(f|G) = \int_{\mathcal{P}_{\pm}} \pm(F_0^{[-1]}(p)-F_1^{[-1]}(p)) \cdot \favdir_f \, dp 
\end{aligned}
\end{equation*}
where
\begin{equation*}
\mathcal{P}_{\pm} =\Big\{ p \in (0,1):\,\, \pm\signbias_p^Q(f|G)=\pm(F_0^{-1}(p)-F_1^{-1}(p)) \cdot \favdir_f > 0 \Big\}.
\end{equation*}
In this case, the model bias is disaggregated as follows:
\begin{equation*}
\Bias_{W_1}(f|G)=\Bias_{W_1}^+(f|G)+\Bias_{W_1}^-(f|G).
\end{equation*}

\item The net model bias is defined by
\begin{equation*}
\Bias_{W_1}^{net}(f|G)=\Bias_{W_1}^+(f|G)-\Bias_{W_1}^-(f|G).
\end{equation*}
\end{itemize}
\end{definition}

We next establish that the positive and negative $W_1$ model biases can be expressed in terms of classifier biases. To establish this, we first prove the following auxiliary lemma.

\begin{lemma}\label{lmm::cdfquantconn}
Let $X_0,X_1$ be random variables with $\E|X_i|<\infty$, $i \in \{0,1\}$. Let $F_i$ denote the CDF of $X_i$ and let
\begin{equation*}
\begin{aligned}
\Tcal_0&=\{t \in \RR: F_1(t) < F_0(t) \},& \Tcal_1&=\{t \in \RR: F_0(t)<F_1(t) \}\\
\Pcal_0&=\{ p\in(0,1): F^{[-1]}_1(p) < F_0^{[-1]}(p) \},& \Pcal_1&=\{p \in (0,1): F_0^{[-1]}(p)<F_1^{[-1]}(p) \}.
\end{aligned}
\end{equation*}
Then
\begin{equation*}
\begin{aligned}
0 \leq \int_{\Tcal_0} F_0(t)-F_1(t) \, dt &= \int_{\Pcal_1} F_1^{[-1]}(p)-F_0^{[-1]}(p) \, dp  < \infty\\
0 \leq  \int_{\Tcal_1} F_1(t)-F_0(t) \, dt  & = \int_{\Pcal_0} F_0^{[-1]}(p)-F_1^{[-1]}(p) \, dp < \infty. \\
\end{aligned}
\end{equation*}
\end{lemma}
\begin{proof}
See Appendix \ref{app::auxlemm}.
\end{proof}

\begin{theorem}\label{thm::intposnegmodbias}
Let $f,G,\favdir_f$, $\mathcal{P}^{\pm}$ and $F_k$ be as in Definition \ref{def::posnegmodbias}. Then
\begin{equation}\label{intposnegmodbias}
\begin{aligned}
\modbias_{W_1}^{\pm}(f|G) &= \int_{\mathcal{P}_{\pm}} \bias^{Q}_p(f|G) \,dp = \int_{\mathcal{T}_{\pm}} \bias^{C}_t(f|G) \, dt \\
\end{aligned}
\end{equation}
where
\begin{equation*}
\mathcal{T}_{\pm} =\big\{ t \in \RR:\,\, \pm\signbias_t^C(f|G)=\pm (F_1(t)-F_0(t)) \cdot \favdir_f >0 \big\}.
\end{equation*}
The net bias satisfies
\begin{equation}\label{netmodbias}
\begin{aligned}
\modbias_{W_1}^{net}(f|G) &= \int_{0}^1 \signbias^Q_p(f|G)dp = \int_{\RR} \signbias^C_t(f|G) \, dt \\
& = \big( \E[f(X)|G=0]-\E[f(X)|G=1] \big) \cdot \favdir_f
\end{aligned}
\end{equation}
\end{theorem}
\begin{proof} Suppose first that favorable direction is $\uparrow$. Since $\E|f(X)|<\infty$, we have $\E[|f(X)||G=k]<\infty$ for $k\in\{0,1\}$.  Then
by Lemma \ref{lmm::cdfquantconn} 
\[
\Bias^{\pm}\big(f|G\big)=\pm\int_{\Pcal^{\pm}} F_{f|G=0}^{[-1]}(p)-F_{f|G=1}^{[-1]}(p) \,dp =\pm \int_{\Tcal^{\pm}} F_{f|G=1}(t)-F_{f|G=0}(t) \,dt < \infty.
\]
Hence \eqref{intposnegmodbias} follows from Definition \ref{def::statbias}, Definition \ref{def::quantbias}, and the above equality. 

Next, by \eqref{intposnegmodbias} and Lemma \ref{lmm:expectviacdf} we have
\[
\begin{aligned}
\Bias^{net}(f|G) & = \Bias^+(f|G)-\Bias^-(f|G) \\
& = \int_{\Tcal^+} \big( F_{f|G=1}(t)-F_{f(X)|G=0}(t) \big) dt - \int_{\Tcal^-} \big( F_{f|G=0}(t)-F_{f|G=1}(t) \big) dt \\
& = \int_{-\infty}^0 \big( F_{f|G=1}(t)-F_{f|G=0}(t) \big) dt + \int_{0}^{\infty} \big( (1-F_{f|G=0}(t))-(1-F_{f|G=1}(t)) \big) dt \\
& = \E[f(X)|G=0]-\E[f(X)|G=1].
\end{aligned}
\]
This proves \eqref{netmodbias}. If the favorable direction is $\downarrow$, the proof of \eqref{intposnegmodbias} and \eqref{netmodbias} is similar.
\end{proof}

In the context of classification, Theorem \ref{thm::intposnegmodbias}  states that the positive $W_1$-based model bias is the integrated classifier bias over the set of thresholds $t \in \mathcal{T}_+$ where the classifiers $Y_t=\1_{\{f(X>t\}}$ favor the non-protected class $G=0$. Similar remark holds for the negative model.

Furthermore, the property \eqref{intposnegmodbias} of $\Bias_{W_1}^{\pm}$ allow one to use thresholds and quantiles interchangeably, which is beneficial in classification problems. For this reason, we choose $W_1$ as our primary metric.


\begin{figure}
\centering
\subfloat{\includegraphics[width=0.36\textwidth]{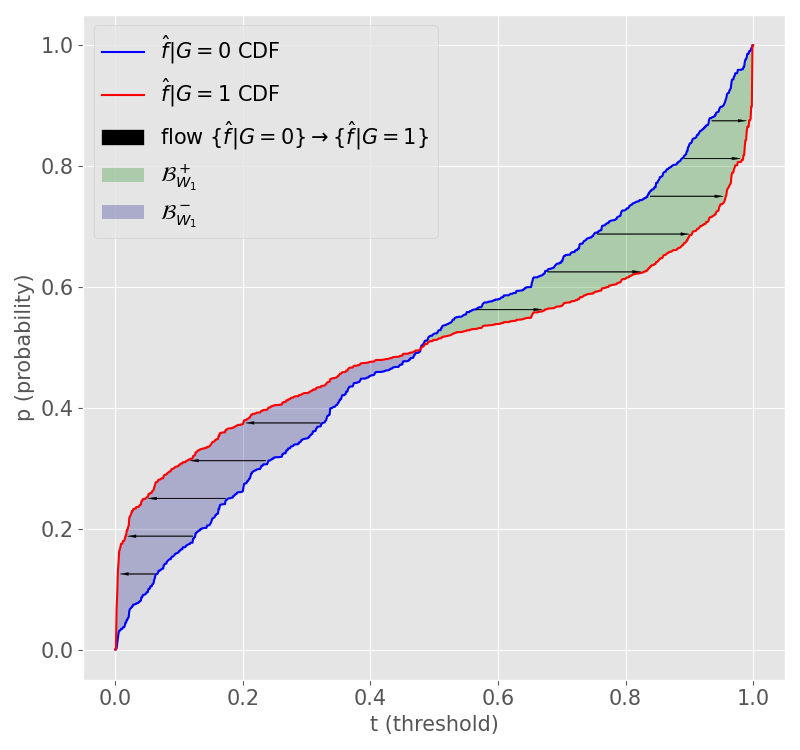}}
\caption{ \footnotesize Positive and negative model biases for the trained XGBoost  model \eqref{difvarpred}, $\favdir_f=-1$. }\label{fig::posnegmodelbias}
\end{figure}


\paragraph{Example.} To understand the statement of  Theorem \ref{thm::intposnegmodbias} consider the following classification risk model ($\favdir_f=-1$) with a predictor whose variance depends on the attribute $G$:
\begin{equation}\label{difvarpred}\tag{M2}
\begin{aligned}
&X \sim N(\mu,(1+G)\sqrt{\mu}), \quad \mu=5\\
&Y \sim Bernoulli(f(X)), \quad f(X)=\P(Y=1|X)={\sigma(\mu-X)}.
\end{aligned}
\end{equation}
which leads to the presence of both positive and negative bias components in the score distribution. Figure \ref{fig::posnegmodelbias} depicts the subpopulation score CDFs of the trained GBM classifier and illustrates the fact that the integrated positive quantile and classifier biases yield the positive model bias (green region), and a similar relationship holds for the negative model bias (purple region). The monotone transport flows are also depicted, showing the connection between the signed model bias and the favorability. Since $\favdir_f=-1$, in the green region the non-protected class is transported towards the non-favorable direction, while in the purple region it is transported towards the favorable one. 

\paragraph{On renormalization of model bias.} If $f(X)$ is a classification score then $\Bias_{W_1}(f|G) \in [0,1]$, which makes it easy to interpret the amount of the bias in the model distribution. 

For regressors, however, the model bias can take any value in $[0,\infty)$. One approach is to normalize the model bias as follows. First, pick an appropriate reference scale $L>0$ corresponding to the response variable. Given the scale $L$ one can define a generalized Wasserstein-based model bias as follows:
\begin{equation}\label{genbias}
Bias_{g,W_1}(f|G) = g\Big( \frac{1}{L}\Bias_{W_1}(f|G) \Big) 
\end{equation}
where the link function $g$ is strictly increasing and satisfies
\[
g(x)=\left\{
\begin{aligned}
&x, \quad x\in[0,0.5]\\
&\text{$g$ increases to $1$}.
\end{aligned}
\right.
\]

Having this setup yields $Bias_{g,W_1}(f|G)=\frac{1}{L}Bias_{W_1}(f|G)$ whenever the transport effort is within the scale of interest $L$, that is, when $Bias_{W_1}(f|G) \leq \frac{L}{2}$. In practice, for bounded distributions, one can pick $L={\rm supp}\, P_{f(X)}$, while for unbounded distributions one can pick $L=2 \sigma(f(X))$. 

In our work, we develop the bias explanation methods to explain the actual amount of transport effort between subpopulations. The generalization to \eqref{genbias} is trivial. 


\subsection{Generalized group-based parity model bias}\label{sec::genparityconsist}

In this section, we will generalize the notions of the Wasserstein-based bias
to the case of generic group-based parity for protected attributes with multiple classes. 
We then apply the generalization to the equalized odds and the equal opportunity parity conditions.

\begin{definition}\label{def::genparity}
Let $f$ be a model, $X \in \RR^n$ predictors, $G \in \{0,1,\dots,K-1\}$ protected attribute, $G=0$ non-protected class, 
and  $\favdir_f$ the sign of the favorable direction of $f$. 
Let $\mathcal{A}=\{A_1,\dots,A_M\}$ be a collection of disjoint subsets of the sample space $\Omega$. Define events
\[
A_{km}=\{G=k\} \cap A_m, \quad k \in \{0,1,\dots,K-1\}, \quad m \in \{1,\dots,M\}.
\]
\begin{itemize}
  \item [(i)] We say that $Y_t= \1_{\{f(X)>t\}}$ satisfies $\mathcal{A}$ group-based parity if
\begin{equation}\label{genparity}
\PP(Y_t=\1_{\{\favdir_f=1\}}|A_{km})=\PP(Y_t=\1_{\{\favdir_f=1\}}|A_{0m}), \quad k \in \{1,\dots,K-1\}, \quad m \in \{1,\dots,M\}.
\end{equation}

\item [(ii)] $(W_1,\mathcal{A})$-based (weighted) model bias is defined by
\begin{equation*}
\Bias_{W_1,\mathcal{A}}^{(w)}(f|G) = \sum_{k=1}^{K-1}\sum_{m=1}^{M} w_{km} \Bias_{W_1}(f|\{A_{0m},A_{km}\}), \quad  w_{km} \geq 0,
\end{equation*}
where the weights satisfy $\sum_{k=1}^{K-1}\sum_{m=1}^{M} w_{km}=1$.

\item [(iii)]
The positive and negative  $(W_1,\mathcal{A})$ weighted model biases are defined by
\begin{equation*}
\Bias_{W_1,\mathcal{A}}^{(w)\pm}(f|G) = \sum_{k,m} w_{km} \Bias_{W_1}^{\pm}( f| \{A_{0m},A_{km}\}).
\end{equation*}
\end{itemize}
\end{definition}

\begin{lemma}
Let $G$ and $\mathcal{A}$ be as in Definition \ref{def::genparity}.
The $(W_1,\mathcal{A})$ model bias is consistent with the generic parity criterion \eqref{genparity} as given by the following:
\begin{equation*}
\begin{aligned}
\Bias_{W_1,\mathcal{A}}(f|G) &= \sum_{k,m} w_{km} \int_0^1 |F^{[-1]}_{f|A_{0m}}-F^{[-1]}_{f|A_{km}}| \, dt
\\
&=\sum_{k,m} w_{km} \int_{\RR} | \PP(Y_t=\1_{\{\favdir_f=1\}}|A_{km})-\PP(Y_t=\1_{\{\favdir_f=1\}}|A_{0m}) | \, dt.
\end{aligned}
\end{equation*}

Similarly, the signed model biases can be expressed
\begin{equation*}
\begin{aligned}
\Bias_{W_1,\mathcal{A}}^{(w)\pm}(f|G) &:= \sum_{k,m} w_{km} \int_{\Pcal_{km\pm}} \pm \big(F^{[-1]}_{f|A_{0m}}(p)-F^{[-1]}_{f|A_{km}}(p)\big) \cdot \favdir_f \ \, dp
\\
&=\sum_{k,m} w_{km} \int_{\Tcal_{km \pm }} | \PP(Y_t=\1_{\{\favdir_f=1\}}|A_{km})-\PP(Y_t=\1_{\{\favdir_f=1\}}|A_{0m}) | \, dt,
\end{aligned}
\end{equation*}
where
\begin{equation*}
\begin{aligned}
\Pcal_{km \pm } &= \Big\{p \in [0,1]: \pm\big(F^{[-1]}_{f|A_{0m}}(p)-F^{[-1]}_{f|A_{km}}(p)\big) \cdot \favdir_f > 0 \Big\}\\
\Tcal_{km \pm } &= \Big\{t \in \RR: \pm\big(F_{f|A_{km}}(t)-F_{f|A_{0m}}(t)\big) \cdot \favdir_f > 0 \Big\}.
\end{aligned}
\end{equation*}
\end{lemma}

\begin{proof}
The claim follows directly from Theorem \ref{lmm::cdfquantconn}.
\end{proof}

\vspace{5pt}

\noindent{\it Example.} Suppose that the favorable direction is $\uparrow$. Suppose that $G \in \{0,1\}$ and that the response variable $Y \in \{0,1\}$. Let $
\mathcal{A}=\{ \{Y=0\}, \{Y=1\} \}$. In that case, the group-based parity condition \eqref{genparity} reads
\[
\PP(Y_t=1|G=0,Y=m)=\PP(Y_t=1|G=1,Y=m), \quad m=0,1,
\]
which is the equalized odds criterion; \cite{Hardt2015}. Then apply the above Lemma.


\subsection{Integral probability metrics for fairness assessment}\label{sec::IPMs}
When assessing fairness of model regressors, it is crucial to pick an appropriate metric because the model output is often used to make decisions. A wide class of candidate metrics could be integral probability metrics (IPMs). These provide a notion of ``distance'' between probability distributions and are designed as generalizations of the Kantorovich-Rubinstein variational formula. They can be defined directly using variational formulas \citep{Muller1997,Sriperumbudur2009}. Specifically, IPMs can be defined by maximizing the difference of expected values over a function space $\A$,
\begin{equation}\label{IPMdef}
W_{\A}(\nu_0, \nu_1) := \sup_{\varphi\in \A}\left\{\int\varphi(x)\,\nu_0(dx) - \int\varphi(x)\,\nu_1(dx)\right\},
\end{equation}
where $\nu_0,\nu_1 \in \mathscr{P}(\Chi)$ and $(\Chi,d)$ is a metric space.
For example, the Wasserstein metric can be obtained by taking $\A = \{ \varphi : [\varphi]_{Lip} \le 1 \}$ in \eqref{IPMdef}, where $[\varphi]_{Lip}$ is the Lipschitz constant of the function $\varphi$; The Dudley metric is obtained by taking $\A = \{\varphi : [\varphi]_{Lip} + \|\varphi\|_{\infty} \le 1 \}$. Dropping the regularity of test functions leads to a discontinuous response to shifting of delta masses. For example, by setting $\A = \{ \varphi : \|\varphi\|_{\infty} \le 1 \}$, one obtains the total variation metric $D_{TV}$. An interesting aspect of the above variational formula is that it can be generalized to include a broader family of distances between probability distributions, namely divergences such as the Kullback-Leibler divergence; see \citet{Birrell2020} for more information. 

Thus, IPMs with regular test functions serve as good candidates for assessing the fairness of the regressor via formula \eqref{def::modelbias}. One of the interesting contenders is $W_{\A^*}$ where $\A^* := \{ \varphi : \|\varphi\|_{\infty} \leq \tfrac{1}{2}, [\varphi]_{Lip}\leq 1 \}$, which is an equivalent metric to the Dudley metric and has the appealing property that its values are in the unit interval. $W_{\A^*}$ provides meaning in fairness assessment, as it could be expressed via a supremum over all ``agents'' in the form of regular randomized classifiers that detect the differences between two probability subpopulations. Specifically, it can be shown that $W_{\A^*}$ coincides with the $D_{rc}$ metric introduced in \citet{Dwork2012} and discussed in Section \ref{sec::opttransport}.

\begin{lemma} \label{lmm::biastestfunc}
Let $(\Chi,d)$ be a metric space. Then $D_{rc}(\mu,\nu; D_{TV},d) = W_{\A^*}(\mu,\nu)$.
\end{lemma}

\begin{proof}
See Appendix \ref{app::auxlemm}.
\end{proof}

Recall that \citet{Dwork2012} established that the statistical parity bias of a randomized classifier is bounded by the $D_{rc}$ distance between subpopulation input distributions. In contrast, we focus on measuring and explaining the bias in the output of non-randomized regressors, including classification scores, for which the notion of statistical parity is not, in general, applicable. In particular, we assess the distance between regressor output subpopulations via the $W_1$ metric. In general, any transport metric can be considered for this task, such as $W_{\A^*}$. Furthermore, we propose a framework that quantifies the contribution of predictors to that distance, which serves as a mechanism that pinpoints the main drivers to the regressor bias.

The lemma below illustrates the different behavior of the two metrics under scaling.

\begin{lemma}\label{lmm::scaling}
  Let $d(x,y)$ be a norm on $\RR^n$. Let $T(x)=cx+x_0$ with $c > 0$. Then
  \begin{equation*}
  \begin{aligned}
    D_{rc}(T_{\#}\mu,T_{\#}\nu; D_{TV},d ) = D_{rc}(\mu,\nu;D_{TV},d_c), \quad \frac{1}{c} W_1(T_{\#}\mu,T_{\#}\nu;d) = W_1(\mu,\nu;d)    
  \end{aligned}
  \end{equation*}
where $\mu,\nu \in \mathscr{P}_1(\RR^n;d)$ and $d_c(x,y)=cd(x,y)$.
\end{lemma}
\begin{proof}
See Appendix \ref{app::auxlemm}.
\end{proof}

Notice that for large $c$ the values of $D_{rc}$ with the $d_c$ norm saturate and approximate one, which is an upper bound for the metric. However, $W_1$ is unbounded and the distance between the pushforward measures $T_{\#}\mu, T_{\#}\nu$ scales linearly by $c$, which is an appealing property.

\citet{Dwork2012}  establishes the connection between $D_{rc}$ and $W_1$ under the assumption that the subpopulation distributions are discrete and $d\le 1$. In what follows, we prove a more general version of \citep[Theorem 3.3]{Dwork2012} that connects the two metrics and holds for all probability measures with bounded support. 

\begin{theorem}\label{thm::rbcbiaswasserstconn}
  Let $\mu,\nu \in \mathscr{P}_1(\RR^n;d)$ have bounded supports and $d(x,y)$ be a norm. Then 
    \begin{equation} \label{biaswassconnection}
  \begin{aligned}
 \frac{1}{L}  W_1(\mu,\nu \,; d) = D_{rc}(\mu,\nu;D_{TV},d_{(1/L)})\\
  \end{aligned}
  \end{equation}
for any $L>0$ such that ${\rm supp}(\mu) ,  {\rm supp}(\nu) \subset B(x_*,\tfrac{L}{2};d)=\{x: d(x,x_*)\leq \tfrac{L}{2} \}$.
  \end{theorem}  
\begin{proof}
 See Appendix \ref{app::auxlemm}.
\end{proof}

When using $D_{rc}$ for fairness, the above theorem implies that saturation can be  partially avoided via scaling. For example, the rescaling factor can be chosen as the second moment of the two probability measures. However, in our paper we focus on the Wasserstein metric because of its appealing scaling property.

\section{Bias explanations}\label{sec::biasexplmain}

\subsection{Relationship between model fairness and predictors}

It is shown in \cite{Gordaliza2020} that the statistical parity bias of (non-randomized) classifiers can be bounded by the total variance distance between predictors subpopulations, while the Wasserstein metric, in general, does not allow for such control (in the sense of a bound). In contrast to the bound in \cite{Gordaliza2020}, $W_1$-bias in predictors can control the statistical parity bias of Lipschitz randomized classifiers as shown in \cite{Dwork2012}, as well as the $W_1$-regressor bias as shown by the following lemma.

\begin{lemma}
Let $X,G,f$ be as in Definition \ref{def::genparity}. If $f$ is  Lipschitz continuous then
\begin{equation}\label{W1bound}
\Bias_{W_1}(f|X,G) \leq [f]_{Lip} \Bias_{W_1}(X|G).
\end{equation}
\end{lemma}
\begin{proof}
The proof follows directly from the Kantorovich-Rubinstein variational formula.
\end{proof}

While the fairness of predictors as a bound is of theoretical importance, it provides little information on the contribution of each predictor to the model unfairness. This is because fairness of predictors is a sufficient requirement for fairness of the model, but not a necessary one. In particular, a model can be slightly unfair while having wildly biased predictors. For example, consider the data generating model
\begin{equation}\label{modpredfairrel1}
  X_1 \sim N(\tau G,\sigma), \quad X_2 \sim N(0,\sigma), \quad Y = f(X) = \frac{\eps}{\tau}X_1 + X_2.
 \end{equation}
Note that $\Bias_{W_1}(X|G) \to \infty$ as $\tau \to \infty$, while $\Bias_{W_1}(f|X,G) = \epsilon$ for any $\tau>0$. 

This pedagogical example motivates us to directly assess the contribution of predictors to the model bias. To accomplish this, we design an interpretability framework that employs optimal transport theory in order to pinpoint the main drivers of the model bias. Information from these drivers can then be used for policy decision-making, regulatory-compliant bias mitigation \citep{Miroshnikov2021b}, as well as in other settings.



\subsection{Model interpretability}\label{subsec::modinterpr}


The bias explanations we develop in the next section make use of model explainers, whose objective is to quantify the contribution of each predictor to the value of  $f(x)$. Several methods of interpreting ML model outputs have been designed and used over the years. Some notable ones are Partial Dependence Plots (PDP) \citep{Friedman} and SHAP values \citep{LundbergLee}.

\paragraph{Partial dependence plots.} PDP marginalizes out the variables whose impacts to the output are not of interest, quantifying an overall impact of the values of the remaining features.

Let $X \in \RR^n $ be predictors, $X_S$ with $S \subseteq \{1, 2, \dots , n\}$ a subvector of $X$, and $-S$ the complement set. Given a model $f$, the partial dependence plot of $f$ on $X_S$ is defined by
\begin{equation}\label{pdpdef}
\PDP_S(x; f) = \E[ f(x_S,X_{-S}) ] \approx \frac{1}{N} \sum_{j=1}^N f(x_S,X^{(j)}_{-S}),
\end{equation}
where we abuse the notation and ignore the variable ordering in $f$.

\paragraph{Shapley additive explanations.} In its original form the Shapley values appear in the context of cooperative games; see \citet{Shapley,Young1985}. A cooperative game with $n$ players is a super-additive set function $v$ that acts on $N=\{1,2,\dots,n\}$ and satisfies $v(\varnothing)=0$. Shapley was interested in determining the contribution by each player to the game value $v(N)$. It turns out that under certain symmetry assumptions the contributions are unique and they are called Shapley values; furthermore, the super-additivity assumption can in principle be dropped (uniqueness and existence still hold).

 It is shown in \citet{Shapley} that there exists a unique collection of values $\{\varphi_i\}_{i=1}^n$ satisfying the axioms of symmetry, efficiency, and law of aggregation, ((A1)-(A3) in \citet{Shapley}), it is given by
\begin{equation}\label{shapform}
\varphi_i[v] = \sum_{S \subseteq N \backslash \{i\}} \frac{s!(n-s-1)!}{n!} [ v(S \cup \{i\}) - v(S) ], \quad  s=|S|, \,  n=|N|.
\end{equation}
The values provide a disaggregation of the value $v(N)$ of the game into $n$ parts that represent a contribution to the worth by each player:
$ \sum_{i=1}^n \varphi_{i}[v] = v(N).$

The explanation techniques explored in \citet{Strumbelji2010} and \citet{LundbergLee} utilize cooperative game theory to compute the contribution of each predictor to the model value. In particular, given a model $f$, \citet{LundbergLee} consider the games
\begin{equation}\label{shapgame}
\vce(S; X, f)=\E[f|X_S], \quad \vpdp(S;X,f)=\E[ f(X_S,X_{-S}) ]|_{x_S=X_S}
\end{equation}
with
\begin{equation*}
\vce(\varnothing; X, f)=\vpdp(\varnothing; X, f)=\E[f(X)].
\end{equation*}

The games defined in \eqref{shapgame} are not cooperative since they do not satisfy the condition  $v(\varnothing)=0$. However, by setting $\varphi_0=\E[f(X)]$, the values satisfy the additivity property:
\begin{equation*}
\sum_{i=0}^n \varphi_i[v(\cdot \,; X, f)]=f(X), \quad v \in \{\vce,\vpdp\}.
\end{equation*}

Throughout the text when the context is clear we suppress the explicit dependence of $v(S; X,f)$ on $X$ and $f$. Furthermore, we will refer to values $\varphi_i[\vpdp]$ and $\varphi_i[\vce]$ as SHAP values and abusing the notation we write
\[
\SHAP_i(X; f, v)=\varphi_i[v(S;X,f)], \quad v \in \{\vce,\vpdp\}.
\]

\paragraph{Conditional and marginal games.} In our work, we refer to the games $\vce$ and $\vpdp$ as conditional and marginal, respectively. If predictors $X$ are independent, the two games coincide. In the presence of dependencies, however, the games are very different. Roughly speaking, the conditional game explores the data by taking into account dependencies, while the marginal game explores the model $f$ in the space of its inputs, ignoring the dependencies. Strictly speaking, the conditional game is determined by the probability measure $P_X$, while the marginal game is determined by the product probability measures $P_{X_{S}} \otimes P_{X_{-S}}$, $S \subset N$ as stated below.

\begin{lemma}[\bf stability]\label{explcont} The SHAP explanations have the following properties:
\begin{itemize}
  \item [(i)]  $\|\SHAP(X;f,\vce)\|_{L^2(\P)} \leq \|f\|_{L^2(P_X)}$.
  \item [(ii)]  $\|\SHAP(X;f,\vpdp)\|_{L^2(\P)} \leq C\|f\|_{L^2(\widetilde{P}_X)}$, with $\widetilde{P}_X=\frac{1}{2^n}\sum_{S \subset N} P_{X_S} \otimes P_{X_{-S}}$.
\end{itemize}
\end{lemma}

\begin{proof}
By the properties of the conditional expectation and \eqref{shapform} we have
\[
\|\SHAP_i(X;f,\vce)\|_{L^2(\Omega)} \leq \sum_{S \subset N \backslash \{i\}} \frac{s!(n-s-1)!}{n!} \| \E[f(X)|X_S]\|_{L^2(\Omega)} \leq \|f\|_{L^2(P_X)}.
\]
Since $\varphi$ is linear, the map in $(i)$ is a bounded, linear operator with the unit norm. This proves $(i)$.

By \eqref{shapform} and \eqref{shapgame} we have
\[
\|\SHAP_i(X;f,\vpdp)\|_{L^2(\Omega)} \leq max_{s \in \{0,\dots,n-1\}} \frac{s!(n-s-1)!}{n!} \sum_{S \subset N \backslash \{i\}} \|f\|_{L^2(P_{X_S} \otimes P_{X_{-S}})} \leq C \|f\|_{L^2(\tilde{P}_X)}.
\]
where $C=C(n)$ is a constant that depends on $n$. This proves $(ii)$.
\end{proof}

To clarify the notation, we let $L^2(\widetilde{P}_X)$ denote the space of functions defined on $\RR^n$ such that
\[
\int f^2(x) \widetilde{P}_X (dx) := \frac{1}{2^n} \sum_{S \subset N}  \int f^2(x_S, x_{-S})   [P_{X_S} \otimes P_{X_{-S}}](x_S,x_{-S}) < \infty,
\]
where as before we ignore the variable ordering in $f$, and for $S=\varnothing$ we assign $P_{X_{\varnothing}} \otimes P_X=P_X$.


We should point out that under dependencies the marginal explanation map $(ii)$ in Lemma \ref{explcont} is in general not continuous in $L^2(P_X)$. Hence the algorithm that produces marginal explanations may fail to satisfy the stability bounds in the sense discussed in \citet{Kearns1999,Bousquet2002}. For a more general version of the above proposition see \citet{Kotsiopoulos2020}. 

In general, SHAPs are computationally intensive to evaluate due to the different combinations of predictors that need to be considered; in addition, computing $\varphi[\vce]$ is challenging when the predictor's dimension is large in light of the curse of dimensionality; see \citet{Hastie et al}. \citet{LundbergTreeSHAP} created a fast method called TreeSHAP but it can only be applied to ML algorithms that incorporate tree-based techniques. The algorithm evaluates $\varphi[\nu]$ for the game $\nu$ that can be chosen as either one that is based upon tree paths and resembles $\vce$, or the marginal game $\vpdp$. To understand the difference between the two games, see \citet{Janzing,Sundararajan,Chen-Lundberg, Kotsiopoulos2020}.


\subsection{Bias explanations of predictors}\label{sec::biasexplanations}


In this section, given a model, we define the bias explanation (or contribution) of each predictor. An extension to groups of predictors maybe found in Section \ref{sec::groupbiasexpl}. 

In what follows we will be using the following notation. Given predictors $X =(X_1,X_2,\dots,X_n)$ and a model $f$, a generic single feature explainer of $f$ that quantifies the attribution of each predictor $X_i$ to the model value $f(X)$ is denoted by 
\begin{equation*}
E(X; f) = (E_1(X;f),E_2(X;f),\dots,E_n(X;f)).
\end{equation*}

For example, a simple way of setting up an explainer $E_i$ is by specifying each component via a conditional or marginal expectation $E_i(X;f)=v(\{i\};X,f)$, $v \in \{\vce,\vpdp\}$. 

A more advanced way of computing single feature explanations is via the Shapley value $E(X;f)=\SHAP[v(\cdot;X,f)]$, $v \in \{\vce,\vpdp\}$. For more details on appropriate game values and their properties see \cite{Kotsiopoulos2020}.


\begin{definition}\label{def::predbiasexpl}
Let $X \in \RR^n$ be predictors, $f$ a model, $G \in \{0,1\}$ the protected attribute, $G=0$ the non-protected class, and $\favdir_{f}$ the sign of the favorable direction of $f$. Let $E(X; f)$ be an explainer of $f$ that satisfies $\E\big[|E(X;f)|\big]<\infty$.
\begin{itemize}[label=$\bullet$]
  \item The bias explanation of the predictor $X_i$ is defined by
  \[
  \beta_i(f|X,G; E_i) = W_1(E_i(X;f)|G=0 , E_i(X;f)|G=1 ) = \int_0^1 |F_{E_i|G=0}^{[-1]}-F_{E_i|G=1}^{[-1]}| \, dp.
  \]
  \item 
  The positive bias and negative bias explanations of the predictor $X_i$ are defined by
  \[
  \begin{aligned}
  \beta_i^{\pm}(f|X,G; E_i) = \int_{\Pcal_{i\pm}} (F_{E_i|G=0}^{[-1]}-F_{E_i|G=1}^{[-1]}) \cdot \favdir_{f} \, dp\\
  \end{aligned} 
  \]  
  where
  \[
  \begin{aligned}
\Pcal_{i\pm} = \{p \in [0,1]: \pm(F^{[-1]}_{E_i|G=0}-F^{[-1]}_{E_i|G=1})\cdot \favdir_{f} > 0 \}.
  \end{aligned} 
  \]

  In this case the $X_i$ bias explanation is disaggregated as follows:
\begin{equation*}
\beta_i(f|X,G; E_i) = \beta_i^{+}(f|X,G; E_i) +\beta_i^{-}(f|X,G; E_i).
\end{equation*} 
  \item The $X_i$ net bias explanation is defined by
  \[
  \beta_i^{net}(f|X,G; E_i) = \beta_i^{+}(f|X,G;E_i) - \beta_i^{-}(f| X,G; E_i).
  \]

  \item The classifier (or statistical parity) bias of the explainer $E_i$ for a threshold $t \in \RR$ is defined by
  \[
  \begin{aligned}
  \signbias^{C}_t(E_i|G) & = \big( F_{E_i|G=1}(t)-F_{E_i|G=0}(t) \big) \cdot \favdir_f.
  \end{aligned}
  \]
\end{itemize}
\end{definition}

By design the contribution $\beta_i^+$  measures the positive contribution to the total model bias, not the positive one. In particular, it measures the contribution to the increase in the positive flow and the decrease to the negative one. The meaning of  $\beta_i^-$ is similar. To better understand their meaning, consider the following data generating model:
\begin{equation}\label{zerobiasmod}
f(X)=X_1+X_2, \quad X_1=N(\mu+\tau G,\sigma), \quad X_2=N(\mu-\tau G,\sigma)  
\end{equation}
where $X_1,X_2$ are independent. Note that $\Bias_{W_1}(f|X,G)=0$, while the bias explanations are $\beta_1^+=\tau$, $\beta_1^-=0$, $\beta_2^+=0$, $\beta_2^- = \tau$ for either model explainer discussed in this section. Note also that both positive and negative model biases are zero. The positive contribution  $\beta_1^+=\tau$ measures how much in total is added to the positive model bias and subtracted from the negative one. A similar discussion holds for $\beta_i^-$. Thus, the amount that $X_1$  contributes to the positive bias is offset by the amount that $X_2$  resists to its increase. This leads to zero positive model bias. A similar discussion applies to the negative model bias.

\begin{lemma}\label{lmm::intposnegexplbias}
Let X, $f$, $G$, $E_i(X;f)$, and $\favdir_{f}$ be as in the definition \ref{def::predbiasexpl}. Then
\begin{equation}\label{netbiasexpl}
\beta^{net}(f|X,G;E_i)= \big(\E[E_i(X;f)|G=0]-\E[E_i(X;f)|G=1] \big) \cdot \favdir_f.
\end{equation}
\end{lemma}
\begin{proof}
Similar to the proof of Theorem \ref{thm::intposnegmodbias} with the assumption $\favdir_{E_i}=\favdir_f$.
\end{proof}

Observe that the bias explanations for a classification score always lie in the unit interval.
\begin{lemma}
Let $f$ be a classification score and $G \in \{0,1\}$ the protected attribute. Let the explainer $E_i$ be either $v(\{i\};X,f)$ or $\SHAP_i[v(\cdot;X,f)]$, where $v \in \{\vce,\vpdp\}$. Then $\beta_i,\beta_i^-,\beta_i^+ \in [0,1]$.
\end{lemma}
\begin{proof}
The lemma follows from the fact that $f \in [0,1]$ and the definition of explainer values.
\end{proof}

The explainer $E_i$ that appears in Definition \ref{def::predbiasexpl} is a generic one. In the examples that follow we chose to work with explainers based on  marginal SHAPs because of the ease of computation. Note that when predictors are independent then the two types of explanations coincide; for the case when dependencies are present see the discussion at the end of the section.


\begin{figure}
\centering
\subfloat{\includegraphics[width= 0.3\textwidth]{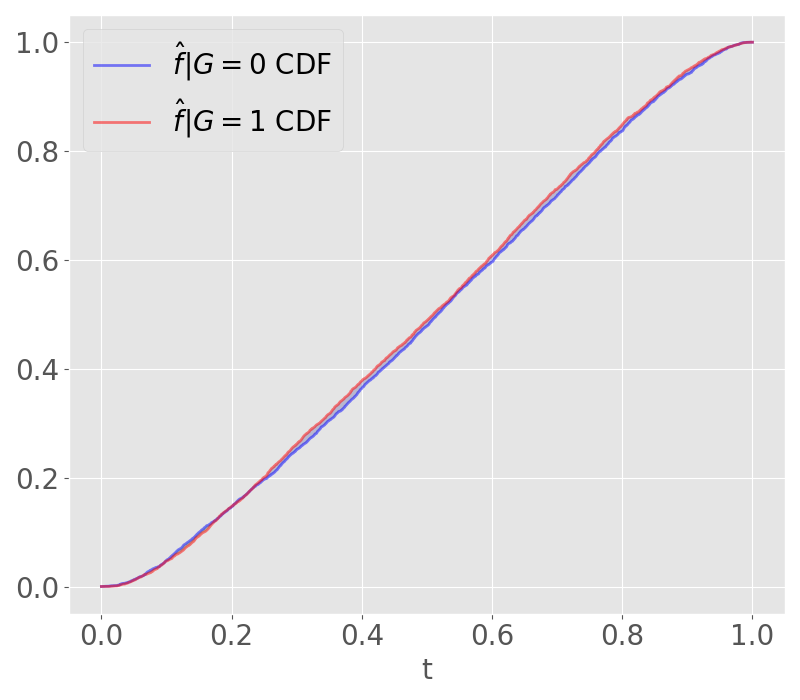}}
~~
\subfloat{\includegraphics[width= 0.3\textwidth]{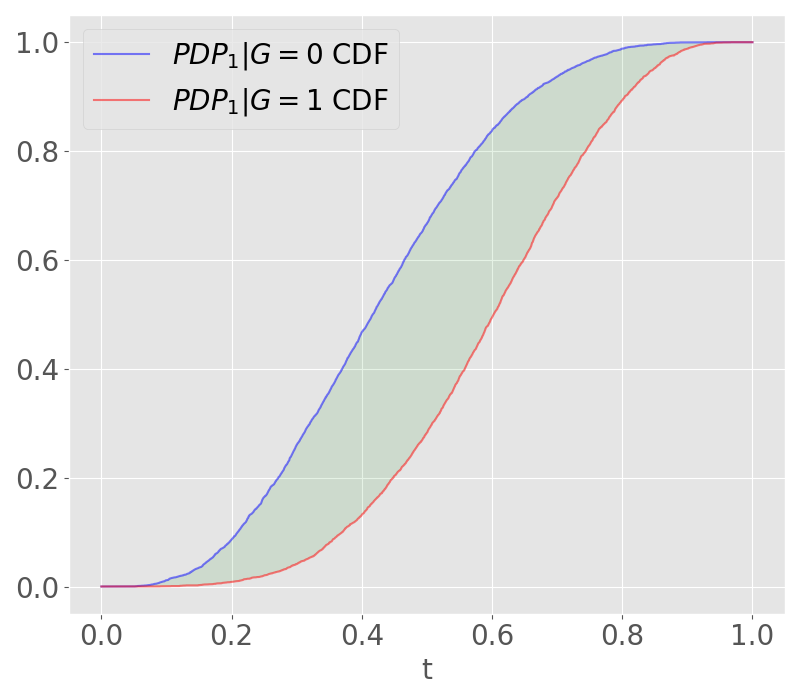}}
~~
\subfloat{\includegraphics[width= 0.3\textwidth]{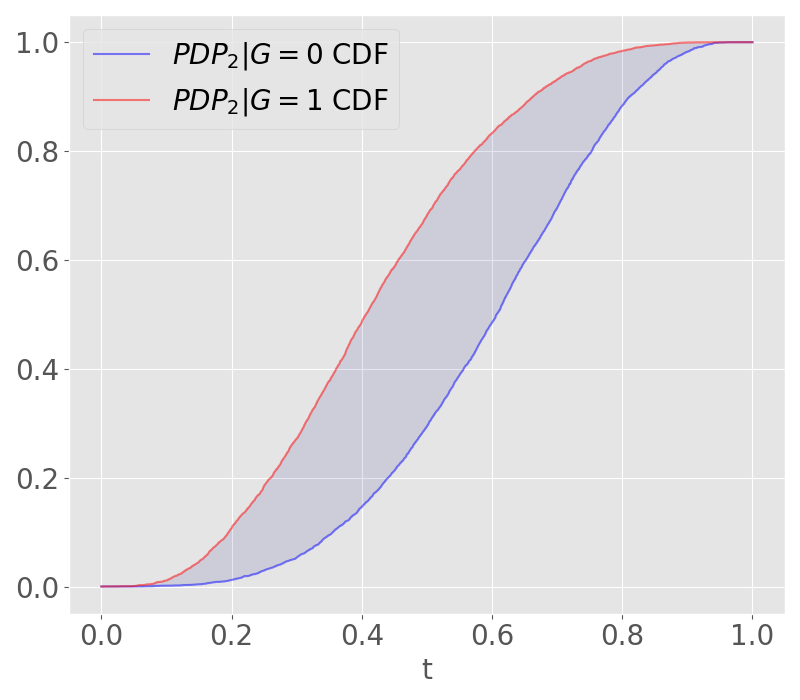}}
\caption{ \footnotesize Model and PDP biases for the model \eqref{offsettmod}, $\favdir_{\fhat}=-1$. }\label{fig::offsettingsimple}
\end{figure}


\paragraph{Intuition.} For a given model $f$ and the explainer $E_i$ the explanation $\beta_i$ quantifies the $W_1$ distance between the distributions of the explainer $E_i|G=0$ and $E_i|G=1$, that is, this value is an assessment of the bias introduced by the predictor $X_i$. The value $\beta_i$ is the area between the corresponding subpopulation explainer CDFs $F_{E_i|G=k}$, $k \in \{0,1\}$, similar to the area depicted in Figure  \ref{fig::posnegmodelbias}. The value $\beta_i^+$ represents the bias across quantiles of the explainer $E_i$ for which the predictor $X_i$ favors the non-protected class $G=0$ and $\beta_i^-$ represents the  bias across quantiles for which $X_i$ favors the protected class $G=1$. The $\beta^{net}_i$ assesses the net contribution across different quantiles and represents an explanation that allows one to assess whether {\it on average} the predictor $X_i$ favors class $G=0$ or class $G=1$; see Lemma \ref{lmm::intposnegexplbias}.

In what follows we consider several simple examples to get more intuition behind the bias explanation values as well as discuss their additivity or the lack thereof. To avoid complex notation when the context is clear we suppress the dependence of the bias explanations on $X$ and the explainer $E$.

\begin{definition} Let $f$, $X$, $G$, and $E_i$ be as in Definition \ref{def::predbiasexpl}.
\begin{itemize}[label=$\bullet$]
\item We say that $E_i$ strictly favors class $G=0\,(G=1)$ if $\beta_i^-(f|G;E_i)=0$ ($\beta_i^+(f|G;E_i)=0$).
  \item We say that $X_i$ has mixed bias explanations if $\beta_i^{\pm}(f|G;E_i)>0$.
\end{itemize}
\end{definition}

\begin{figure}
\centering
\subfloat{\includegraphics[width= 0.3\textwidth]{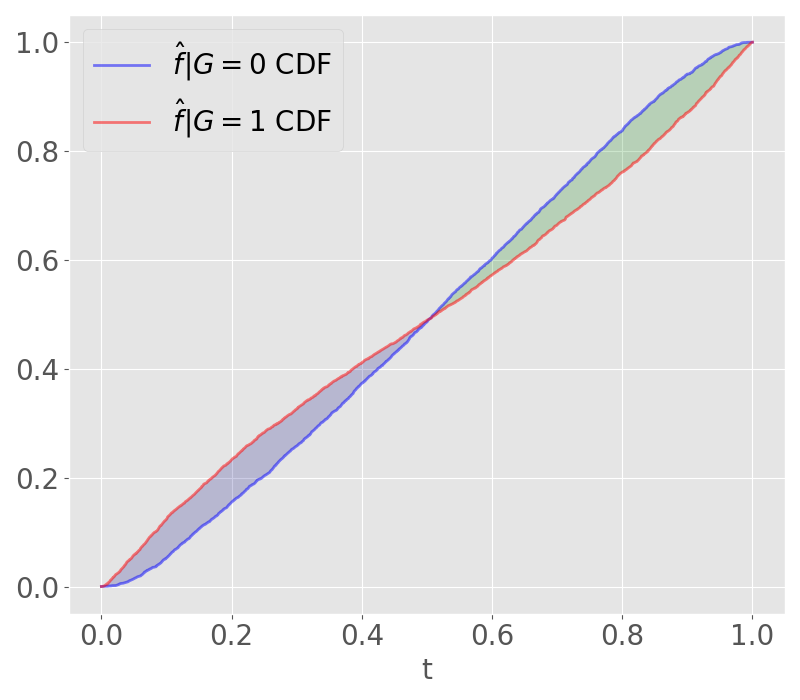}}
~~
\subfloat{\includegraphics[width= 0.3\textwidth]{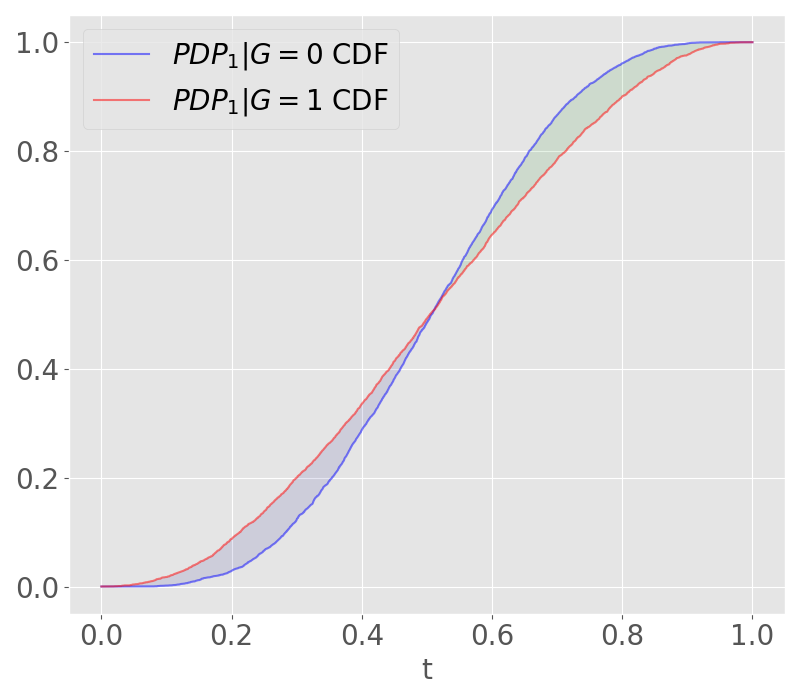}}
~~
\subfloat{\includegraphics[width= 0.3\textwidth]{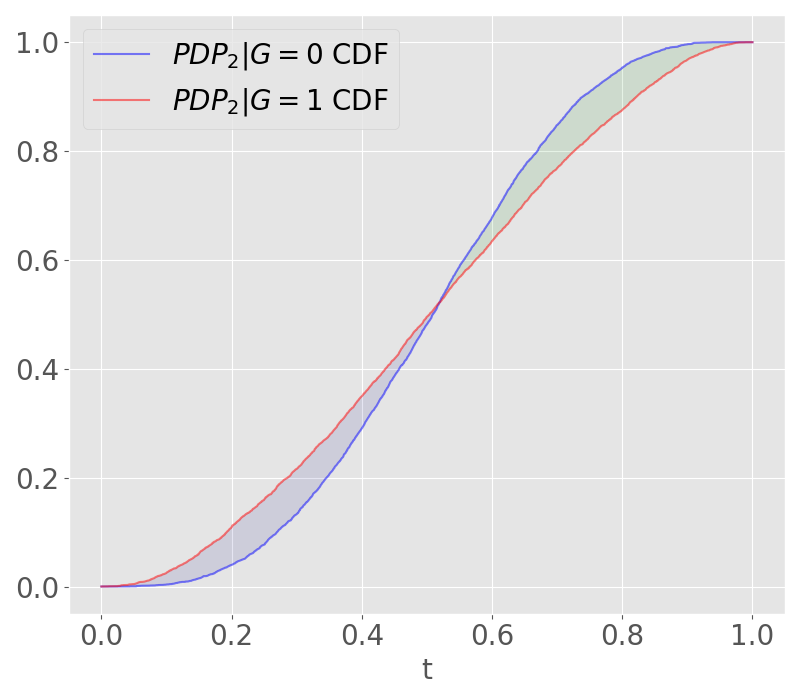}}
\caption{ \footnotesize Model and PDP biases for the model \eqref{mixbiasmodampl}, $\favdir_{\fhat}=-1$. }\label{fig::offsettingmixampl}
\end{figure}


\begin{figure}
\centering
\subfloat{\includegraphics[width= 0.3\textwidth]{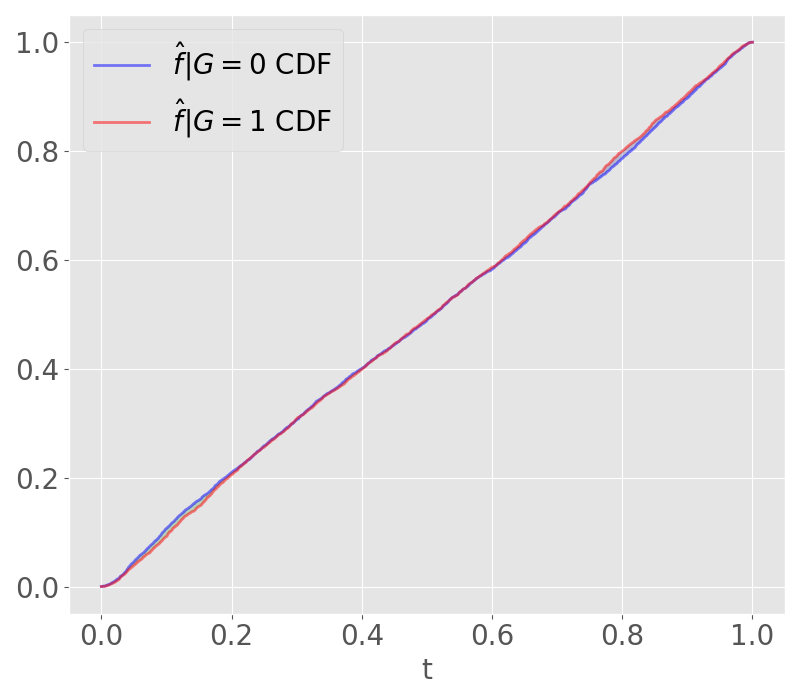}}
~~
\subfloat{\includegraphics[width= 0.3\textwidth]{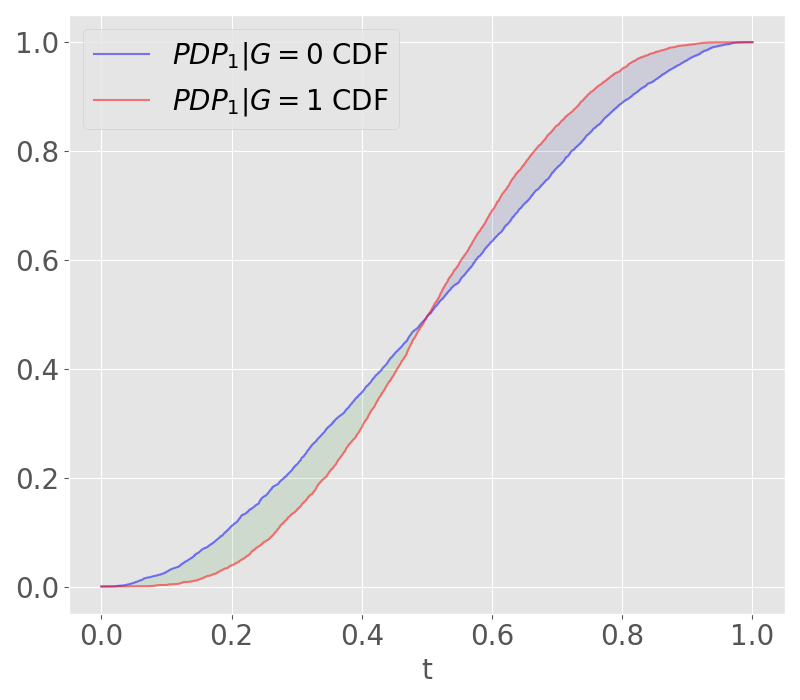}}
~~
\subfloat{\includegraphics[width= 0.3\textwidth]{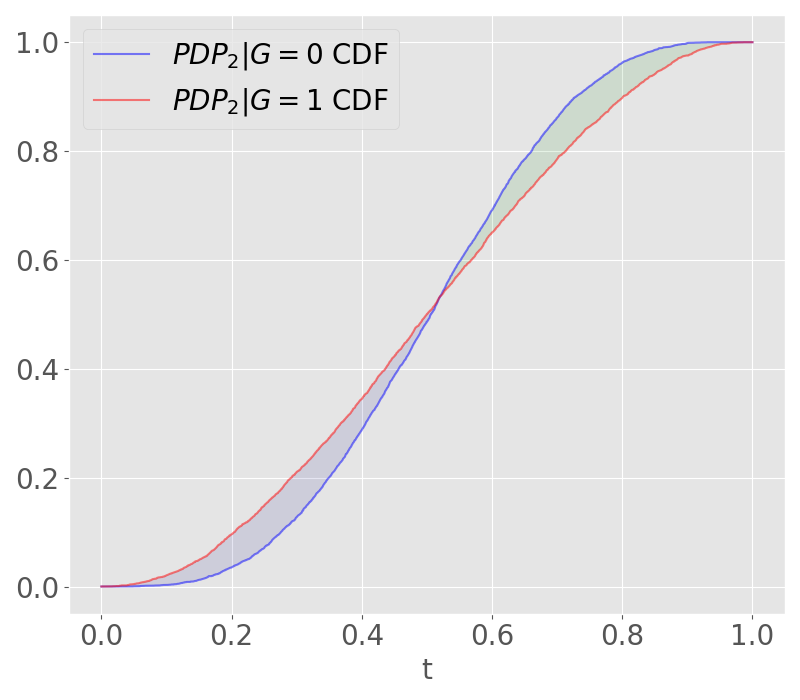}}
\caption{ \footnotesize Model and PDP biases for the model \eqref{mixbiasmodvanish}, $\favdir_{\fhat}=-1$. }\label{fig::offsettingmixvanish}
\end{figure}


\paragraph{Offsetting.} Since each predictor may favor one class or the other, the predictors may offset each other in terms of the bias contributions to the model bias. To understand the offsetting effect consider a binary classification risk model ($\favdir_f=-1$) with two predictors: 
\begin{equation}\label{offsettmod}\tag{M3}
\begin{aligned}
&X_1 \sim N(\mu+G,1), \quad X_2 \sim N(\mu-G,1) \\
&Y \sim Bernoulli(f(X)), \quad f(X)=\P(Y=1|X)={logistic(2\mu-X_1-X_2)}\\
\end{aligned}
\end{equation}
where $\mu=5$, and $\{X_i|G=k\}_{i,k}$ are independent and  $\P(G=0)=\P(G=1)$. We next  train logistic regression score $\fhat(X)$, with $\favdir_{\fhat}=-1$, and choose the explainer to be $E_i=\PDP_i$. By construction the explanation $E_1$ of the predictor $X_1$ strictly favors class $G=0$, while that of $X_2$ strictly favors class $G=1$. Moreover,  
\[
\beta_1(\fhat|G; E_1)=\beta^+_1(\fhat|G; E_1)=\beta^-_2(f|G;E_2)=\beta_1(\fhat|G; E_2)\approx 0.17.
\]
Combining the two predictors at the model level leads to bias offsetting. By construction the resulting model bias is $\modbias_{W_1}({f}|G)=0$. Figure \ref{fig::offsettingsimple} displays the CDFs for the trained score subpopulations $\fhat|G=k$ and the corresponding explainers $E_i|G=k$, which illustrates the offsetting phenomena numerically.

\begin{figure}
\centering
\subfloat{\includegraphics[width= 0.3\textwidth]{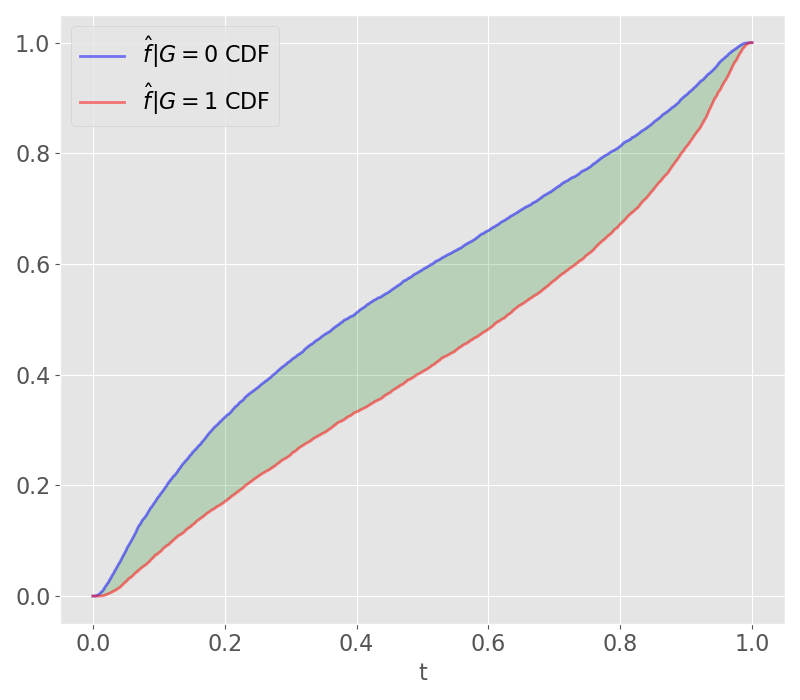}}
~~
\subfloat{\includegraphics[width= 0.3\textwidth]{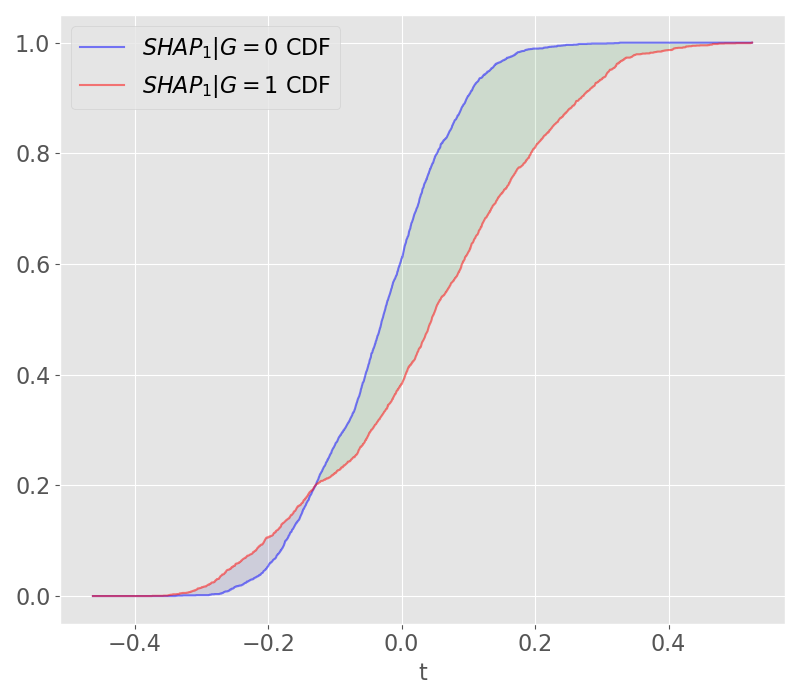}}
~~
\subfloat{\includegraphics[width= 0.3\textwidth]{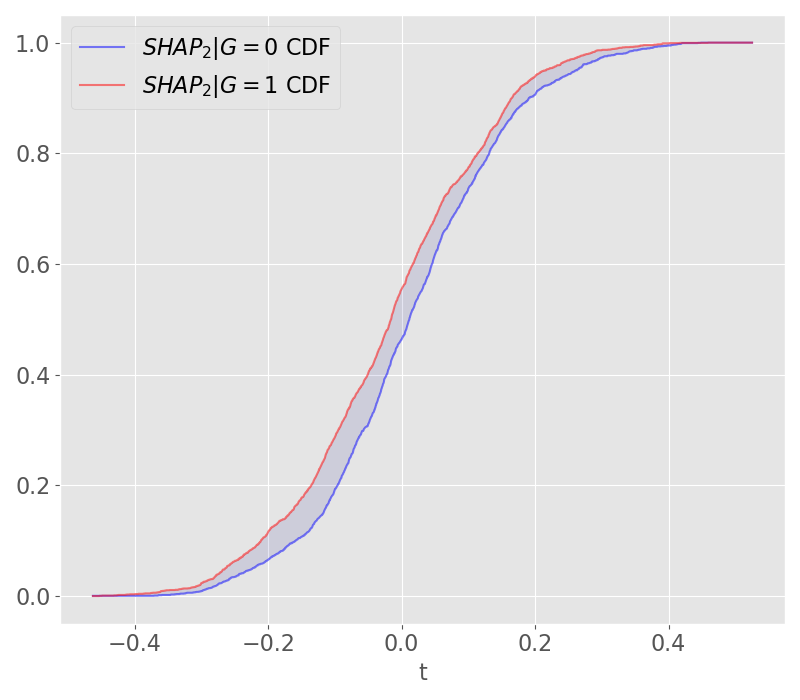}}
\\
\subfloat{\includegraphics[width= 0.3\textwidth]{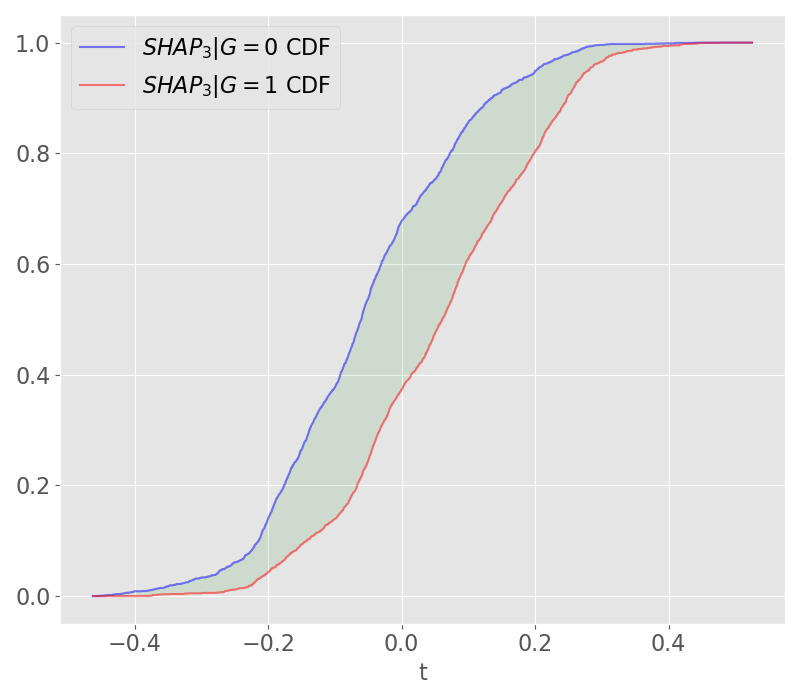}}
~~
\subfloat{\includegraphics[width= 0.3\textwidth]{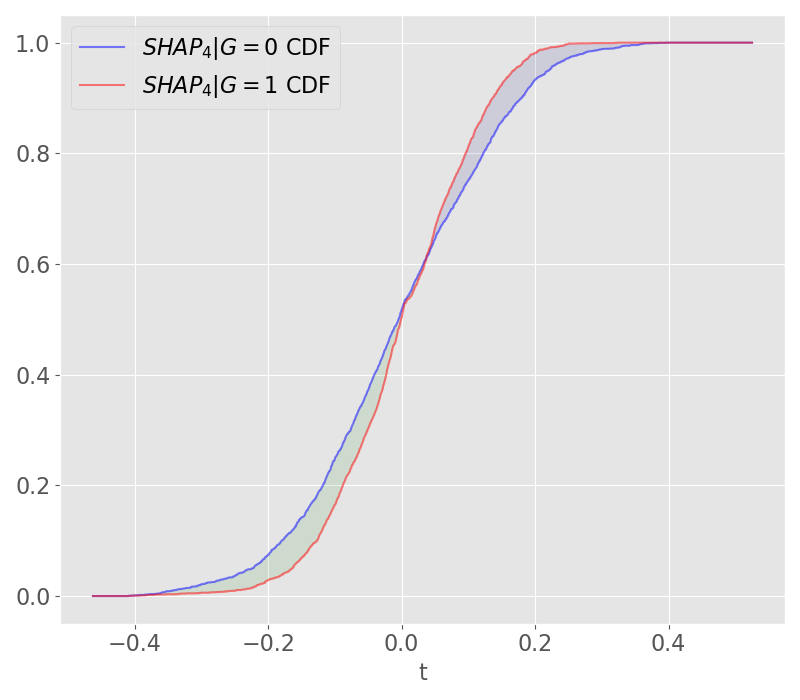}}
~~
\subfloat{\includegraphics[width= 0.3\textwidth]{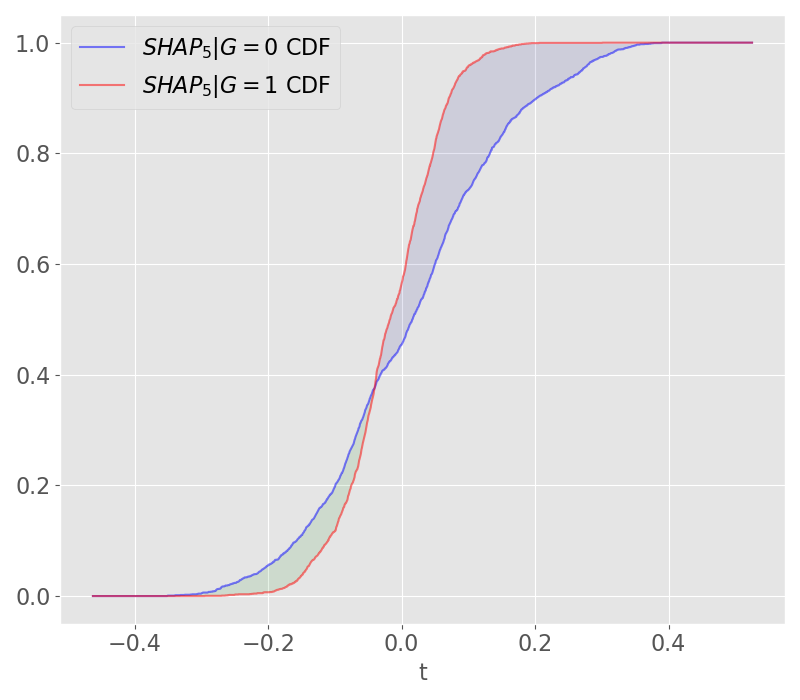}}
\caption{ \footnotesize Model bias and SHAP explainer biases for trained XGBoost \eqref{realisticmod}, $\favdir_{\fhat}=-1$. }\label{fig::biasrealisticmod}
\end{figure}

Another important point we need to make is that the equality $\beta_i^{net}=0$ does not in general imply that the predictor $X_i$ has no effect on the model bias. This is a consequence of \eqref{netbiasexpl}. Moreover, predictors with mixed bias might amplify the model bias as well as offset it. To understand how mixed bias predictors interact at the level of the model bias consider the following risk classification model ($\favdir_f=-1$).
\begin{equation}\label{mixbiasmodampl}\tag{M4}
\begin{aligned}
&X_1 \sim N(\mu, 1+G), \quad X_2 \sim N(\mu, 1+G) \\
&Y \sim Bernoulli(f(X)), \quad f(X)=\P(Y=1|X)={logistic(2\mu-X_1-X_2)}.
\end{aligned}
\end{equation}
where $\mu=5$, and $\{X_i|G=k\}_{i,k}$ are independent and $\P(G=0)=\P(G=1)$. As before we train a logistic regression score $\fhat$, with $\favdir_{\fhat}=-1$, and choose $E_i=\PDP_i$. By construction, the true classification score $f$ satisfies $\beta^{net}_i(f|G)=0$ for each predictor $X_i$. Furthermore, the CDFs of explainers satisfy
\[
(F_{E_i(X,f)|G=0}(t)-F_{E_i(X,f)|G=1}(t)) \cdot {\sgn}(t-0.5)>0
\]
for any threshold $t \neq 0.5$. Combining the two predictors at the level of the model leads to amplifying the positive and negative model biases and hence the model bias itself. Figure \ref{fig::offsettingmixampl} displays the CDFs for the trained score subpopulations $\fhat|G=k$ and the corresponding explainers $E_i(\fhat)|G=k$. The numerics illustrate that as long as the regions for positive and negative bias of mixed predictors agree, when combined they will increase the model bias. 

If the regions of positive and negative bias for two predictors do not agree, then offsetting will happen. To see this, let us modify the above example as follows:
\begin{equation}\label{mixbiasmodvanish}\tag{M5}
\begin{aligned}
&X_1 \sim N(\mu, 2-G ), X_2 \sim N(\mu, 1+G ) \\
&Y \sim Bernoulli(f(X)), \quad f(X)=\P(Y=1|X)={logistic(2\mu-X_1-X_2)}.
\end{aligned}
\end{equation}
By construction, $\beta^{net}_i(f|G)=0$ for each predictor. However, the region of thresholds where the explainer $E_1(f)$ favors class $G=0$ coincides with the region where $E_2(f)$ favors class $G=1$, and the same holds for the two complimentary regions. This leads to bias offsetting so that $\modbias_{W_1}(f|G)=0$. The numerical results for this example are displayed in Figure \ref{fig::offsettingmixvanish}.

\paragraph{Bias explanation plots.} Given a machine learning model $f$, predictors $X \in \RR^n$, protected attribute $G$, and the explainers $E_i$, the corresponding bias explanations 
\[
\big\{(\beta_i,\beta_i^+,\beta_i^-,\beta^{net}_i)(f|G;E_i)\big\}_{i=1}^n
\]
are sorted according to any desired entry in the 4-tuple and then displayed in that order. This plot is called {\it Bias Explanation Plot} (BEP). 

To showcase how BEP works, consider a classification  risk model ($\favdir_f=-1$):
\begin{equation}\label{realisticmod}\tag{M6}
\begin{aligned}
&\mu=5, \quad a=\tfrac{1}{20}(10,-4,16,1,-3)\\
&X_1 \sim N(\mu-a_1 (1-G), 0.5+G ), \quad X_2 \sim N(\mu-a_2 (1-G), 1 ) \\
&X_3 \sim N(\mu-a_3 (1-G), 1 ),     \quad X_4 \sim N(\mu-a_4 (1-G), 1-0.5 G ) \\
&X_5 \sim N(\mu-a_5 (1-G),1-0.75 G ) \\
&Y \sim Bernoulli(f(X)), \quad f(X)=\P(Y=1|X)={logistic(\textstyle{\sum_{i}} X_i-24.5)}.
\end{aligned}
\end{equation}
where $\{X_i|G=k\}_{i,k}$ are independent and $\P(G=0)=\P(G=1)$. We next generate $20,000$ samples from the distribution $(X,Y)$ and train a regularized XGBoost model which produces the score $\fhat$. Figure \ref{fig::biasrealisticmod} displays the CDFs of the subpopulation scores $\fhat|G=k$ (top left), and those of the explainers $E_i=\SHAP_i(\fhat, \vpdp)$. We see that there is positive model bias in the plot showing the CDFs, thus class $G=0$ is favored. For the predictors, the bias explanation plots show that $X_1,X_4$ and $X_5$ have mixed biases that arise due to differences in subpopulation variances of predictors, while the bias in $X_2$ strictly favors class $G=1$ and the bias in $X_3$ favors $G=0$.

\begin{figure}
\centering
\subfloat{\includegraphics[width= 0.3\textwidth]{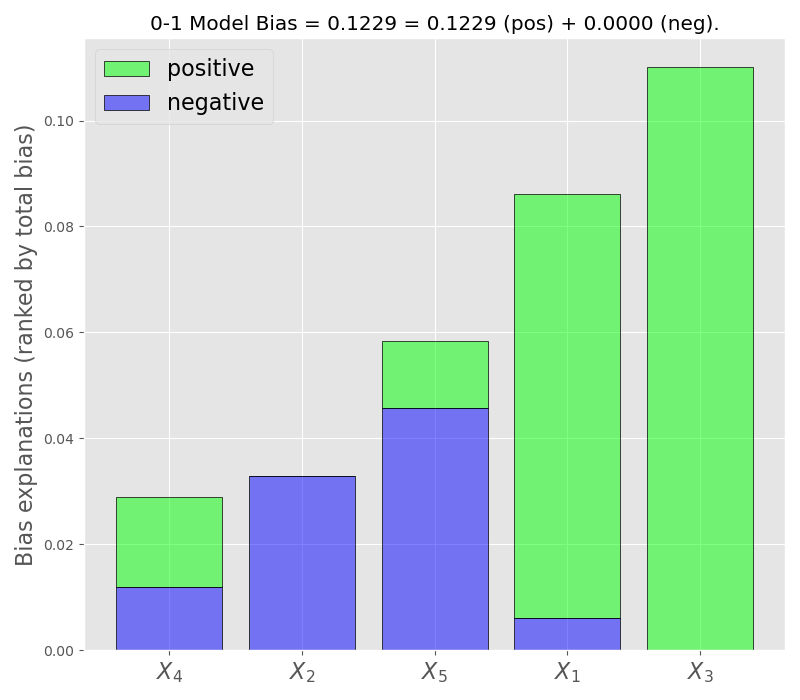}}
~~
\subfloat{\includegraphics[width= 0.3\textwidth]{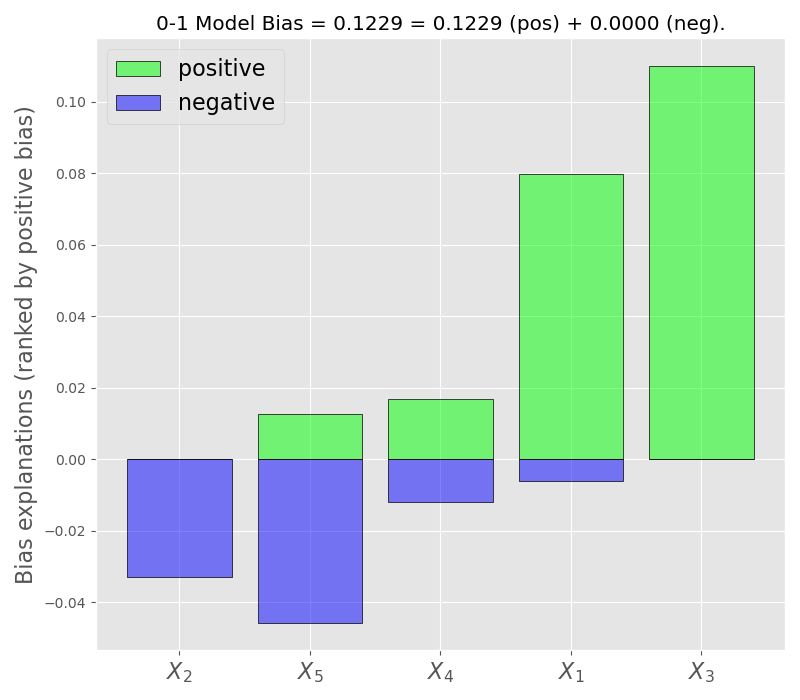}}
~~
\subfloat{\includegraphics[width= 0.3\textwidth]{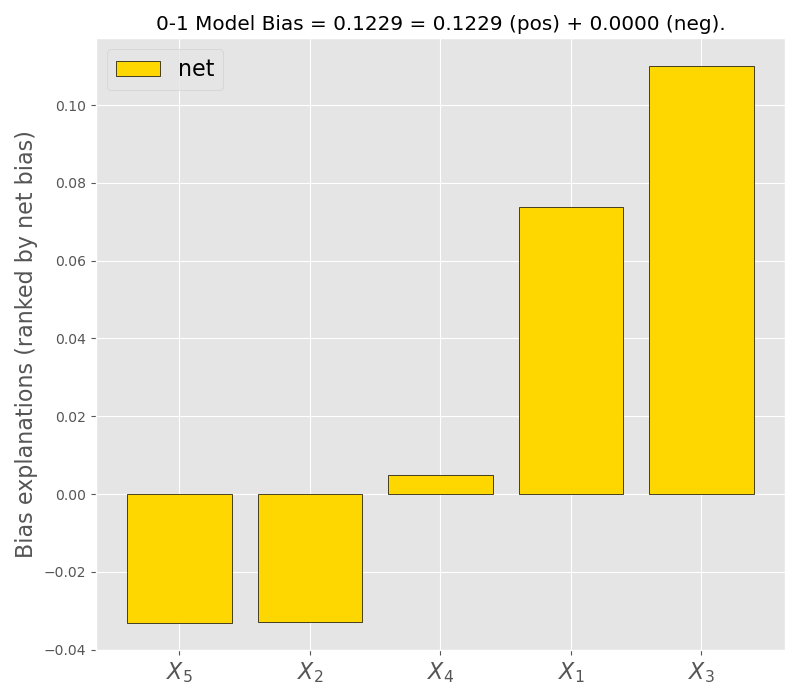}}
\caption{ \footnotesize Bias explanations ranked by $\beta_i$ and $\beta_i^+$ and ranked $\beta^{net}_i$ for the model \eqref{realisticmod}, $\favdir_{\fhat}=-1$. }\label{fig::biasexplrealisticmod}
\end{figure}


The numerically computed model bias and its disaggregation are given by
\[
(\modbias_{W_1},\modbias^+_{W_1},\modbias^-_{W_1},\modbias^{net}_{W_1})(\fhat|G)=(0.1533,0.1533,0,0.1533)
\]
The bias explanations are then computed as the Earth Mover distance, and its disaggregation, between the distributions of subpopulation explainers $E_i(\fhat)|G=k$. The bias explanations are given by
\begin{equation*}
\begin{aligned}
(\beta_1,\beta_1^+,\beta_1^-,\beta^{net}_1)&=(0.0860,0.0799,0.0061,0.0738)\\
(\beta_2,\beta_2^+,\beta_2^-,\beta^{net}_2)&=(0.0328,0, 0.0328,-0.0328)\\
(\beta_3,\beta_3^+,\beta_3^-,\beta^{net}_3)&=(0.1100,0.1100,0,0.1100)\\
(\beta_4,\beta_4^+,\beta_4^-,\beta^{net}_4)&=(0.0289,0.0169,0.0119,0.0050)\\
(\beta_5,\beta_5^+,\beta_5^-,\beta^{net}_5)&=(0.0584,0.0127,0.0457,-0.0330)\\
\end{aligned}
\end{equation*}
Figure \ref{fig::biasexplrealisticmod} displays the above bias explanations for each predictor in increasing order by total bias (left), positive bias (middle), and ranked net bias (right), respectively. Clearer information can be obtained from these plots compared to Figure \ref{fig::biasrealisticmod}. For example, one can now see how mixed $X_1,X_4,X_5$ are and how the positive and negative parts compare.

\paragraph{Relationship with model bias.} The positive and negative bias explanations provide an informative way to determine the main drivers for positive and negative bias among predictors, which can be done by ranking the bias attributions. However, though informative, the positive and negative bias explanations are {\it not additive}. That is, in general 
\begin{equation*}
\textstyle\modbias_{W_1}^{\pm}(\fhat|G) \neq \sum_{i=1}^n \beta_i^{\pm}(\fhat|G; E_i). 
\end{equation*}

The main reason for lack of additivity is the presence of {\it bias interactions} which happen at the level of quantiles, or thresholds. The bias explanations by design compute the contribution to the cost of transport but do not track how mass is transported; see Figures \ref{fig::offsettingmixampl}, \ref{fig::offsettingmixvanish}. To better understand the bias interactions, motivated by \citet{Strumbelji2010}, we introduce a game theoretic approach in Section \ref{app::biasgame} that yields additive bias explanations. 

For additive models with independent predictors, however, we have the following result.

\begin{lemma}\label{lmm::biasexpladdmod} 
Let $X \in \RR^n$ be predictors. Let the model $f$ be additive, that is, $f(X)=\sum_{i=1}^n f_i(X_i)$. Let an explainer $E_i$ be either $\vpdp(\{i\};X,f)$ or $\SHAP_i[\vpdp(\cdot;X,f)]$. Let $\{\beta_i,\beta_i^+,\beta_i^-,\beta^{net}_i\}_i$ be the bias explanations of $(X,f)$.
Then
\begin{equation*}
\modbias^{net}_{W_1}(f|G)=\modbias^{+}_{W_1}(f|G)-\modbias^{-}_{W_1}(f|G) = \sum_{i=1}^n \big(\beta^+_i - \beta^-_i\big)=\sum_{i=1}^n \beta^{net}_i.
\end{equation*}
\end{lemma}
If $X$ are independent then the lemma holds for $E_i$ in the form $\vce(\{i\};X,f)$ or $\SHAP_i[\vce(\cdot;X,f)]$.

\begin{proof}
Suppose that $E_i(X;f)=\vpdp(\{i\};X,f)$. Then, in view of the additivity of $f$, we have 
\[
\vpdp(\{i\};X,f) = f_i(X_i)-\E[f_i(X_i)]+\E[f(X)]
\]
and hence by Lemma \ref{lmm::intposnegexplbias} we have
\[
\beta_i^{net}(f|G; \vpdp) = \big( \E[f_i(X_i)|G=0]-\E[f_i(X_i)|G=1] \big) \cdot \favdir_{f}.
\]
Summing up the net bias explanations gives
\begin{equation}\label{pdpexplmodrel}
\begin{aligned}
\sum_i \beta_i^{net}(f|G; \vpdp ) &= \sum_i \big( \E[f_i(X_i)|G=0]-\E[f_i(X_i)|G=1] \big) \cdot \favdir_{f} \\
& = \big( \E[f(X)|G=0]-\E[f(X)|G=1] \big) \cdot \favdir_{f}  = \Bias^{net}_{W_1}(f|G).
\end{aligned}
\end{equation}
Suppose that $E_i(X;f)=\SHAP_i(X; f, \vpdp)$. Since $\{X_i\}_{i=1}^n$ are independent and $f$ is additive,
\[
\SHAP_i(X; f, \vpdp)=\SHAP_i(X; f, \vce) = f_i(X_i)-\E[f_i(X_i)] = \vpdp(\{i\};X,f) + \E[f(X)].
\]
Since a shift in the distribution does not affect the bias, the bias explanation based on $\SHAP_i[\vpdp]$ coincide with that of $\vpdp$. This together with \eqref{pdpexplmodrel} and the independence assumption proves the lemma.
\end{proof}

\paragraph{Example.} Let $f$ be as in Lemma \ref{lmm::biasexpladdmod}. Suppose that $f$ is either positively biased or negatively biased, that is, $\modbias_{W_1}(f|G) = (1-\delta) \cdot \modbias_{W_1}^+(f|G)+\delta \cdot \modbias_{W_1}^-(f|G)$ with $\delta \in \{0,1\}$. Then
\begin{equation*}
\modbias_{W_1}(f|G)= (-1)^{\delta} \sum_{i=1}^n  (\beta_i^+ - \beta_i^-).
\end{equation*}

\subsection{Stability of marginal and conditional bias explanations}\label{subsec::biasexplstab}

Under dependencies the marginal and conditional bias explanations differ in their description. The conditional bias explanations rely on the joint distribution $(X,Y)$ and encapsulate the interaction between the bias in predictors and the response variable, while the marginal explanations encapsulate the interaction between bias in predictors and the structure of the model, that is, the map $x\to f(x)$; for details see \citet{Kotsiopoulos2020}. In particular, we have the following result.

\begin{theorem}[\bf stability]\label{biasexplcont} Let $X \in \RR^n$ be predictors. Let $E_i=\SHAP_i[v]$, $v \in \{\vce,\vpdp\}$. The bias explanations based on the marginal and conditional Shapley values satisfy the following:
\begin{itemize}
  \item [(i)]  For all $f,g \in L^2(P_X)$, we have
\[
|\beta_i^{\pm}(f|X,G,\SHAP_i[\vce])-\beta_i^{\pm}(g|X,G,\SHAP_i[\vce]) | \leq C\|f-g\|_{L^2(P_X)}.
\]
  \item [(ii)]  For all $f,g \in L^2(\widetilde{P}_X)$, we have
\[
|\beta_i^{\pm}(f|X,G,\SHAP_i[\vpdp])-\beta_i^{\pm}(g|X,G,\SHAP_i[\vpdp]) | \leq C\|f-g\|_{L^2(\widetilde{P}_X)}.
\]
\end{itemize}
\end{theorem}

\begin{proof}
Take $f,g \in L^2(P_X)$. Take $i \in \{1,2,\dots,n\}$ and set
\[
A=\varphi_i[ \vce(\cdot;X,f)], \quad B=\varphi_i[ \vce(\cdot;X,g)].
\]
Let $\mu_k=P_{A|G=k}$, $\nu_k=P_{B|G=k}$, and $\gamma=P_{(A,B)|G=k}$ for $k \in \{0,1\}$. By construction $\gamma_k \in \Pi(\mu_k,\nu_k)$ and hence
\[
\begin{aligned}
\sum_{k \in \{0,1\}} W_1(\mu_k,\nu_k) & \leq \sum_{k \in \{0,1\}} \int |x_1 - x_2 | P_{(A,B)|G=k}(dx_1,dx_2) \\
& \leq \sum_{k \in \{0,1\}} \E[ |A-B| G=k]\\
& \leq  C  \|A-B\|_{L^2(\P)} \leq C \|f-g\|_{L^2(P_X)}
\end{aligned}
\]
where $C=\max_{ k \in \{0,1\} }\big\{ \tfrac{1}{\P(G=k)} \big\}$ and the last inequality follows from Lemma \ref{explcont}$(i)$. 

Then, using the triangle inequality and the inequality above, we obtain
\[
\begin{aligned}
|\beta_i(f|X,G,\varphi_i[\vce])-\beta_i(g|X,G,\varphi_i[\vce])|&=|W_1(\mu_1,\mu_2)-W_1(\nu_1,\nu_2)| \\
&\leq W_1(\mu_1,\nu_1)+W_1(\nu_2,\mu_2) \\
&\leq C \|f-g\|_{L^2(P_X)}.
\end{aligned}
\]

We next establish the bounds for the net-bias explanations. Assuming $\favdir_{f}=\favdir_{g}$ and using Lemma \ref{lmm::intposnegexplbias} we obtain
\[
\begin{aligned}
&|\beta_i^{net}(f|X,G,\varphi_i[\vce])-\beta_i^{net}(g|X,G,\varphi_i[\vce])| \\
&=| \E[A|G=0] - \E[ A|G=1] - \E[B|G=0] + \E[B|G=1] |\\
&\leq \sum_{k \in \{0,1\}} \E[ |A-B| | G=k]\\
& \leq C  \|A-B\|_{L^2(P)} \leq C \|f-g\|_{L^2(P_X)}.
\end{aligned}
\]

Combining the above inequalities and using the fact that $\beta^{\pm}=\frac{1}{2}(\beta \pm \beta^{net})$ gives $(i)$. To prove $(ii)$, we follow the same steps as above and use Lemma 
\ref{explcont}$(ii)$.
\end{proof}

\begin{remark}\rm
 Proposition \ref{biasexplcont} implies that the  map $f \to \beta_i^{\pm}(f|X,G,\SHAP_i[\vce])$ is continuous in $L^2(P_X)$ and the map $f \to \beta_i^{\pm}(f|X,G,\SHAP_i[\vpdp])$ is continuous in $L^2(\widetilde{P}_X)$.
\end{remark}



\subsection{Shapley-bias explanations}\label{app::biasgame}



\begin{figure}
\centering
\subfloat{\includegraphics[width= 0.3\textwidth]{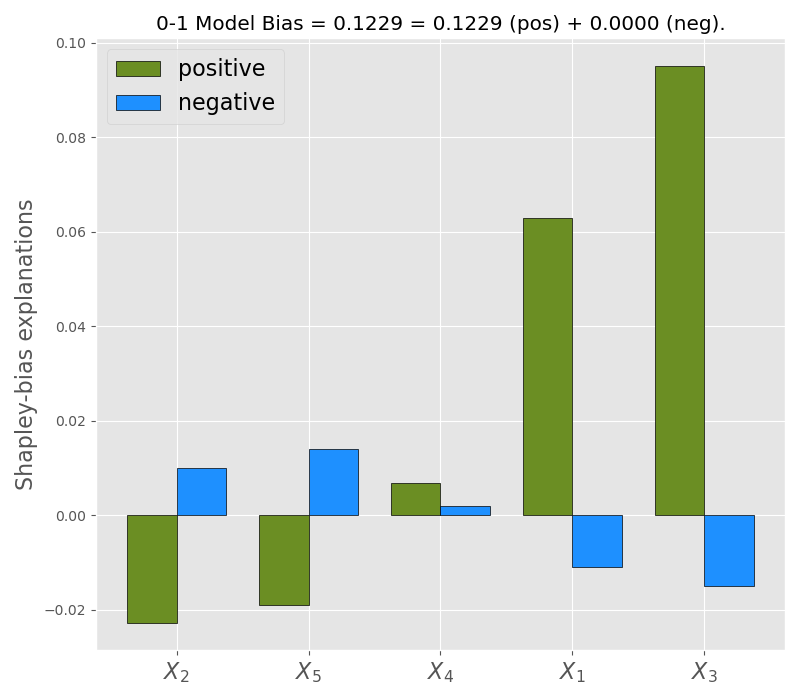}}
~~
\subfloat{\includegraphics[width= 0.3\textwidth]{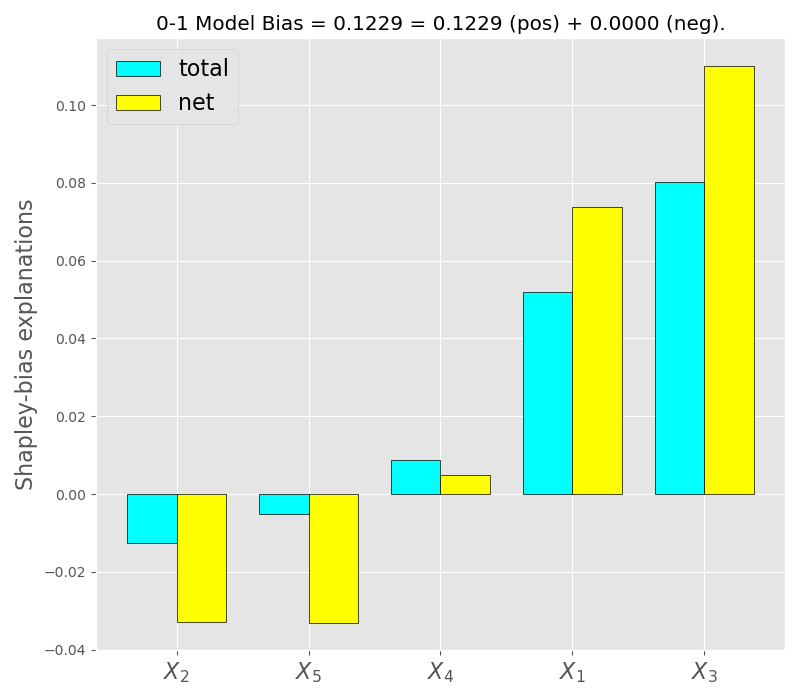}}
~~
\subfloat{\includegraphics[width= 0.3\textwidth]{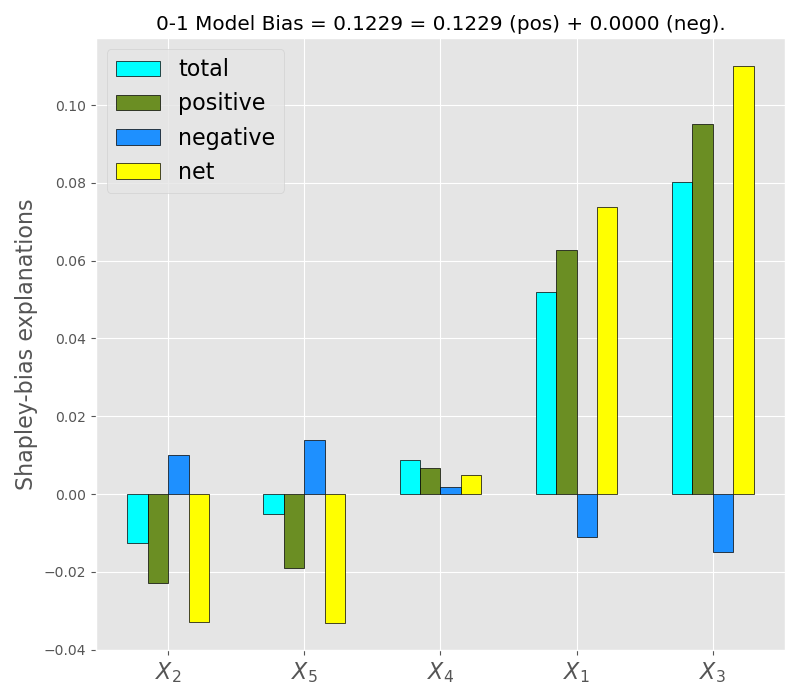}}
\caption{ \footnotesize Additive Shapley-bias explanations based on the game $v^{bias,ME}$ for the model \eqref{realisticmod}. }\label{fig::biasshapexplrealisticmod}
\end{figure}


As discussed in Section 4.2, the non-additive bias explanations measure the positive and negative contributions to the model bias, but not to each flow. To this end, we measure signed contributions to each positive and negative model bias by employing a game-theoretic approach, which has been explored in numerous works in the area of machine learning interpretability; see \citet{Lipovetsky2001,Strumbelji2010,LundbergLee}. In the spirit of \citet{Strumbelji2010}, we define a cooperative game in which the players are predictors and the payoff is their bias contributions and then compute corresponding additive Shapley values.

\paragraph{Group explainers.} Let $X \in \RR^n$ be predictors and $f$ a model. A generic {\it group explainer} of $f$ is denoted by
\begin{equation*}
E(S; X, f), \quad S \subset \{1,2,\dots,n\}.
\end{equation*}
We assume that $E(S; X,f)$ quantifies the attribution of each predictor $X_S$ with $S \subset\{1,2,\dots,n\}$ to the model value $f(X)$ and satisfies
\begin{equation*}
E(\varnothing, X, f)=\E[f(X)], \quad E(\{1,2,\dots,n\}; X,f)=f(X).
\end{equation*}

Relatively straightforward group explainers can be constructed using conditional and marginal game or game value. In particular, for a nonempty $S \subset \{1,2,\dots,n\}$ one can set a trivial group explainer as
\begin{equation}\label{grouppdpshap}
v(S; X,f) \quad \text{or} \quad \varphi_S[v]=\SHAP_S(X;f,v)=\sum_{i \in S} \SHAP_i(X;f,v) \quad \text{where} \quad v \in \{ \vce, \vpdp\}.
\end{equation}

\begin{definition}\label{def::biasgame} \rm
Let $X,G,f,\favdir_f$ be as in Definition \ref{def::predbiasexpl}. Let $E(\cdot \,; X,f)$ be a group explainer.
\begin{itemize}[label=$\bullet$]
  \item Cooperative bias-game $v^{bias}$ associated with $X,G,f$ and $E$ is defined by
  \[
  v^{bias}(S; G, E(\cdot; X,f))=W_1(E(S; X,f)|G=0,E(S; X,f)|G=1), \quad S\subset\{1,2,\dots,n\}.
  \]
 $v^{bias}(S)$ is the minimal cost of transporting $E(S)|G=0$ to $E(S)|G=1$ and vice versa.


\item Under optimal transport the positive bias-game and negative bias-game, respectively, are defined by: 
\begin{itemize}
  \item[$\circ$] $v^{bias+}(S)$ is the effort of transporting $E(S)|G=0$ in the non-favorable direction.
  \item[$\circ$] $v^{bias-}(S)$ is the effort of transporting $E(S)|G=0$ in the favorable direction.
\end{itemize}
The above values are specified in Lemma \ref{lmm::W1decomp} for $q=1$.

\item Net bias-game is defined by 
\[
v^{bias,net}=v^{bias+}-v^{bias+}.
\]

\item  The Shapley-bias explanations of $(X,f)$ based on the group explainer $E$ are defined by 
\begin{equation}\label{shabdef}
\begin{aligned}
\varphi^{bias}(f|G) = \varphi[v^{bias}], \quad \varphi^{bias\pm}(f|G)=\varphi[v^{bias\pm}], \quad \varphi^{bias,net}(f|G) = \varphi[v^{bias,net}]
\end{aligned}
\end{equation}
where $\varphi$ denotes the Shapley value \eqref{shapform} and where we suppressed the dependence on $X$ and $E$.
\end{itemize}
\end{definition}

Unlike the regular bias explanations which by construction are always non-negative, the Shapley-bias explanations are signed, that is, they can be both positive and negative. 

\begin{lemma} Given $(X,f)$ and the explainer $E$, the Shapley bias-explanations defined in \eqref{shabdef} satisfy
\begin{equation*}
\begin{aligned}
\sum_{i=1}^n \varphi_i^{bias}=\modbias_{W_1}(f|G), \quad \sum_{i=1}^n \varphi_i^{bias\pm}=\modbias_{W_1}^{\pm}(f|G), \quad \sum_{i=1}^n \varphi_i^{bias,net}=\modbias_{W_1}^{net}(f|G)
\end{aligned}
\end{equation*}
and, thus,
\begin{equation*}
\begin{aligned}
\varphi[v^{bias}] & =  \varphi[v^{bias+}] + \varphi[v^{bias-}] \\
\varphi[v^{bias,net}] &=  \varphi[v^{bias+}] - \varphi[v^{bias-}]. 
\end{aligned}
\end{equation*}
\end{lemma}
\begin{proof}
The result follows from \citet{Shapley} and the properties of the $W_1$-based model bias. 
\end{proof}

For Shapley-bias explanations based on the conditional and marginal games we have the following.
\begin{theorem}\label{gamebiasexplcont} Given $(X,f)$, let the conditional and marginal bias games be defined by 
\begin{equation*}
\begin{aligned}
v^{bias,\CE}(S;X,f) &=v^{bias}(S \,;\varphi_S[v^{CE}(\cdot;X,f)])\\
v^{bias,\ME}(S;X,f) &=v^{bias}(S \,;\varphi_S[v^{ME}(\cdot;X,f)])
\end{aligned}
\end{equation*}
The conditional and marginal Shapley-bias explanations have the following properties:
\begin{itemize}
  \item [(i)]  $|\varphi_i^{bias\pm}(f|G,\varphi_S[\vce])-\varphi_i^{bias\pm}(g|G,\varphi_S[\vce])| \leq C\|f-g\|_{L^2(P_X)}$.
  \item [(ii)]  $|\varphi_i^{bias\pm}(f|G,\varphi_S[\vpdp])-\varphi_i^{bias\pm}(g|G,\varphi_S[\vpdp])| \leq C\|f-g\|_{L^2(\widetilde{P}_X)}$.
\end{itemize}
\end{theorem}
\begin{proof}
The proof follows the same steps as in Theorem \ref{biasexplcont}.
\end{proof}

 \paragraph{Example.} Applying the above methodology to $\fhat$ and $G$ of the model \eqref{realisticmod} we compute the Shapley-bias explanations of predictors $X_i$, $i\in\{1,2,\dots,5\}$ using the group explainer $E(S)=\SHAP_S[\vpdp]$ defined in \eqref{grouppdpshap} for the construction of the bias-game. 
  
  The results are displayed in Figure \ref{fig::biasshapexplrealisticmod}. On the left, the explanations are plotted in increasing order of the positive bias, and in the middle plot by the total bias, while the right plot contains the information on all four types of biases. By comparing these to the non-additive bias explanation plots in Figure \ref{fig::biasexplrealisticmod} we see how the signed values provide further information on how the predictors contribute to the model bias.
  
  For example, from \eqref{realisticmod} we have that $X_3$, as a contributor to the model $\fhat$, favors the class $G=0$ since $\beta_3^+>0$ and $\beta_3^-=0$. Recall that $\beta_3^+$ captures the total contribution to the increase of the positive model bias plus the decrease (or resistance) to the negative model bias. The Shapley-bias explanations, however, allow one to estimate separately the (signed) contributions to both positive and negative model bias. 
  
In particular, the left plot of Figure \ref{fig::biasshapexplrealisticmod} informs us that $X_3$ in $\fhat$ contributes to the increase of the positive model bias (green), measuring the contribution to  pushing the subpopulation of the non-protected class towards the favorable direction, while its contribution to the negative model bias (blue) is negative, which indicates the resistance towards the subpopulation's pull in the non-favorable direction.

\subsection{Group Shapley-bias explanations}
\label{sec::groupbiasexpl}
It might be important for a practitioner to understand the main factors within the data itself that contribute to the bias in the response variable and not how the model structure contributes to it. To do this, one needs to generate bias explanations based on the conditional game $\vce$. The conditional game, when predictors are independent, coincides with the marginal game and the conditional expectations $\E[f(X)|X_S]$ can be computed through averaging with error control. However, under dependencies, the conditional expectations and corresponding Shapley-bias explanations are difficult to compute in light of the curse of dimensionality.

Another important aspect to consider is that highly dependent predictors carry similar information. For instance, in the case where a group of predictors is represented via a smaller collection of latent variables, the latent variable explanations are spread out among the predictors in that group; see \citet{Chen-Lundberg}. Under dependencies, for practical and business purposes, one may want to explain the information carried by the entire group rather than the predictors themselves.

The two issues mentioned above can be addressed simultaneously by adapting the ideas from \citet{Aas,Kotsiopoulos2020}. In particular, grouping predictors based on dependencies and utilizing specially-designed group explainers to compute the contribution of the group help unite the marginal and conditional approaches. Therefore, applying similar techniques, one can approximate the conditional Shapley-bias explanations of weakly independent groups using the marginal approach, which only requires averaging over a small dataset. Furthermore, grouping allows one to reduce complexity.

In what follows we adapt the techniques from \citet{Kotsiopoulos2020} to construct group Shapley-bias explanations. To this end, we first introduce required notation. Let $X \in \RR^n$ and $\{S_j\}_{j=1}^m$ be disjoint sets that partition the set of the predictors' indexes,
\begin{equation}\label{partition}
\textstyle N = \{1,2,\dots,n\} = \bigcup_{j=1}^m S_j, \quad \mathcal{P}=\{S_1,S_2,\dots,S_m\},
\end{equation}
so that $X_{S_1},X_{S_2},\dots,X_{S_m}$ form weakly independent groups such that within each group the predictors share significant amount of mutual information. Given a cooperative game $v$ on $N$, we define the quotient game by
\[
\textstyle v^{\mathcal{P}}(A)=v\big( \bigcup_{j \in U} S_j \big), \quad A \subset M=\{1,2,\dots,m\}.
\]
By design, $v^{\mathcal{P}}(A)$ is played by the groups, viewing the elements of the partition as players.

\begin{definition}\label{def:group-bias}
Given $X,f,G$ as in Definition \eqref{def::predbiasexpl}, and the partition $\mathcal{P}$ as in \ref{partition}.

\begin{itemize}[label=$\bullet$]
  \item The conditional and marginal group bias-games are defined by 
\begin{equation}\label{groupbiasgame}
v^{bias}_{\mathcal{P}}(A;X,G,f,v)=W_1\big( v^{\mathcal{P}}(A) |G=0, v^{\mathcal{P}}(A)|G=1 \big), \quad v \in \{\vce,\vpdp\}. 
\end{equation}

\item  The corresponding Shapley-bias explanations of $\{X_{S_j}\}_{j=1}^m$ are then defined by
\begin{equation*} 
\begin{aligned}
\varphi^{bias,\mathcal{P}}_{S_j}(f|X,G; v) &= \varphi_j[v^{bias}_{\mathcal{P}}(\cdot\,;v)], \quad v \in\{\vce,\vpdp\}.
\end{aligned}
\end{equation*}
\end{itemize}
\end{definition}

\begin{lemma}\label{lmm::explcontgroup} Given $X,f,G,\mathcal{P}$ as in Definition \ref{def:group-bias}. If $\{X_{S_j}\}_{j=1}^m$ are independent, then
\begin{equation}\label{unifmce}
\varphi^{bias,\mathcal{P}}_{S_j}(f|X,G; v^{\CE})=\varphi^{bias,\mathcal{P}}_{S_j}(f|X,G; v^{\ME}), \quad S_j \in \mathcal{P}.
\end{equation}
Consequently,
\[
|\varphi^{bias,\mathcal{P}}_{S_j}(f|X,G; v)-\varphi^{bias,\mathcal{P}}_{S_j}(g|X,G; v)| \leq C\|f-g\|_{L^2(P_X)}, \quad v \in \{\vce,\vpdp\}.
\]
\end{lemma}

\begin{proof}
By independence, we have $v^{\ME,\mathcal{P}}=v^{\CE,\mathcal{P}}$. Hence by \eqref{groupbiasgame} we obtain
\[
v^{bias}_{\mathcal{P}}(A;\vce)=v^{bias}_{\mathcal{P}}(A;\vpdp), \quad A \subset M
\]
and this yields \eqref{unifmce}. The stability argument can be carried out similarly to Lemma \ref{biasexplcont}. 
\end{proof}

Similar construction is used to compute positive and negative bias explanations $\varphi^{bias+,\mathcal{P}}_{S_j}$ and $\varphi^{bias-,\mathcal{P}}_{S_j}$, respectively.

\begin{remark}\rm
 The importance of equality \eqref{unifmce} is that the expression on the right-hand side can be computed via averaging using a dataset with $O(\tau^{-2})$ samples for a given error tolerance $\tau$. This makes the computation of the conditional explanation feasible. Furthermore, the complexity of computations becomes $O(2^m)$ where $m$ is the number of independent groups. For example, given a classification score and $X \in \RR^{100}$, having $100$ predictors split into $10$ independent groups, it is sufficient to use a dataset with $10000$ samples in order to compute conditional group Shapley-bias explanations of a classification score with error tolerance $\tau=0.01$ and complexity $O(2^{10} \cdot 10000^2)$, which is feasible and easily parallelizable. If the number of independent groups is still large the above technique can modified to incorporate recursive groupings.
\end{remark}

\section{On the application of the framework}\label{sec::application}

\subsection{Bias mitigation under regulatory constraints}

In this section, we will discuss how the fairness interpretability framework can be used for real-world applications in financial institutions that work under regulatory constraints.

An operational flow for model development in many FIs may consists of the following stages: 1) Model training, 2) Fair Lending Compliance governance review, and 3) Production, which includes model prediction and decision-making steps. Steps 1 and 3 are carried out by quantitative departments, while step 2 by the dedicated Compliance Office (CO), a department separate from business.
The CO provides oversight to the company’s compliance with federal and state regulations.

FIs are explicitly prohibited from collecting some protected  information on customers such as race and ethnicity (apart from mortgage lending), as stated by the ECOA. Furthermore, protected attributes cannot be used in training or inference. However, proxy information on the protected attribute such as the one derived from Bayesian Improved Surname and Geocoding (BISG) is allowed to be used by the compliance office solely for fairness  analysis \citep{Elliot2009}. Proxy information, however, must remain within the compliance office and the business does not (and should not) have access to the proxy data.

For fairness assessment, the CO carries out the bias assessment step. The CO can determine the main drivers contributing to model bias using our method and subsequently utilize bias mitigation methods. The bias mitigation step can include model postprocessing. However, in order to adhere to regulations, a post-processed model must not utilize the proxy attribute $\tilde{G}$ or any information on the joint distribution $(X,\tilde{G})$, such as probabilities $\P(\tilde{G}|X)$. The reasons for that are a) in the production step one can only have access to $X$, and b) a post-processed model is shared with business units that should be prevented from inferring the protected attribute from $X$.

Some rudimentary techniques for bias mitigation include recommendations on which predictors to drop from training or model post-processing via nullifying a given predictor by fixing its value. A more efficient technique has been proposed in our companion paper \citet{Miroshnikov2021b}. There we construct an efficient frontier over a family of compliant post-processed models utilizing the interpretability framework developed in the current article.  Other examples of compliant methods include those that vary hyper-parameters to get an efficient frontier, such as those in \cite{Schmidt2021}.

\subsection{Pedagogical example on bias mitigation}\label{subsec::pedexample}

In this section we provide a pedagogical example that showcases how to properly utilize the information on the positive and negative bias explanations when it comes to bias mitigation. A rudimentary mitigation technique one can employ is to construct a regulatory-compliant post-processed model by neutralizing an appropriate collection of predictors $X_S$. This is accomplished by fixing their values in the model to some reference values $x_S^*$ and setting $\tilde{f}(x;S,x^*)=f(x_S^*,x_{-S})$.

Often the objective of the bias mitigation in FIs is the reduction of the positive model bias which quantifies how much the model favors the majority class. In practice, regressor models are usually positively-biased, meaning  $\Bias_{W_1}^+(f|G)>0$ and $\Bias_{W_1}^-(f|G)=0$.

Taking into account the above discussion, let us assume for the sake of explanation that $f(X)=\sum_{i=1}^n f_i(X_i)$ is an additive and positively-biased model. Let $\beta^+_i$, $\beta^-_i$, where $i \in N=\{1,\dots,n\}$, be the positive and negative marginal bias explanations, respectively. Finally, let us decompose the predictor index set as follows: $N=N_+ \cup N_- \cup N_0$ where
\[
N_+=\{i: \beta_i^+>\beta_i^-\},  \quad N_-=\{i: \beta_i^->\beta_i^+ \}, \quad N_0=\{i: \beta_i^+=\beta_i^-\}.
\]
In this case, by Lemma 12 the model bias is given by
\[
\Bias_{W_1}(f|X,G)=\Bias_{W_1}^+(f|X,G)= \sum_{i \in N_+} (\beta_i^+ - \beta_i^-) - \sum_{i \in N_-} (\beta_i^- - \beta_i^+) > 0
\]
which illustrates the bias offsetting mechanism.

Note that neutralizing the predictor $i_0 \in N_-$ would cause the model bias, which is equal to the positive model bias, to increase, while neutralizing $i_1 \in N_+$would cause the model bias to decrease.

Thus, one approach to reduce the model bias is to rank order the predictors in $N_+$  by their net-bias explanations and, subsequently, neutralize them one by one in that order. This will incrementally reduce the positive model bias until the point where neutralizing the next predictor causes the model bias to become equal to the negative model bias (with the positive model bias being zero), which operates as a stopping criterion of the approach. This simple and rather naïve strategy illustrates that a) the decomposition of explanations is useful for bias mitigation and that b) neutralization of biased predictors ranked by total bias contribution is not always the optimal strategy.

\subsection{Example on census income dataset}

\begin{algorithm}
\SetAlgoLined 
\KwData{Model $f$, dataset $D = (X,G)$ with $m$ samples, $X\in \RR^n$ and $G\in\{0,1\}$, and model explainer $E_i$.
}
 \KwResult{Output the model biases $\Bias_{W_1}^{\pm}(f|X,G)$ and predictor bias explanations $\{\beta_i^{\pm}\}_{i=1}^n$. } 
Partition the predictors in $D$ according to $G=k$, $k \in \{0,1\}$. This yields subsets $D_{X,0},\ D_{X,1}$.\\
For each $k \in \{0,1\}$ evaluate $f(x)$ on $D_{X,k}$ to obtain the set of subpopulation model values $S_k$.\\
For each $k \in \{0,1\}$ compute the empirical CDF $\hat{F}_k$ of $f(X)|G=k$ based upon $S_k$.\\
$\Bias_{W_1}^{\pm}(f|X,G):= \int_{\mathcal{P}_{\pm}} |\hat{F}^{[-1]}_0-\hat{F}_1^{[-1]}| \, dp $ with $\mathcal{P}_{\pm}$ as in Definition \ref{def::posnegmodbias}.\\
\For{$i$ in $\{1,\dots,n\}$}
{
For each $k \in \{0,1\}$ evaluate $E_i(x)$ on $D_{X,k}$ to obtain the set of subpopulation values $S_{i,k}$.\\
For each $k \in \{0,1\}$ compute the empirical CDF $\hat{F}_{i,k}$ of $E_i(X)|G=k$ based upon $S_{i,k}$.\\
$\beta_i^{\pm}:= \int_{\mathcal{P}_{i\pm}} |\hat{F}^{[-1]}_{i,0}-\hat{F}_{i,1}^{[-1]}| \, dp $ with $\mathcal{P}_{i\pm}$ as in Definition \ref{def::predbiasexpl}.\\
}
 \caption{Model bias and bias explanations}\label{biasalgo}
\end{algorithm}

In this section, we showcase the application of the framework to the 1994 Census Income dataset from the UCI Machine Learning Repository \citep{Dheeru2017}.

This dataset contains fourteen predictors and a dependent variable $Y$ that indicates if an individual earns more or less than \$$50$K annually. After investigating the predictors, we removed the protected attributes `sex', `race', `age', and `native-country'. We also excluded `fnlwght' and `relationship', the latter due to its high dependence with `sex' since in the dataset the categorical values `Husband' and `Wife' correspond to `Male' and `Female', respectively. The remaining seven predictors used for model training are `workclass', `education-num', `occupation', `marital-status', `capital-gain', `capital-loss', and `hours-per-week'.


\begin{figure}
\centering
\subfloat[\footnotesize Feature importance.]{\includegraphics[width=0.45\textwidth]{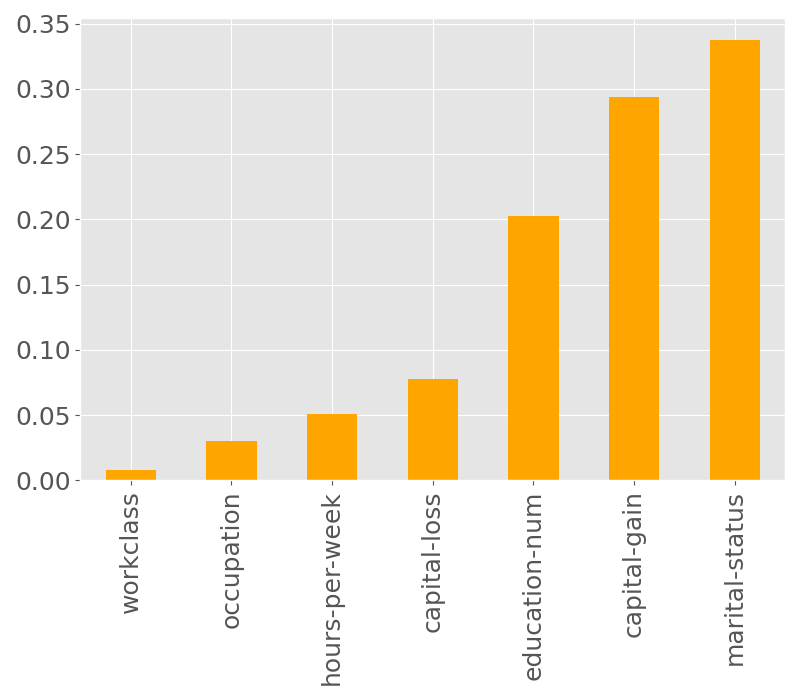}\label{fig::featimp}}
~~
\subfloat[\footnotesize Score subpopulation CDFs; bias $\approx 0.19$.]{\includegraphics[width=0.45\textwidth]{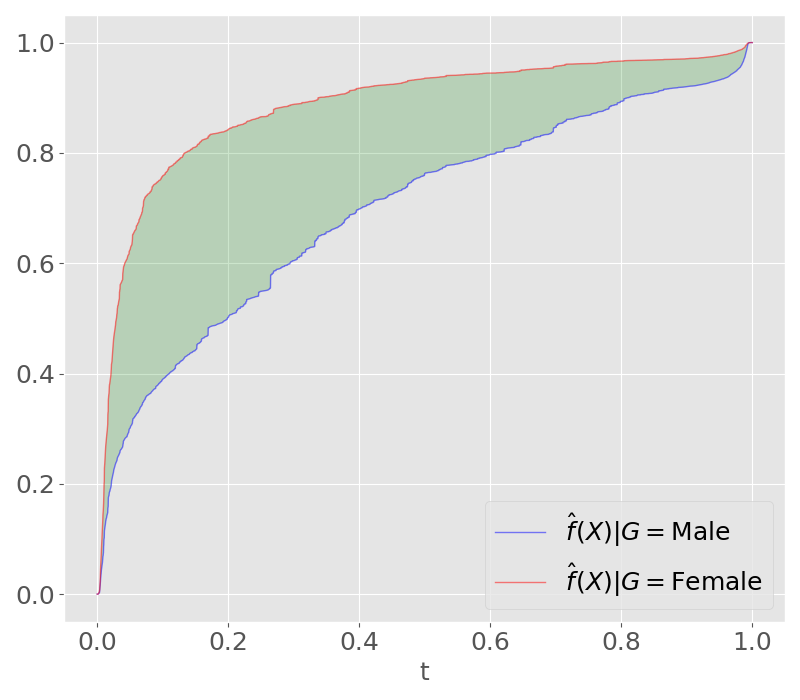}\label{fig::posbiascdf}}
\caption{ \footnotesize Model training and protected attribute analysis.} 
  \label{fig::applres}
\end{figure}


For the model training, we use the training dataset $D_{train}$ with $32561$ samples to build a classification score
\[\fhat(x)=\widehat{\PP}(Y=\text{`$>$50K'}|X=x),\]
using Gradient Boosting. For training we use the following parameters: \texttt{n\_estimators}=$200$, \texttt{min\_samples\_split}=$5$, \texttt{subsample}=$0.8$, \texttt{learning\_rate}=$0.1$. The feature importance of each predictor can be seen in Figure \ref{fig::featimp}, with the most significant predictors being `marital-status', `capital-gain', and `education-num'. 

Performance metrics for the GBM model on the trained dataset, and test dataset with $16251$ samples, were evaluated. Specifically, the mean logloss on the train and test set is approximately $0.288$ and $0.292$ respectively, and the AUC is $0.922$ and $0.918$ respectively.

The focus of the application is to evaluate and explain the model bias with respect to the protected attribute $G=$`sex', with values `Female' and `Male', where `Female' is the protected class. To this end, following the steps in Algorithm \ref{biasalgo}, we form the dataset $S$ containing the classification scores 
\[
S = \big\{\fhat(x^{(i)}): (x^{(i)},y^{(i)})\in D_{train} \big\},
\] 
and partition it based on each class of $G$. This yields the sets $S_{M}$ and $S_{F}$ containing the classification scores for `Female' and `Male' respectively, which we use to construct the empirical CDFs of the subpopulation scores, $\hat{F}_{Female}$ and $\hat{F}_{Male}$, using the \texttt{ECDF} class from the \texttt{statsmodels} library.

  Figure \ref{fig::posbiascdf} depicts the empirical CDFs, where we see that the model has almost exclusively positive bias, and the positive direction is assumed to be $\favdir_{\hat{f}}=1$.  To confirm this observation, we subsequently compute the positive and negative model biases by integrating the difference of the two CDFs over the sets where $\hat{F}_{Female}> \hat{F}_{Male}$ and $\hat{F}_{Female}<\hat{F}_{Male}$, respectively, as indicated in Definition 
\ref{def::posnegmodbias}. This yields the following values:
\[
\Bias^+_{W_1}(\fhat|X,G) \approx 0.19, \quad \Bias^-_{W_1}(\fhat|X,G) \approx 0.00.
\]

To understand the contributions of the predictors to the model bias, we next construct the bias explanations based on the marginal model explainer. To accomplish this, we subsample the predictors from the training set, and obtain a background dataset $D_X$ with $m=4000$ samples. Next, we compute the model explanations for each predictor $X_i$ yielding the sets
  \[
  S_{E_i} = \Big\{ \tfrac{1}{m}\sum_{x\in D_X}\fhat(x_i^{*},x_{-i}), \ x^{*}\in D_X \Big\}.
  \]
Similar to obtaining the model bias, we then partition $S_{E_i}$ based on each class of $G$ and obtain the empirical CDFs of $E_i(X)|G=g, \ g\in \{\text{`Female'},\text{`Male'}\}$, which are then used to compute the bias explanations $\beta_i^{\pm}$ according to Definition \ref{def::predbiasexpl}. These are depicted in Figure \ref{fig::bepincome} and are ranked in ascending order of the positive bias. All the values for the negative bias explanations are close to zero, which further indicates the positively biased nature of the predictors.  Observe that the most positively contributing predictor to the model bias is `marital-status' by far with  value $\approx 0.12$.

Since `marital-status' is the most impactful, it merits further investigation into its effect on the model bias. To this end, we group the different values of `marital-status' into three categories: $M_1 =$`never-married', $M_2=$`married', and $M_3=$`was-married'. Then, we segment the dataset $S$ of classification scores into three subsets $S_{M_i}, \ i\in \{1,2,3\}$, that correspond to the aforementioned categories. To gain further understanding on how each of these categories contributes to the model bias, we compute the model bias on each segment. The negative model bias on each segment turns out to be zero, while the positive model biases are plotted in Figure \ref{fig::posmbiasmarital}. The plot indicates that the category `never-married' exhibits an insignificant level of bias, while there is some substantial positive bias in `married' and `was-married'.


\begin{figure}
\centering
\subfloat[\footnotesize Bias explanations.]{\includegraphics[width= 0.45\textwidth]{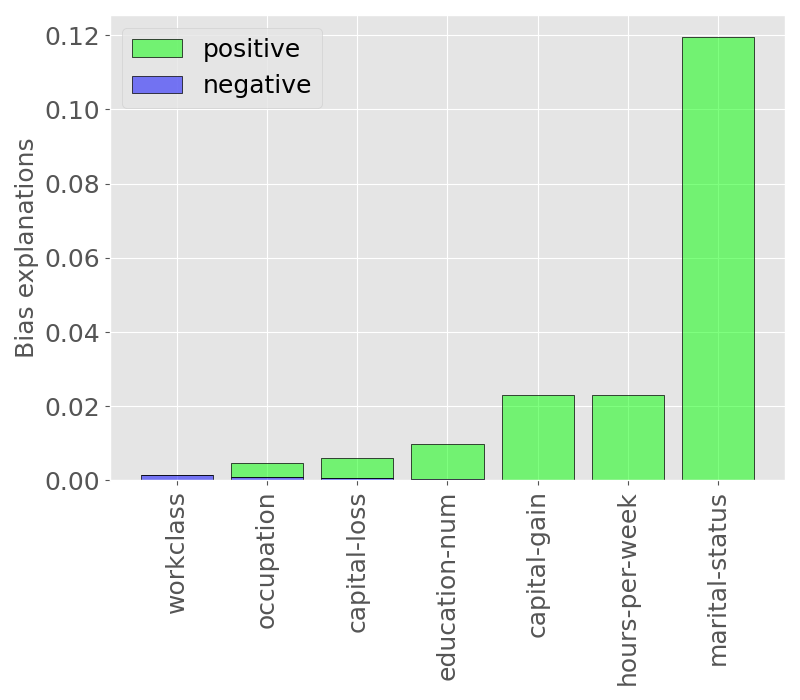}\label{fig::bepincome}}
~~
\subfloat[\footnotesize Model bias segmented by marital status.]{\includegraphics[width= 0.45\textwidth]{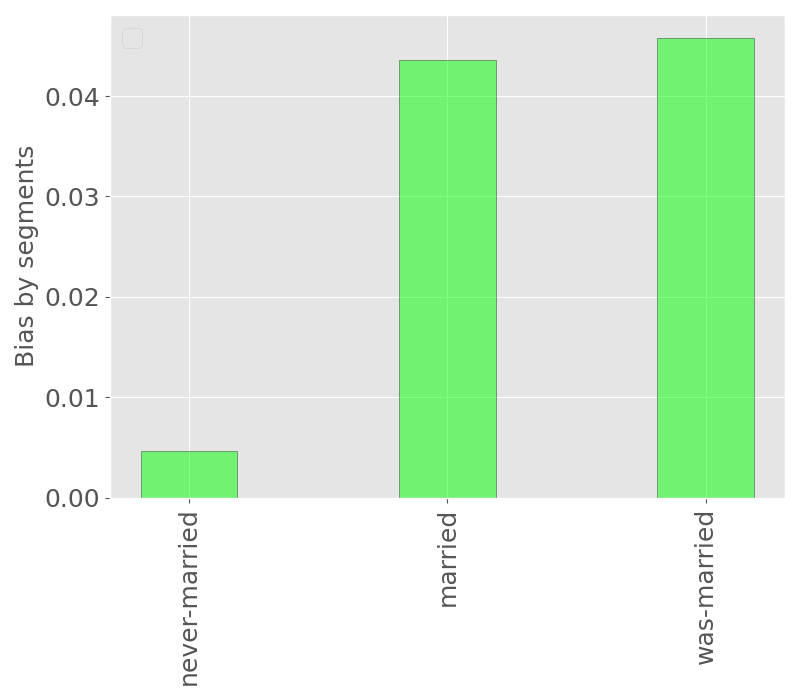}\label{fig::posmbiasmarital}}
\caption{ \footnotesize Model bias explanations.} 
  \label{fig::bepanalysis}
\end{figure}


\begin{figure}
\centering
\subfloat[\footnotesize Score subpopulation CDFs; bias $\approx$ 0.10.]{\includegraphics[width= 0.45\textwidth]{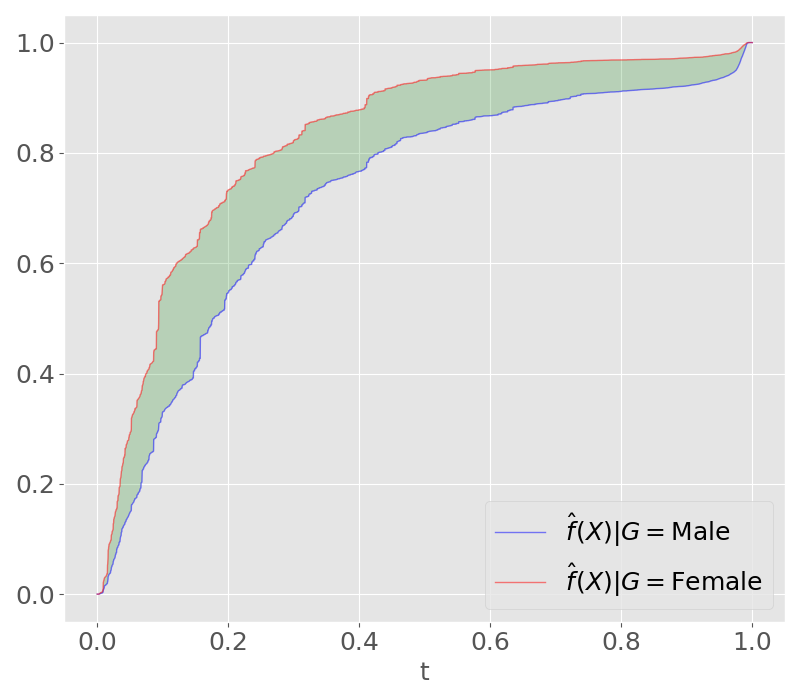}}
~~
\subfloat[\footnotesize Bias explanations.]{\includegraphics[width= 0.45\textwidth]{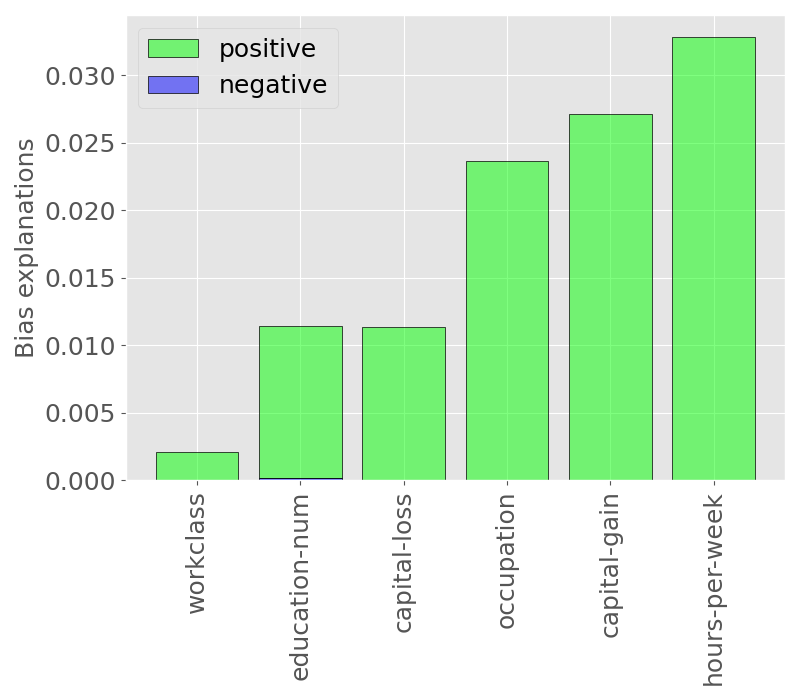}}
\caption{ \footnotesize Bias explanations for the re-trained model without `marital-status' predictor.} \label{fig::bepanalysis_adj}
\end{figure}


Given the above analysis, one can attempt to reduce the model bias either by applying the postprocessing technique discussed in Section \ref{subsec::pedexample}, or, alternatively, retrain the model by dropping some of the biased predictors. We showcase the latter approach by dropping `marital-status' and retraining the model with the same parameters. We check the performance of the new model on the train and test sets. The mean logloss is $0.358$ and $0.363$ respectively, and AUC is $0.862$ and $0.855$ respectively. We then compute the model bias and bias explanations; see Figure \ref{fig::bepanalysis_adj}. The positive model bias has been reduced to approximately $0.10$, while the negative stays zero. The trade-off is a drop in performance, as seen by the performance metric values above. The bias explanations in the retrained model have slightly increased since `marital-status' was dropped and the importance of the remaining predictors increased.

 We would like to point out that the technique used above might not lead to bias reduction under the presence of strong dependencies, since other predictors could be used as proxies for the dropped predictor. However, the postprocessing technique outlined in Section \ref{subsec::pedexample} modifies the model score directly and the dependencies do not play a significant role. Keep in mind that this technique is rather crude and one may opt to employ the postprocessing methods described in \cite{Miroshnikov2021b} which apply to numerical predictors, but can be adjusted for categorical ones.

\section{Conclusion}

In this paper, we presented a novel bias interpretability framework  for measuring and explaining bias in classification and regression models at the level of a distribution that utilizes the Wasserstein metric and the theory of optimal mass transport. We introduced and theoretically characterized bias predictor attributions to the model bias and constructed additive bias explanations utilizing cooperative game theory. To our knowledge, bias interpretability methods at the level of a regressor distribution have not been addressed in the literature before.

At a higher level, the model bias is a non-trivial superposition of predictor bias attributions. The bias explanations we introduced determine the contribution of a given predictor to the model bias. However, any two or more predictors will interact in the context of the bias explanations. For example, if one predictor favors the non-protected class and the other favors the protected class, it might be possible that when both predictors are utilized by the model the total effect on model bias is zero. This phenomenon opens up numerous avenues for future research to investigate the interactions of predictors across subpopulation distributions in the context of bias explanations. This is where ML interpretability techniques can come into play and aid with the study of predictor interactions in the model bias.

To make bias explanations additive we utilized cooperative game theory which lead to additive Shapley-bias explanations. These explanations rely on the Shapley formula, which makes them computationally expensive. The intractability of such calculations can be mitigated by grouping predictors based on dependencies and then computing the Shapley bias attributions for each group (via a quotient game) which reduces the dimensionality. However, if the number of groups is large, the issue of computational intensity remains. Thus, a possible research direction is to investigate methods that allow for approximation of the additive bias explanations and their fast computations.

In this paper, we formulated a methodology that computes the model bias and quantifies the contribution of predictors to that bias. However, an important application of the bias explanation methodology lies in bias mitigation, which will be useful in regulatory settings such as the financial industry, and may utilize information about the main drivers of the model bias. This will be investigated in our upcoming paper. The framework is generic and in principle can be applied to a wide range of predictive ML systems. For instance, it might be insightful to understand the predictor attributions to probabilistic differences of populations studied in physics, biology, medicine, economics, etc.

\section*{Acknowledgment}
{
The authors would like to thank Steve Dickerson (CAO, Decision Management at Discover Financial Services (DFS)), Raghu Kulkarni (VP, Data Science at DFS)  and Melanie Wiwczaroski (Sr. Director, Enterprise Fair Banking at DFS) for formulation of the problem as well as helpful business and compliance insights. We also thank Patrick Haggerty (Director \& Senior Counsel at DFS) and Kate Prochaska (Sr. Counsel \& Director, Regulatory Policy at DFS) for their helpful comments relevant to regulatory issues that arise in the financial industry. We also would like to thank professors Markos Katsoulakis and Robin Young from the University of Massachusetts Amherst, professor Matthias Steinr\"{u}cken from the University of Chicago and professor Hangjie Ji from North Carolina State University for their valuable comments and suggestions that aided us in writing this article.
}


\begin{appendices}

\section*{Appendix}


\section{Kantorovich transport problem}\label{sec::kpminimization}

To formulate the transport problem we need to introduce the following notation. Let $\mathcal{B}(\RR^k)$ denote the $\sigma$-algebra of Borel sets. The space of all Borel probability measures on $\RR^k$ is denoted by $\mathscr{P}(\RR^k)$. The space of probability measure with finite $q$-th moment is denoted by 
\[
\mathscr{P}_q(\RR^k)=\{ \mu\in \mathscr{P}(\RR^k): \int_{\RR^k} |x|^q d\mu(x) < \infty \}.
\]

\begin{definition}[{\bf push-forward}]\label{def::push-forward}\rm
\begin{itemize}

  \item [(a)] Let $\PP$ be a probability measure on a measurable space $(\Omega,\mathcal{F})$. Let $X \in \RR^p$ be a random vector defined on $\Omega$. The push-forward probability distribution of $\PP$ by $X$ is defined by
  \[
  P_X(A):=\PP\big( \{ \omega \in \Omega: X(\omega) \in A \} \big).
  \]

  \item [(b)] Let $\mu \in \mathscr{P}(\RR^k)$ and $T:\RR^k \to \RR^m$ be Borel measurable, the pushforward of $\mu$ by $T$, which we denote by $T_{\#}\mu$ is the measure that satisfies
  \[
  (T_{\#}\mu)(B)=\mu\big(T^{-1}(B)\big), \quad B \subset \mathcal{B}(\RR^k).
  \]  

  \item [(c)] Given measure $\mu=\mu(dx_1,dx_2,...,dx_k) \in \mathscr{P}(\RR^k)$ we denote its marginals onto the direction $x_j$ by $(\pi_{x_j})_{\#}\mu$ and the cumulative distribution function by 
  \[
  F_{\mu}(a_1,a_2,\dots,a_k)=\mu( (-\infty,a_1]\times(-\infty,a_2] \dots, (-\infty,a_k])
  \]
\end{itemize}
\end{definition}

\begin{theorem}[\bf change of variable]\label{prop::changeofvar} Let $T:\RR^k \to \RR^m$ be Borel measurable map and $\mu \in \mathscr{P}(\RR)$. Let $g\in L^1(\RR^m, T_{\#}\mu)$. Then
\[
\int_{\RR^m} g(y) T_{\#}\mu(dy) = \int_{\RR^k} g(T(x)) \, \mu(dx).
\]
\end{theorem}
\begin{proof}
See \citet[p.~196]{Shiryaev}.
\end{proof}

\begin{definition}[\bf Kantorovich problem on $\RR$]
 Let $\mu_1,\mu_2 \in \mathscr{P}(\RR)$ and $c(x_1,x_2) \geq 0$ be a cost function.  Consider the problem 
 \[
   \inf_{\gamma \in \Pi(\mu_1,\mu_2) } \Bigg\{ \int_{\RR^2} c(x_1,x_2) \gamma(dx_1,dx_2) \bigg\} =: \mathscr{T}_c(\mu_1,\mu_2)
\]
where  $\Pi(\mu_1,\mu_2)=\{ \gamma \in \mathscr{P}(\RR^2): (\pi_{x_j})_{\#}\gamma = \mu_j \}$
denotes the set of transport plans between $\mu_1$ and $\mu_2$, and $\mathscr{T}_c(\mu_1,\mu_2)$ denotes the minimal cost of transporting $\mu_1$ into $\mu_2$.
\end{definition}


\begin{definition} Let $q \geq 1$ and let $d$ be a metric on $\RR$. Let the set $\mathscr{P}_q(\RR^n;d)=\{\mu \in \mathscr{P}(\RR^n): \int d(x,x_0)^q d\mu(x) < \infty\}$ where $x_0$ is any fixed point. The Wasserstein distance $W_q$ on $\mathscr{P}_q(\RR^n;d)$ is defined by
\[
\begin{aligned}
  W_q(\mu_1,\mu_2;d) := \mathscr{T}^{1/q}_{d(x_1,x_2)^q}(\mu_1,\mu_2), \quad \mu_1,\mu_2 \in \mathscr{P}_q(\RR^n;d)
\end{aligned}
\]
where
\[
\mathscr{T}_{d(x_1,x_2)^q}(\mu_1,\mu_2) = \inf_{\gamma \in \mathscr{P}(\RR^2)} \bigg\{ \int_{\RR^2} d(x_1,x_2)^q d\gamma, \quad \gamma \in \Pi(\mu_1,\mu_2) \bigg\}.
\]
We drop the dependence on $d$ in the notation of the Wasserstein metric when $d(x,y)=|x-y|$.
\end{definition}


The following theorem contains well-known facts established in the texts such as \citet{Shorack1986,Villani2003,Santambrogio2015}.
\begin{theorem}\label{thm::transportprop} Let $\mu_1,\mu_2 \in \mathscr{P}(\RR)$. Let $c(x_1,x_2)=h(x-y) \geq 0$ with $h$ convex and let
\begin{equation*}
\pi^* := (F^{-1}_{\mu_1},F^{-1}_{\mu_2})_{\#} \lambda|_{[0,1]} \in \mathscr{P}(\RR^2)
\end{equation*}
where $\lambda|_{[0,1]}$ denotes the Lebesgue measure restricted to $[0,1]$. Suppose that $\mathscr{T}_c(\mu_1,\mu_2)<\infty$. Then

\begin{itemize}
  \item [(1)] $\pi^* \in \Pi(\mu_1,\mu_2)$ and $F_{\pi^*}=\min(F(a),F(b))$.

  \item [(2)] $\pi^*$ is an optimal transport plan that is
  \[
  \mathscr{T}_c(\mu_1,\mu_2)=\int_{\RR^2} h(x_1-x_2) \, d\pi^*(x_1,x_2).
  \]

  \item [(3)] $\pi^*$ is the only monotone transport plan, that is, it is the only plan that satisfies the property 
  \[
  (x_1,x_2),(x_1',x_2')\in {\rm supp(\pi^*)} \subset \RR^2\quad   x_1 < x_1' \quad \Rightarrow \quad x_2 \leq x_2'.
  \]

  \item [(4)] If $h$ is strictly convex then $\pi^*$ is the only optimal transport plan. 

  \item [(5)] If $\mu_1$ is atomless, then $\pi^*$ is determined by the monotone map $T^*=F_{\mu_2}^{[-1]}\circ F_{\mu_1}$, called an optimal transport map. Specifically, $\mu_2=T^*_{\#}\mu_1$ and hence $\pi^* = (I,T^*)_{\#}\mu_1$, where $I$ is the identity map. Consequently,
\[
\int_{\RR^2} h(x_1-x_2) \, d\pi^*(x_1,x_2) = \int_{\RR} h(x_1-T^*(x_1)) d\mu_1(x_1) = \E[X_1-T^*(X_1)], \quad \mu_1=P_{X_1}.
\]

\item [(6)] For $q \in [1,\infty)$, we have
\[
\begin{aligned}
{W_q}^q(\mu_1,\mu_2) &= \mathscr{T}_{|x_1-x_2|^q}(\mu_1,\mu_2) = \int_{\RR^2} |x_1-x_2|^q d\pi^*(x_1,x_2) \\
&= \int_0^1 |F^{[-1]}_{\mu_1}(p)-F^{[-1]}_{\mu_2}(p)|^q dp < \infty.
\end{aligned}
\]

\end{itemize}
\end{theorem}

\begin{definition}
  Given a set of probability measures $\{ \mu_j \}_{j=1}^J \subset \mathscr{P}_2(\RR^n)$, with $J \geq 1$, with finite second moments, and weights $\{ \omega_j \}_{j=1}^J$, the Wasserstein barycenter is the minimizer of the map $\nu \to \sum_{j\in J} \omega_j W_2^2(\nu, \mu_j).$
\end{definition}


\section{Proofs and auxiliary lemmas}\label{app::auxlemm} 

\begin{definition}[{\bf geometric continuity}]\label{def::geomcontmetrics}
Let $D(\cdot,\cdot)$ be a metric on $\mathscr{P}_k(\RR^n)$, with $k \geq 0$. We say that $D$ is continuous with respect to the geometry of the distribution if for any $\mu\in\mathscr{P}_k(\RR^n)$ $\lim_{\epsilon \to 0+} D(\mu, T_{\eps\#}\mu )=0$, for any family $\{T_{\eps}\}_{\eps>0}$ of continuously differentiable maps from $\RR^n$ to $\RR^n$ that satisfy 
\begin{itemize}
  \item[$(i)$] $det \,\nabla T_{\eps}>0$.

  \item[$(ii)$]  The family $\{T_{\eps}-I\}_{\eps}$ has a common compact support.

  \item[$(iii)$] $T_{\eps} \to I$ uniformly on $\RR^n$ as $\eps \to 0$, where $I$ is the identity map.
\end{itemize}
\end{definition}

\begin{definition}[{\bf invariance}]\label{def::invarmetrics}
Let $D(\cdot,\cdot)$ a metric on $\mathscr{P}_k(\RR^n)$. Let $T:\RR^n \to \RR^n$ be a map such that $T_{\#}\mu \in \mathscr{P}_k(\RR^n)$ for every $\mu \in \mathscr{P}_k(\RR^n)$.
We say that $D$ is invariant under the transformation $T$ if $D(\mu_1,\mu_2)=D\big( T_{\#}\mu_1, T_{\#}\mu_2 \big)$.
\end{definition}

\paragraph{\bf Proof of Theorem \ref{thm::geomcontWass}}
\begin{proof}
Let $q\in[1,\infty)$. Let $T_{\epsilon}$ be a family of maps from $\RR$ to $\RR$ as in Definition \ref{def::geomcontmetrics}. Take $\mu \in \mathscr{P}_q(\RR)$. Since $T_{\eps}-I$ has compact support, there is a bounded $B\subset \RR$ such that $T_{\eps}(x)=x$ for all $x \in B^c$. Thus,
\[
\int_{\RR} |x|^q d T_{\eps\#}\mu(x) = \int_{\RR} |T_{\eps}(x)|^q d \mu(x)=\int_{B} |T_{\eps}(x)|^q d \mu(x)+\int_{B^c} |x|^q d \mu(x) < \infty
\]
and hence $T_{\eps\#}\mu \in \mathscr{P}_q(\RR)$.

Next, consider a probability measure $\pi = (I,T_{\eps})_{\#}\mu$. By construction, its marginals are $\mu$ and $T_{\eps\#}\mu$ and hence $\pi$ is a transport plan. Then, Lemma \ref{prop::changeofvar} and the definition of the distance $D_{W_q}$ imply
\[
D_{W_q}^q(\mu_{\eps},T_{\eps\#}\mu) \leq \int_{\RR^2} |x_1-x_2| d\pi(x_1,x_2) = \int_{\RR} |x_1-T_{\eps}(x_1)| d\mu(x_1).
\]
Sending $\eps \to 0$ in the above inequality, and using the assumption that $I-T_{\eps} \to 0$ uniformly in $\RR$, we conclude that $D_{W_q}^q(\mu,T_{\eps\#}\mu) \to 0$. This proves the statement $(a)$.

Let $T:\RR\to\RR$ be continuous and strictly increasing. Let $q\in[1,\infty)$. Suppose that $D_{W_q}$ on $\mathscr{P}_q(\RR)$ is invariant under $T$. Let $\mu_1=\delta_{a}$ and $\mu_2=\delta_{b}$ for $a<b$. Then by invariance we obtain
\begin{equation*}
(T(b)-T(a))^q = D_{W_q}^q(T_{\#}\mu_1,T_{\#}\mu_2) = D_{W_q}^q(\mu_1,\mu_2) = (b-a)^q.
\end{equation*}
Since $a,b$ are arbitrarily chosen, we conclude that $T(x)=x+C$. This proves  (b).
\end{proof}

\paragraph{\bf Proof of Lemma \ref{lmm::biastestfunc}}
\begin{proof}
First, take any $M \in Lip_1(\Chi,\mathscr{P}(\{0,1\}))$ and set $\varphi(x)=[M(x)](\{0\})$. Then
\[
d(x,y) \geq D_{TV}(M(x),M(y))=\frac{1}{2} \sum_{a \in \{0,1\}} \big|[M(x)](a)-[M(y)](a)\big|=|\varphi(x)-\varphi(y)|
\]
and hence $\tilde{\varphi}=\varphi-\tfrac{1}{2} \in \A^*$.

Next, take $\tilde{\varphi} \in \A^*$. Let $\varphi=\tilde{\varphi}+\tfrac{1}{2}$. Take $x \in \Chi$ and pick $M(x)$ to be a probability measure such that $[M(x)](\{0\})=\varphi(x)$. Then $M \in Lip_1(\Chi,\mathscr{P}(\{0,1\});D_{TV},d)$ 

The lemma follows from the above and the fact that $M_{\mu}(\{0\})-M_{\nu}(\{0\})=\int \tilde{\varphi} d[\mu-\nu]$.
\end{proof}

\begin{lemma}\label{biasscale}
  Let $d(x,y)=\|x-y\|$ be a norm on $\RR^n$. Let $T(x)=cx+x_0$ with $c > 0$. Then
  \[
    D_{rc}(T_{\#}\mu,T_{\#}\nu; D_{TV},d ) = D_{rc}(\mu,\nu;D_{TV},d_c),\quad \mu,\nu \in \mathscr{P}_1(\RR^n;d),
     \]
  where $d_c(x,y)=cd(x,y)$.
\end{lemma}

\begin{proof} 
  \[
  \begin{aligned}
    D_{rc}(T_{\#}\mu,T_{\#}\nu;D_{TV},d) &=\sup_{\varphi \in Lip_1(\RR^n,[0,1];d)} \int \varphi(x) [\tilde{\mu}-\tilde{\nu}](dx)\\
  &=\sup_{\varphi \in Lip_1(\RR^n,[0,1];d )} \int \varphi(cx+x_0)[\mu-\nu](dx) \\
  &=\sup_{u \in Lip_1(\RR^n,[0,1];d_c )} \int u(x)[\mu-\nu](dx) \\
  &= D_{rc}(\mu,\nu;D_{TV},d_c).
  \end{aligned}
  \]
\end{proof}

  \begin{lemma}\label{transpscale}
  Let $d(x,y)=\|x-y\|$ be a norm on $\RR^n$ and $\mu,\nu \in \mathscr{P}_1(\RR^n;d)$. Let $c>0$. Then
  \begin{itemize}
    \item [(i)] $W_1(\mu,\nu; d_c) = c\, W_1(\mu,\nu;d)$, $d_c(x,y)=cd(x,y)$.
  
    \item [(ii)] For any $T(x)=cx+x_0$
    \[
  \begin{aligned}
    W_1(T_{\#}\mu,T_{\#}\nu;d) &= c W_1(\mu,\nu;d).\\
  \end{aligned}
    \]
  \end{itemize}
  \end{lemma}
  \begin{proof}
  The lemma follows directly from the definition of $W_1$ and the fact that $d$ is a norm.
  \end{proof}

\paragraph{\bf Proof of Lemma \ref{lmm::scaling}}
\begin{proof}
The proof follows from Lemma \ref{biasscale} and Lemma \ref{transpscale}.
\end{proof}

\paragraph{\bf Proof of Theorem \ref{thm::rbcbiaswasserstconn}}
\begin{proof}
  Take any $L>0$ and $x_*$ such that the supports of $\mu$ and $\nu$ are contained in $B(x_*,\frac{L}{2};d)$. By the Kantorovich-Rubinstein duality theorem \citep{Kantorovich1958,Dudley1976}, we have
  \[
  W_1(\mu,\nu;d) = \sup \Big\{ \int u(x) [\mu-\nu](dx), \, u \in Lip_1(\RR^n;d) \Big\}.
  \]
  
  Since $Lip_1(\RR^n,[0,L];d) \subset Lip_1(\RR^n;d)$ we have 
  \[
  \begin{aligned}
  \sup \Big\{ \int \tilde{u}(x) [\mu-\nu](dx), \, \tilde{u} \in Lip_1(\RR^n,[0,L];d) \Big\} \leq W_1(\mu_1,\mu_2;d).
  \end{aligned}
  \]
  
  Next, take any $u \in Lip_1(\RR^n;d)$. Observe that
  \[
  u_0:=\inf_{x \in B(x_*,L/2)} u(x)= \big( u(x_*) + \inf_{x \in B(x_*,L/2)} (u(x)-u(x_*)) \big) \in [u(x_*)-\tfrac{L}{2},u(x_*)+\tfrac{L}{2}].
  \]
  Define 
  \[
  \tilde{u}(x)=\min(\max(u(x)-u_0,0),L).
  \] 
  Note that $\tilde{u} \in Lip_1(\RR^n,[0,L];d)$. Furthermore,
  \[
  0 \leq u(x)-u_0 \leq \sup_{z \in B(x_*,\frac{L}{2})} d(x,z) \leq L, \quad x \in B(x_*,\tfrac{L}{2})
  \]
  and hence $\tilde{u}=u(x)-u_0$ for $x \in B(x_*,\tfrac{L}{2})$.
  Then, since $\mu$ and $\nu$ have support in $B(x_*,\tfrac{L}{2})$, we have
  \[
  \int u(x) [\mu-\nu](dx)=\int \tilde{u}(x) [\mu-\nu](dx)
  \]
  and hence
  \[
  \sup \Big\{ \int \tilde{u}(x)[\mu-\nu](dx), \, \tilde{u} \in Lip_1(\RR^n,[0,L];d) \Big\} \geq W_1(\mu,\nu;d).
  \]
  Thus, we conclude 
  \begin{equation}\label{altwasserst}
  W_1(\mu,\nu;d)=\sup \Big\{ \int \tilde{u}(x) [\mu-\nu](dx), \, \tilde{u} \in Lip_1(\RR^n,[0,L];d) \Big\}
  \end{equation}
  for any norm $d$ and any ball $B(x_*,\tfrac{L}{2};d)$ containing the supports of $\mu$ and $\nu$. 
  
  Let $d_{(1/L)}(x,y)=\frac{1}{L}\|x-y\|=\frac{1}{L}d(x,y)$. Then $B(x_*,\tfrac{1}{2};d_{(1/L)})=B(x_*,\tfrac{L}{2};d)$ and hence using \eqref{altwasserst} and Lemma \ref{transpscale}, we obtain
  \[
  \begin{aligned}
  &\frac{1}{L}W_1(\mu,\nu;d)=W_1(\mu,\nu;d_{(1/L)})\\
  &\quad =\sup \Big\{ \int \tilde{u}(x) [\mu-\nu](dx), \, \tilde{u} \in Lip_1(\RR^n,[0,1];d_{(1/L)}) \Big\}=D_{rc}(\mu,\nu;D_{TV},d_{(1/L)})
  \end{aligned}
  \]
  which proves the equality.
  \end{proof}

\paragraph{\bf Proof of Lemma \ref{lmm::cdfquantconn}}
\begin{proof}
Define the set
\[
A_0 = \{(p,t)\in (0,1) \times \RR:  F_1(t) < p \le F_0(t) \}.
\]
Note $(p,t) \in A_0$ implies $t \in \Tcal_0$. Hence, applying Lemma \ref{lmm:intexten2d}, we obtain
\[
\begin{aligned}
\lambda^2(A_0) = \int_{\Tcal_0} F_0(t)-F_1(t) \, dp < \infty
\end{aligned}
\]
where the finiteness of the right-hand side follows from the fact that $\E|X_i|<\infty$ and Lemma \ref{expectviacdf}.

Observe next that  the definition of the generalized inverse implies that
\begin{equation*}
\begin{aligned}
F^{[-1]}_i(p) \leq t \, \Leftrightarrow \,  p \leq F_i(t),\quad  F^{[-1]}_i(p) > t \, \Leftrightarrow \,  p > F_i(t)
\end{aligned}
\end{equation*}
and hence
\[
A_0 = \{(p,t)\in (0,1) \times \RR:   F_0^{[-1]}(p) \le t < F_1^{[-1]}(p)\}.
\]
Note by above $(p,t) \in A_0$ implies that $p \in \Pcal_1$. Hence, Lemma \ref{lmm:intexten2d}  imply
\[
\begin{aligned}
\lambda^2(A_0) = \int_{\Pcal_1} F_1^{[-1]}(p)-F_0^{[-1]}(p) \, dp
\end{aligned}
\]
and this proves \eqref{lmm::cdfquantconn}$_1$. The proof of \eqref{lmm::cdfquantconn}$_2$ is similar.
\end{proof}

\begin{lemma}\label{lmm:expectviacdf}
Let $X$ be a random variable with $\E|X|<\infty$. Let $X^+=\max(0,X)$, $X^-=\max(0,-X)$. Then
\begin{equation}\label{expectviacdf}
\begin{aligned}
\E[X]=\E[X^+]-\E[X^-], \quad \E[X^+]=\int_0^\infty (1-F(t)) \, dt, \quad \E[X^-]=\int_{-\infty}^0 F(t) dt
\end{aligned}
\end{equation}
where $F$ is the CDF of $X$.
\end{lemma}
\begin{proof}
Note that $|X(\omega)| \geq X^+(\omega),X^-(\omega) \geq 0$ and hence $\E[X^+]$ and $\E[X^-]$ are finite. Recalling that $X=X^+-X_-$, we obtain \eqref{expectviacdf}$_1$. 

Next, by definition of the expectation, we have
\[
\begin{aligned}
\infty > \E[X^+]&=\int_{\Omega} X^+(\omega) \, \P(d\omega) = \int_{\Omega} \bigg( \int_{\RR} \1_{\{0 \leq x\leq X^+(\omega)\}} dx \bigg)\, \P(d\omega) \\
& = \int_{\RR} \1_{\{0 \leq x\}} \bigg( \int_{\Omega}  \1_{\{x\leq X^+(\omega)\}} \P(d\omega) \bigg)\, dx = \int_0^{\infty} (1-F(x)) \, dx
\end{aligned}
\]
where we applied the Tonelli's theorem to exchange the order of integration. This proves \eqref{expectviacdf}$_2$. The proof for \eqref{expectviacdf}$_3$ is similar.
\end{proof}

\begin{lemma}\label{lmm:intexten2d}
Let $\lambda$ denote the Lebesgue measure on $\RR$. Let $f,g$ be $\lambda$-measurable functions such that $g \leq f$.
\begin{itemize}
  \item [(i)]  If $f-g \in L^1(\RR)$, then
\begin{equation}\label{intexten2d}
\begin{aligned}
\lambda \otimes \lambda \Big( \big\{(x,y): g(x) < y < f(x) \big\} \Big) & =   \int_{\RR} (f-g) \, d\lambda \\
& = \, \lambda \otimes \lambda\Big( \big\{(x,y): g(x) \leq y \leq f(x) \big\} \Big) < \infty.
\end{aligned}
\end{equation}

  \item [(ii)]  If $\lambda \otimes \lambda \Big( \big\{(x,y): g(x) < y < f(x) \big\} \Big) < \infty $, then $f-g \in L^1(\RR)$ and \eqref{intexten2d} holds. 
\end{itemize}
\end{lemma}
\begin{proof}
Suppose that $f-g \in L^1(\RR)$. Since $f$ and $g$ are measurable, the set $\{(x,y): g(x) < y < f(x) \big\}$ is measurable with respect to the product measure $\lambda^2=\lambda \otimes \lambda$.  Then by the Tonelli's theorem we obtain
\[
\begin{aligned}
\infty & > \int_{\RR} (f(x)-g(x)) \, d\lambda(x)  = \int_{\RR} \Big( \int_{\RR} \1_{\{ y: g(x) < y < f(x) \}} \, d\lambda(y) \Big) d \lambda(x) \\
& \quad = \int_{\RR^2} \1_{\{ (x,y): g(x) < y < f(x) \}} \, d(\lambda \otimes \lambda) 
= \lambda^2 \big( \{(x,y): g(x) < y < f(x) \} \big),
\end{aligned}
\]
which proves the first equality in \eqref{intexten2d}. The second equality \eqref{intexten2d} is proved similarly. This gives $(i)$. 

Suppose that $\lambda^2 \big( \{(x,y): g(x) < y < f(x) \} \big)<\infty$. Following the calculations above in the reverse order we conclude that $f-g \in L^1(\RR)$ and hence  \eqref{intexten2d} holds. This proves $(ii)$.
\end{proof}


\end{appendices}


\end{document}